\documentclass{article} 
\usepackage{iclr2027_preprint,times}

\usepackage{amsmath,amsfonts,bm}

\def\eqref#1{equation~\ref{#1}}

\def\1{\bm{1}}

\DeclareMathAlphabet{\mathsfit}{\encodingdefault}{\sfdefault}{m}{sl}
\SetMathAlphabet{\mathsfit}{bold}{\encodingdefault}{\sfdefault}{bx}{n}

\usepackage[utf8]{inputenc} 
\usepackage[T1]{fontenc}    
\usepackage{hyperref}       
\usepackage{url}            
\usepackage{booktabs}       
\usepackage{amsfonts}       
\usepackage{nicefrac}       
\usepackage{microtype}      
\usepackage{xcolor}         

\usepackage{graphicx}   
\usepackage{amsmath}
\usepackage{amsthm}
\usepackage[capitalize,noabbrev]{cleveref}  
\usepackage{array}
\usepackage{multirow}
\usepackage{subcaption}
\usepackage{makecell}

\usepackage{algorithm}
\usepackage{algorithmic}    
\usepackage{wrapfig}

\theoremstyle{plain}
\newtheorem{theorem}{Theorem}[section]
\newtheorem{proposition}[theorem]{Proposition}
\newtheorem{lemma}[theorem]{Lemma}

\theoremstyle{definition}

\theoremstyle{remark}
\newtheorem{remark}[theorem]{Remark}

\title{Formalizing the Sampling Design Space of Diffusion-Based Generative Models via Adaptive Solvers and Wasserstein-Bounded Timesteps}

\author{%
  Sangwoo Jo \\
  Department of Artificial Intelligence \\
  Korea University \\
  \texttt{jasonjo97@korea.ac.kr} \\
  \And
  Sungjoon Choi \\
  Department of Artificial Intelligence \\
  Korea University \\
  \texttt{sungjoonc@korea.ac.kr} \\
}

\iclrpreprint 
\begin{document}

\maketitle

\begin{abstract}
Diffusion-based generative models have achieved remarkable performance across various domains, yet their practical deployment is often limited by high sampling costs. While prior work focuses on training objectives or individual solvers, the broader sampling design problem, specifically solver selection and scheduling, remains largely governed by static heuristics. We propose SDM, a principled, training-free sampling framework that adapts both the numerical solver and the timestep schedule to the intrinsic properties of the diffusion trajectory. By analyzing the PF-ODE dynamics, we show that velocity variation is small in high-noise stages and increases near the data manifold, identifying intervals where solver order is most consequential. In parallel, we introduce an offline-calibrated adaptive scheduling method that explicitly controls the local Wasserstein discretization error and projects the calibrated trajectory to a prescribed NFE budget. We further extend the formulation to a mixed-transition Wasserstein error bound, providing a unified error-propagation view of adaptive scheduling and solver selection within the overall SDM framework. Across standard benchmarks, with extensions to modern ODE samplers, high-resolution synthesis, and text-to-image generation, SDM achieves improved sample quality compared to baseline methods, attaining an FID of 1.93 on CIFAR-10, 2.41 on FFHQ, and 1.98 on AFHQv2, with a reduced number of function evaluations compared to existing samplers. 
Our code is available at \url{https://github.com/aiimaginglab/sdm}.
\end{abstract}

\section{Introduction}
Diffusion-based generative models have established themselves as the premier framework for high-fidelity synthesis, defining the state-of-the-art in image generation \citep{ho2020denoising, song2020score, rombach2022high}, video synthesis \citep{ho2022video, singer2022make}, and 3D content creation \citep{poole2022dreamfusion, lin2023magic3d}.
Fundamentally, these models cast generation as a differential equation problem: sampling is equivalent to numerically integrating a Probability Flow ODE (PF-ODE) that transports a simple prior distribution to the complex data distribution.

However, this theoretical elegance comes with a practical bottleneck.
To minimize discretization error, the ODE must be simulated with fine-grained steps, leading to prohibitively high computational costs (measured in Number of Function Evaluations, NFE).
While distillation methods \citep{salimans2022progressive, meng2023distillation} and Consistency Models \citep{song2023consistency, luo2023latent, wang2024phased, kim2023consistency, yang2024consistency} offer few-step generation, they require expensive retraining or fine-tuning pipelines.
For the vast majority of pre-trained models, the central challenge remains: \textit{How can we solve the generative ODE as efficiently as possible without modifying the model itself?}

Current approaches to this integration problem are largely dominated by static heuristics. Standard solvers (e.g., Euler, Heun, DPM-Solver++ \citep{lu2025dpm}) and schedules (e.g., time-uniform, EDM \citep{karras2022elucidating}) treat the diffusion trajectory as homogeneous, applying the same numerical precision and step-size logic from start to finish.
This ignores a critical geometric reality: the stiffness of the diffusion ODE is not constant.
As we demonstrate in this work, the trajectory evolves from a nearly linear flow in high-noise regions (dominated by the prior) to a highly non-linear, curved path as it converges onto the low-dimensional data manifold.
Treating these distinct regimes with a uniform strategy is inherently suboptimal.

In this work, we formalize the sampling design space through a geometric lens, proposing \textbf{SDM}, "\textbf{S}ampling Design Space for \textbf{D}iffusion \textbf{M}odels", a principled framework that \textbf{}adapts the numerical solver and timestep schedule to the intrinsic properties of the diffusion trajectory. Our contributions are two-fold:

\textbf{(i) Curvature Analysis and Solver Allocation.}
We provide a rigorous analysis of the PF-ODE curvature under standard parameterizations.
As illustrated in \cref{fig:main_figure}, the trajectory curvature is negligible during early sampling stages but spikes significantly near the data manifold ($t \to 0$).
Based on this, we propose a dynamic allocation strategy that deploys efficient first-order solvers early and transitions to stable higher-order solvers only when geometrically required.
We provide theoretical justification for this solver allocation strategy and empirically validate its effectiveness in \cref{sec:experiments}. 

\begin{figure*}[t]
  \centering
  \includegraphics[width=1.0\linewidth]{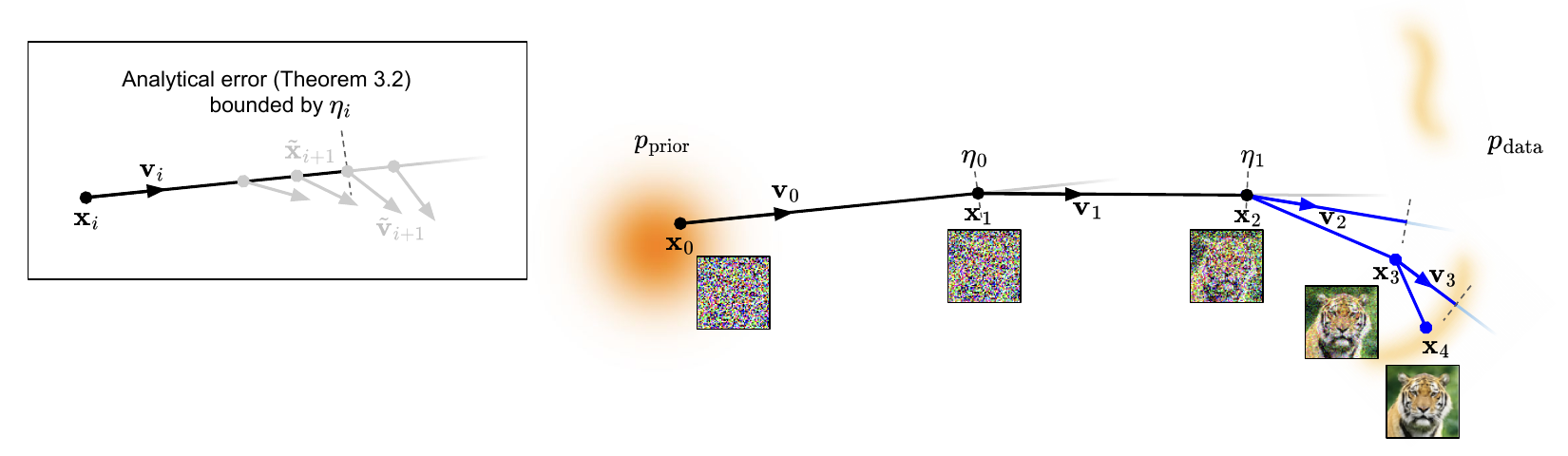}
  \caption{Overview of SDM. \textbf{(Inset)} We introduce an adaptive timestep optimization derived from an analytical upper bound on the Wasserstein error. Here, $\eta_i$ serves as a controllable schedule that explicitly governs the allowable error budget at each step, tightening step sizes where the flow is most sensitive. \textbf{(Main)} Complementing this, SDM adapts the solver order to the trajectory's geometry. In the early high-noise regime (black), the flow is nearly linear, allowing efficient low-order integration. As the trajectory approaches the data manifold \textcolor{blue}{(blue)}, the curvature increases, necessitating high-order solvers. These two strategies act as independent and complementary components, jointly improving both sample quality and computational efficiency.}
  \label{fig:main_figure}
\end{figure*}

\textbf{(ii) Wasserstein-Bounded Adaptive Scheduling.}
Moving beyond heuristic timestep selection, we introduce an optimization framework grounded in Wasserstein geometry. We derive an adaptive schedule that explicitly bounds the local Wasserstein distance between the discretized trajectory and the true continuous flow.
We further introduce an \textit{$N$-step resampling} procedure that projects this adaptive path onto a fixed NFE budget, allowing flexible control over the quality-efficiency trade-off.

Crucially, SDM is a \textit{training-free} framework.
It can be plugged into existing pre-trained diffusion models (e.g., EDM, Stable Diffusion) to improve the Pareto frontier of quality versus efficiency.
We demonstrate that SDM achieves state-of-the-art performance across standard benchmarks, including an FID of 1.93 on CIFAR-10, 2.41 on FFHQ, and 1.98 on AFHQv2, surpassing established baselines that rely on static heuristics. Furthermore, we show that SDM can be extended to modern ODE samplers (e.g., DPM-Solver++, UniPC \citep{zhao2023unipc}), and scales effectively to higher-resolution and text-to-image generation settings.

\section{Preliminaries}
\label{sec:preliminaries}

\subsection{Probability Flow ODE}
Diffusion models are a class of generative model framework that gradually injects noise into data through a forward process and subsequently removes noise through a learned reverse process. The following generative process can be formulated as a stochastic differential equation (SDE) parameterized by a noise schedule, where $s(t)$ and $\sigma(t)$ denote the scale and noise term, respectively. In addition to the stochastic reverse-time SDE, there exists a deterministic \textit{probability flow ODE} that produces the same marginal distribution as the reverse SDE \citep{song2020score}: 
\begin{equation}
\label{eq:pf_ode}
dx = \left[\frac{\dot{s}(t)}{s(t)}x - s(t)^2 \dot{\sigma}(t) \sigma(t) \nabla_x \log p\left(\frac{x}{s(t)}; \sigma(t)\right)\right] dt.
\end{equation}
Sampling from diffusion models is performed by discretizing this ODE into a finite number of steps from prior distribution at $t=T$ to the data distribution at $t=0$ using numerical solvers.

\subsection{Numerical Integration and Scheduling}
Solving \cref{eq:pf_ode} requires discretizing the time interval $[0, T]$ into a sequence $\{t_i\}_{i=0}^N$.
The sample quality is governed by two coupled design choices: numerical solver and timestep schedules.

\textbf{Numerical Solvers.}
Solvers navigate the trade-off between local precision and computational cost (NFE).
First-order methods (e.g., Euler) are computationally cheap (1 NFE/step) but incur a discretization error of $O(h^2)$.
To improve precision without increasing steps, high-order solvers such as DPM-Solver++ \citep{lu2025dpm}, UniPC \citep{zhao2023unipc}, and DEIS \citep{zhang2022fast} employ higher-order Taylor approximations or predictor-corrector schemes.
While these methods significantly reduce discretization error ($O(h^k)$ for $k \ge 2$), they typically require complex multi-step history or additional evaluations per step, making their uniform application across the entire trajectory potentially inefficient.

\textbf{Timestep Scheduling.}
Diffusion timesteps $\{t_i\}_{i=0}^N$ determines the step size $\Delta t_i$ and the noise levels $\sigma(t_i)$.
Early works employed heuristic time-based schedules, such as linear \citep{ho2020denoising} or cosine schedules \citep{nichol2021improved}.
More recently, Karras et al. \citep{karras2022elucidating} proposed a noise-level based polynomial schedule (EDM), 
which has become a standard baseline due to its empirical robustness. Subsequent studies have explored timestep optimization based on various criteria including score-based measures \citep{williams2024score, xue2024accelerating}, KL divergence \citep{sabour2024align, park2024textit, li2023autodiffusion}, and Wasserstein-based error estimates \citep{hu2024adaflow}. However, there does not exist a unified framework that integrates solver selection and timestep allocation based on the intrinsic geometry of the diffusion trajectory. Motivated by this gap, we propose a unified sampling design space that characterizes both of these dimensions. 

\section{SDM: Sampling via Adaptive Solvers and Timestep Schedulers}
We investigate the sampling design space of pretrained diffusion models from a principled perspective. While prior works often rely on static heuristics, we propose a geometric framework, \textit{SDM}, which provides a theoretically grounded for both numerical solver selection and timestep scheduling based on the intrinsic properties of the diffusion trajectory.

In \cref{subsec:adaptive_ode_solvers}, we analyze the local stiffness of the probability flow ODE. While state-of-the-art exponential integrators (e.g., DPM-Solver++ \citep{lu2025dpm}) achieve efficiency by analytically solving the linear component of the diffusion ODE, they fundamentally rely on fixed-order local approximations of the nonlinear residual. Our analysis generalizes this by explicitly quantifying the curvature of this residual, demonstrating that low-order solvers, such as Euler method, are sufficient for high-noise regimes, while higher-order corrections become necessary only in later stages. In \cref{subsec:adaptive_scheduling}, we introduce a Wasserstein-bounded adaptive strategy that optimizes the timestep schedule to minimize the discretization error under a fixed computational budget. We empirically show that this leads to consistent improvements in sample quality. 

\subsection{Adaptive ODE Solvers}
\label{subsec:adaptive_ode_solvers}

\subsubsection{Analysis of ODE curvature} 
\label{subsec:ode_curvature_analysis}
To design an efficient adaptive sampling strategy, we must first quantify the geometric properties of the diffusion trajectory that govern numerical stability. The discretization error of any numerical solver is strictly bounded by the magnitude of the trajectory's higher-order derivatives, specifically the acceleration $\|\ddot{x}\|$. This quantity serves as a direct proxy for the \textit{local stiffness} of the probability flow ODE \citep{10.5555/3722577.3722688}.

Exponential integrators mitigate stiffness by explicitly solving the linear drift term. However, the discretization error is then dominated by the higher-order derivatives of the neural network prediction itself (the non-linear noise residual). Our curvature analysis focuses on this residual complexity, providing a solver-agnostic measure of difficulty that applies regardless of the integrator's baseline order. In \cref{thm:second_curvature}, we derive the closed-form expression of $\|\ddot{x}\|$ for standard diffusion parameterizations.
\begin{proposition}
\label{thm:second_curvature}
Under $\text{EDM}/\text{VP}/\text{VE}$ parameterizations, the second derivative of the probability flow ODE can be expressed as follows:

\textbf{EDM}
\begin{equation}
\label{eq:edm_sec_derivative}
\ddot{x} = - \frac{1}{\sigma^2}J_x(x-D_{\theta}) - \frac{1}{\sigma} D_{\sigma},
\end{equation}
\textbf{VP}
\begin{equation}
\begin{aligned}
\ddot{x} &= \left(\frac{1}{4}B^2 - \frac{1}{2}\beta_d
\right)x + \left[\frac{1}{2}\beta_d(1+\sigma^{-2}) - \frac{1}{4}B^2 \left(1+2\sigma^{-2}+\sigma^{-4}\right)\right](x-sD_\theta) \\
&\quad - \frac{1}{4}B^2 \frac{\sigma+2\sigma^{-1}+\sigma^{-3}}{\sqrt{1+\sigma^2}} \left(J_x\epsilon_\theta + D_\sigma\right),
\end{aligned}
\end{equation}
\textbf{VE}
\begin{equation}
\label{eq:ve_sec_derivative}
\ddot{x} = - \frac{1}{4\sigma^4}(I+J_x)(x-D_{\theta}) - \frac{1}{4\sigma^3} D_{\sigma},
\end{equation}
\end{proposition}
where $J_x$ denotes the Jacobian of the composed map $D_\theta(x/s(t);\sigma(t))$ with respect to $x$, and $D_\sigma$ denotes the partial derivative with respect to $\sigma$ at fixed $x/s(t)$. Detailed proof and closed-form expressions for $B$, $\dot{\sigma}$, and $\ddot{\sigma}$ under the VP formulation are provided in Appendix~\ref{appendix:ode_derivation}. 

The observed power-law relationship, i.e., $\ddot{x}\propto\sigma^{-p}$ with $p \geq 0$, in \cref{eq:edm_sec_derivative}--\cref{eq:ve_sec_derivative} implies that the PF-ODE trajectory becomes increasingly curved as the process approaches the data manifold. This theoretical insight motivates the use of a
mixed-solver strategy that progressively adapts the solver order to the trajectory’s changing stiffness. 

\subsubsection{Mixture of ODE Solvers} 
While \cref{thm:second_curvature} highlights the varying stiffness across the diffusion trajectory, utilizing this insight requires a practical metric to allocate solver types efficiently. Computing the exact second derivative $\|\ddot{x}\|$ involves Hessian-vector products, which are computationally prohibitive to evaluate repeatedly. To address this, we use discrete proxies based on vector-field variation, allowing us to construct a solver-order policy without additional model evaluations during inference.

Given diffusion timesteps $\{t_i\}_{i=0}^{N}$ such that
\begin{equation}
t_0 > t_1 > \cdots > t_N = 0,
\end{equation}
let $v_{t_i} := v(x_{t_i},t_i)$ denote the vector field evaluated at the $(x_{t_i}, t_i)$. We define the \textit{absolute local curvature}\footnote{In this work, we use the term ``curvature'' to refer to the trajectory non-linearity (acceleration $\|\ddot{x}\|$) that governs the discretization error of ODE solvers.} as 
\begin{equation}
\label{eq:abs_local_curvature}
\kappa_{\mathrm{abs}}(t_i) := \frac{\|v_{t_{i+1}}-v_{t_i}\|}{\Delta t_i} \approx \|\ddot{x}_{t_i}\|,
\end{equation}
where $\Delta t_i = t_i - t_{i+1}$.
To obtain a scale-invariant measure of stiffness relative to the signal magnitude, we further introduce the \textit{relative local curvature},
\begin{equation}
\kappa_{\mathrm{rel}}(t_i) := \frac{\kappa_{\mathrm{abs}}(t_i)}{\|v_{t_i}\|}.
\end{equation}
Computing $\kappa_{\mathrm{rel}}(t_i)$ directly requires requires evaluating the vector field at both $(x_{t_i},t_i)$ and $(x_{t_{i+1}},t_{i+1})$, which increases twice the number of function evaluations (NFE). To avoid this overhead, we hence propose a \textit{cache-based curvature} that does not require additional network evaluations. Specifically, we define 
\begin{equation}
{\widehat{\kappa}}_{\mathrm{rel}}(t_i) := \frac{\|v_{t_i}-v_{t_{i-1}}\|}{\Delta \hat{t}_i \|v_{t_{i-1}}\|},
\end{equation}
where $\Delta \hat{t}_i$ is the interval between the actual evaluation points used by the EDM sampler. Since $v_{t_{i-1}}$ is computed at the previous timestep and $v_{t_i}$ is required for the current update, the curvature $\widehat{\kappa}_{\mathrm{rel}}(t_i)$ can be obtained by simply reusing the cached evaluations, resulting in $\text{NFE}=1$ per step. A detailed analysis justifying that $\widehat{\kappa}_{\mathrm{rel}}(t_i)$ provides an one-step delayed estimate of the relative local curvature is given in Appendix~\ref{sec:proxy_estimator}. Based on the curvature profile of the sampling trajectory, we propose a \textit{dynamic ODE solver selection} strategy by introducing threshold $\tau_{\kappa}$ that determines when to switch from low-order to high-order (e.g., Euler/Heun, DPM-Solver++(1M/2M)) solver updates.

\subsection{Adaptive Timestep Schedulers}
\label{subsec:adaptive_scheduling}
\subsubsection{Adaptive step size based on Wasserstein distance} 
Given two probability distributions $\mu$ and $\nu$, the squared 2-Wasserstein distance is defined as:
\begin{equation}
W_2^2(\mu,\nu) := \inf_{\pi \in \Pi(\mu,\nu)} \mathbb{E}_{(X,Y) \sim \pi}[\|X-Y\|^2], 
\end{equation}
where $\Pi(\mu,\nu)$ denotes the set of all couplings between $\mu$ and $\nu$. In diffusion-based generative models, each numerical integration step induces a transport between probability distributions at adjacent timesteps. Within this framework, the step size directly controls the local Wasserstein distance induced by the sampler. Driven by this insight, we propose an \textit{adaptive step size scheduling} algorithm that explicitly bounds the local Wasserstein distance at each sampling step.

\begin{theorem}[Local Wasserstein Step-Size Condition]
\label{thm:stepsize_bound}
Given a diffusion timestep $t_i$ and a step size $\Delta t_i>0$, let $t_{i+1}=t_i-\Delta t_i$. Let $\Phi_{t_i \to t_{i+1}}^{*}$ denote the exact probability-flow transition from $t_i$ to $t_{i+1}$, and let
\begin{equation}
\Phi_{t_i \to t_{i+1}}^{\mathsf{E}}(x_{t_i})
= x_{t_i} - \Delta t_i \, v(x_{t_i}, t_i).
\end{equation}
be the Euler transition. Suppose the along-trajectory velocity $v(x_s^{*},s)$ is absolutely continuous on $[t_{i+1},t_i]$, and define
\begin{equation}
S_{i} = \left( \mathbb{E} \left[ \sup_{s\in[t_{i+1},t_i]} \left\| \frac{d}{ds}v(x_s^{*},s) \right\|^2 \right] \right)^{1/2}.
\end{equation}
Then, the sufficient step-size condition
\begin{equation}
\Delta t_i \leq \sqrt{\frac{2\eta}{S_{i}}}
\label{eq:adaptive_stepsize}
\end{equation}
ensures that the local Wasserstein distance does not exceed the given upper bound $\eta$.
\end{theorem}
The detailed proof is provided in Appendix~\ref{appendix:stepsize_bound}. The quantity $S_{i}$ measures the local variation of the velocity field along the true trajectory. It is worth noting that $S_{i}$ can be expressed as a function of local curvature measure defined in \cref{eq:abs_local_curvature}, further motivating the joint design of adaptive solver allocation and timestep scheduling. Since $S_{i}$ is intractable, we approximate it during an offline calibration phase using a trial Euler step $\Delta t_{\mathrm{trial}} := t_i - \tilde{t}_{\mathrm{trial}}$, evaluated at candidate $(x_{\tilde{t}_\mathrm{trial}}, \tilde{t}_{\mathrm{trial}})$:
\begin{equation}
\label{eq:approx_adaptive_stepsize}
\widehat{S}_{i} := \frac{\left\| v(x_{\tilde{t}_\mathrm{trial}}, \tilde{t}_{\mathrm{trial}}) - v(x_{t_i}, t_i) \right\|}{\Delta t_{\mathrm{trial}}}.
\end{equation}
The corresponding empirical Euler local transition error proxy is then defined as
\begin{equation}
\hat{\eta}_{i}^\mathsf{E} := \frac{\Delta t_{\mathrm{trial}}^2}{2}\widehat{S}_{i}.
\label{eq:euler_proxy_calibration}
\end{equation}
In practice, the timestep schedule is calibrated with the Euler-based local transition error term derived in \cref{eq:euler_proxy_calibration}. We empirically find that the following method suffices to generate optimized schedules that yields improved model performance. Furthermore, despite being based on a first-order approximation, it should be noted that the first-order bound provides a conservative estimate that remains valid for higher order solvers, as the error term for second-order method, for instance, scales to $\mathcal{O}(\Delta t^3)$ compared to $\mathcal{O}(\Delta t^2)$ behavior of first-order methods (see Appendix~\ref{appendix:stepsize_bound_heun}).

\paragraph{$\eta$-scheduling.}
Based on the local Wasserstein error budget $\eta_i$, we perform timestep scheduling by selecting the largest step size $\Delta t_i$ that satisfies the corresponding error bound. The practical realization of our Wasserstein-bounded scheduling adopts a computationally efficient predictor-corrector strategy to enforce the local error budget derived in \cref{thm:stepsize_bound}.
To mitigate the computational cost of exhaustive search, we initialize the candidate $\tilde{t}_{i+1}$ from a pre-defined reference grid, effectively warm-starting the search procedure with a value proximal to the optimal solution. Given the current timestep $t_i$ and the candidate $\tilde{t}_{i+1}$, we
estimate the local velocity variation by
$\widehat{S}_i =
\|\tilde{v}_{t_{i+1}}-v_{t_i}\|/
(t_i-\tilde{t}_{i+1})$.
The candidate is then verified by checking whether the resulting local error satisfies the prescribed budget $\eta_i$. If the candidate violates the budget or, conversely, proves overly conservative, we refine $\tilde{t}_{i+1}$
using an \textit{exponential backoff mechanism}.
By expanding or contracting the step size via multiplicative factors, this procedure ensures logarithmic search complexity $O(\log \frac{\Delta t}{\delta})$ akin to classical backtracking methods like the Armijo-Wolfe conditions \citep{armijo1966minimization, wolfe1969convergence, wolfe1971convergence}. 
We provide the detailed algorithm in Appendix~\ref{appendix:pseudo_algorithm}.

\paragraph{Total Wasserstein Error Bound for Adaptive Solvers.}
We further demonstrate that taking the proposed step-size selection consequently bounds the total Wasserstein error under the mixed-solver strategy introduced in \cref{subsec:adaptive_ode_solvers}. In particular, as long as each selected numerical transition is Wasserstein-stable, the global Wasserstein error remains bounded by the accumulated local transition error budgets, propagated according to the stability of subsequent transitions. Thus, switching between different numerical solvers does not introduce additional error terms; it only changes the local transition error and the corresponding stability factor at each step. A formal bound and its proof are provided in Appendix~\ref{appendix:mixed_solver_accounting}.

In practice, the difficulty of the sampling dynamics varies significantly across noise levels. To account for this, we allow the error tolerance term $\eta_i$ to depend on the current noise level $\sigma_i$ as the following:
\begin{equation}
\eta_{i} = (\eta_{\mathrm{max}}-\eta_{\mathrm{min}})\left(\frac{\sigma_i}{\sigma_{\mathrm{max}}}\right)^p + \eta_{\mathrm{min}},
\end{equation}
where $p \geq 0$ controls how strongly the error budget is concentrated throughout different noise levels. The following design assigns a larger tolerance during the early, high-noise stages, while the error bound becomes more restrictive as $\sigma_i \rightarrow 0$.

\subsubsection{$N$-step resampling}
The proposed scheduling method provides a sequence of diffusion timesteps such that the local Wasserstein error remains below a prescribed threshold, but the resulting number of steps is not directly controlled. To obtain a schedule with a desired number of inference steps $N$, we therefore perform an additional resampling procedure.

Given a calibrated sequence of diffusion timesteps
$\{t_i\}_{i=0}^{M}$ and corresponding noise levels $\{\sigma_i\}_{i=0}^{M}$, defined by scheduler $\sigma$ such that $t_i = \sigma^{-1}(\sigma_i)$, we define an incremental cost $L(t_i,t_{i+1})$ and its
corresponding cumulative geodesic length as 
\begin{equation}
\Gamma = \sum_{i=0}^{M-1}\sqrt{L(t_i,t_{i+1})}.
\end{equation}
As shown in prior work \citep{williams2024score}, an optimal $N$-step schedule can be obtained by uniformly discretizing the cumulative geodesic length, such that
\begin{equation}
\sqrt{L(t_i^*,t_{i+1}^*)} \approx \Gamma/N, \quad i=0,...,N-1,
\end{equation}
We further show that the same conclusion holds and expand the following resampling procedure under a weighted incremental cost setting, as proven in Appendix~\ref{prop:weighted_sum}.
\begin{equation}
\label{eq:weighted_cost}
\tilde{L}(t_i,t_{i+1}) := w(t_i)L(t_i,t_{i+1}), \qquad w(t_i)>0.
\end{equation}
Hence, we define our \textit{weighted} incremental cost function using the empirically measured local error proxy $\hat{\eta}_{i}$ from \cref{eq:euler_proxy_calibration} as
\begin{equation}
\label{eq:local_error_proxy}
\tilde{L}(t_i,t_{i+1}) := w(t_i)\hat{\eta}_i.
\end{equation}
We then obtain an $N$-step schedule $\{t_i^*\}_{i=0}^{N}$ by uniformly discretizing the weighted cumulative geodesic length, which we denote the following procedure as the \textit{$N$-step resampling} algorithm. For our implementation, we define the weighted function in terms of the noise level $\sigma$ as
\begin{equation}
g(\sigma_i) :=
\left(\frac{\sigma_i}{\sigma_{\mathrm{max}}}\right)^{-q},
\qquad
w(t_i)=g(\sigma_i)^2,
\end{equation}
where parameter $q \geq 0$ controls the degree of emphasis on the later steps. Larger values of $q$ allocate a greater fraction of the sampling budget to the low-$\sigma$ regime. The offline resampling and online inference procedure is provided in Appendix~\ref{appendix:pseudo_algorithm}. We additionally provide proof that the Wasserstein-based local error interpretation remains valid after $N$-step resampling in Appendix~\ref{appendix:error_bound_n_step_resampling}.

\begin{table*}[t]
\caption{Quantitative results on unconditional generation for CIFAR-10, FFHQ, and AFHQv2 dataset, measured by FID and NFE. ``SDM (Adaptive Scheduling)'' denotes the Wasserstein-bounded timestep schedule for the solver indicated in the leftmost column. Rows under ``SDM (Adaptive Solver)'' use the curvature-based solver selection with the optimized threshold parameter $\tau_{\kappa}$. The best FID within each solver block for a given dataset and parameterization are highlighted in \textbf{bold}, while the overall best FID in each column are highlighted in \underline{\textbf{bold}}. Overall, the proposed SDM sampler achieves superior performance in both sample quality and computational efficiency over all baselines.}
\label{table:results_unconditional}
\centering
\renewcommand{\arraystretch}{1.15}
\setlength{\tabcolsep}{4pt}

\resizebox{\textwidth}{!}{
\begin{tabular}{
    p{0.18\textwidth} @{\hspace{12pt}} l p{0.06\textwidth}
    *{6}{>{\centering\arraybackslash}p{0.07\textwidth}}
    }
    \toprule
        \multicolumn{3}{c}{} & 
        \multicolumn{2}{c}{Unconditional} &
        \multicolumn{2}{c}{Unconditional} &
        \multicolumn{2}{c}{Unconditional} \\      
        \multicolumn{3}{c}{} & 
        \multicolumn{2}{c}{CIFAR-10 $32 \times 32$} &
        \multicolumn{2}{c}{FFHQ $64 \times 64$} &
        \multicolumn{2}{c}{AFHQv2 $64 \times 64$} \\
        \cmidrule(lr){4-5}\cmidrule(lr){6-7}\cmidrule(lr){8-9}
        \multicolumn{3}{c}{} & 
        \multicolumn{1}{c}{VP} & \multicolumn{1}{c}{VE} &
        \multicolumn{1}{c}{VP} & \multicolumn{1}{c}{VE} &
        \multicolumn{1}{c}{VP} & \multicolumn{1}{c}{VE} \\
        \midrule
        \textbf{Solver} & \textbf{Schedule} & & & & & & & \\
        \midrule
        \multirow{4}{*}{Euler} & EDM $(\rho=7)$ & & \multicolumn{1}{c}{7.61} & \multicolumn{1}{c}{7.75} & \multicolumn{1}{c}{4.59} & \multicolumn{1}{c}{4.76} & \multicolumn{1}{c}{2.93} & \multicolumn{1}{c}{3.21} \\
        & COS \citep{williams2024score} & & \multicolumn{1}{c}{7.43} & \multicolumn{1}{c}{7.25} & \multicolumn{1}{c}{4.29} & \multicolumn{1}{c}{\textbf{4.46}} & \multicolumn{1}{c}{2.89} & \multicolumn{1}{c}{3.10} \\
        & SDM (Adaptive Scheduling) & & \multicolumn{1}{c}{\textbf{6.18}} & \multicolumn{1}{c}{\textbf{6.48}} & \multicolumn{1}{c}{\textbf{4.16}} & \multicolumn{1}{c}{4.51} & \multicolumn{1}{c}{\textbf{2.71}} & \multicolumn{1}{c}{\textbf{2.97}} \\
        \cmidrule{2-9}
        & & NFE & \multicolumn{1}{c}{17} & \multicolumn{1}{c}{17} & \multicolumn{1}{c}{39} & \multicolumn{1}{c}{39} & \multicolumn{1}{c}{39} & \multicolumn{1}{c}{39} \\
        \midrule
        \addlinespace[2pt]
        
        \multirow{4}{*}{Heun} & EDM $(\rho=7)$ & & \multicolumn{1}{c}{\textbf{1.96}} & \multicolumn{1}{c}{1.97} & \multicolumn{1}{c}{2.47} & \multicolumn{1}{c}{2.59} & \multicolumn{1}{c}{2.04} & \multicolumn{1}{c}{2.17} \\
        & COS \citep{williams2024score} & & \multicolumn{1}{c}{2.08} & \multicolumn{1}{c}{2.10} & \multicolumn{1}{c}{2.48} & \multicolumn{1}{c}{2.61} & \multicolumn{1}{c}{\textbf{2.04}} & \multicolumn{1}{c}{2.16} \\
        & SDM (Adaptive Scheduling) & & \multicolumn{1}{c}{2.00} & \multicolumn{1}{c}{\underline{\textbf{1.95}}} & \multicolumn{1}{c}{\underline{\textbf{2.41}}} & \multicolumn{1}{c}{\underline{\textbf{2.52}}} & \multicolumn{1}{c}{2.08} & \multicolumn{1}{c}{\textbf{2.13}} \\
        \cmidrule{2-9}
        & & NFE & \multicolumn{1}{c}{35} & \multicolumn{1}{c}{35} & \multicolumn{1}{c}{79} & \multicolumn{1}{c}{79} & \multicolumn{1}{c}{79} & \multicolumn{1}{c}{79} \\
        \midrule
        \addlinespace[2pt]
        
        \multirow{5}{*}{\shortstack[l]{SDM \\ (Adaptive Solver)}} & EDM $(\rho=7)$ & & \multicolumn{1}{c}{\underline{\textbf{1.93}}} & \multicolumn{1}{c}{1.97} & \multicolumn{1}{c}{2.48} & \multicolumn{1}{c}{\textbf{2.58}} & \multicolumn{1}{c}{\textbf{1.98}} & \multicolumn{1}{c}{\textbf{2.10}} \\
        \cmidrule{2-9}
        & & NFE & \multicolumn{1}{c}{31} & \multicolumn{1}{c}{32} & \multicolumn{1}{c}{72} & \multicolumn{1}{c}{76} & \multicolumn{1}{c}{66} & \multicolumn{1}{c}{67} \\
        \cmidrule{2-9}
        & SDM (Adaptive Scheduling) & & \multicolumn{1}{c}{\underline{\textbf{1.93}}} & \multicolumn{1}{c}{1.97} & \multicolumn{1}{c}{2.47} & \multicolumn{1}{c}{\textbf{2.58}} & \multicolumn{1}{c}{\underline{\textbf{1.97}}} & \multicolumn{1}{c}{\underline{\textbf{1.99}}} \\
        \cmidrule{2-9}
        & & NFE & \multicolumn{1}{c}{32} & \multicolumn{1}{c}{33} & \multicolumn{1}{c}{76} & \multicolumn{1}{c}{77} & \multicolumn{1}{c}{75} & \multicolumn{1}{c}{72} \\
    \bottomrule
\end{tabular}
}
\end{table*}

\section{Experiments}
\label{sec:experiments}
\subsection{Implementation Details}
\textbf{Datasets, Models, and Metrics.}
We evaluate our proposed SDM sampler across various standard benchmarks. For unconditional and conditional generation, we conduct experiments on CIFAR-10 ($32 \times 32$), and FFHQ, AFHQv2, and ImageNet ($64 \times 64$), following the baseline \citep{karras2022elucidating}. We use pretrained models from EDM \citep{karras2022elucidating} with no additional training. For higher-resolution generation, we use EDM2-M \citep{karras2024analyzing} pretrained on ImageNet $512 \times 512$. Sample quality and sampling efficiency are measured using FID and the number of function evaluations (NFE), respectively. For text-to-image generation, we employ SDXL \citep{podell2024sdxl}, Stable Diffusion 3 \citep{esser2024scaling}, and FLUX \citep{labs2025flux} with a guidance scale of 5.0, 7.0, and 3.5, respectively, and evaluate on COCO dataset \citep{lin2014microsoft}, using FID and CLIP scores as our performance metrics.

\textbf{Baselines}
We adopt the timestep discretization from the EDM schedule, where the noise levels $\{\sigma_i\}_{i=0}^{N}$ are defined as 
\begin{equation}
\sigma_{i < N} = \left({\sigma_{\mathrm{max}}}^{\frac{1}{\rho}}+\frac{i}{N-1}({\sigma_{\mathrm{min}}}^{\frac{1}{\rho}}-{\sigma_{\mathrm{max}}}^{\frac{1}{\rho}})\right)^{\rho}, \space  \sigma_{N} = 0.
\end{equation}
with $\rho = 7$. For unconditional and conditional generation settings, we use a fixed number of timesteps for each dataset: 18 steps for CIFAR-10, 40 steps for FFHQ and AFHQv2, and 256 steps for ImageNet. We adopt stochastic settings with $S_{\mathrm{churn}} = 40, S_{\mathrm{min}} = 0.05, S_{\mathrm{max}} = 50$, and $S_{\mathrm{noise}} = 1.003$ for ImageNet. For baseline comparisons, we evaluate both EDM and Corrector Optimized Schedules \citep{williams2024score}, which we abbreviate as COS throughout the paper. 
We use 32 steps for EDM2 and 10 steps for text-to-image generation settings.

\textbf{ODE Solvers and Timestep Scheduling.}
For adaptive solvers, we experiment with combinations of Euler/Heun and DPM-Solver++(1M/2M) \citep{lu2025dpm} updates. During the offline scheduling for conditional generation on ImageNet $64 \times 64$, we remove the stochastic settings from EDM to accurately estimate the local velocity field variations and construct the Wasserstein bounded timestep sequence. During inference, however, we employ the same stochastic churn parameters as the EDM baseline. We use DPM-Solver++(2M) for SDXL and Euler for Stable Diffusion 3 and FLUX.

\textbf{Hyperparameters.}
We set the threshold term $\tau_{\kappa}$, Wasserstein error tolerance parameters $\eta_{\mathrm{min}}, \eta_{\mathrm{max}}, p$, and $N$-step resampling parameter $q$ as our hyperparameters. 

\subsection{Results}
\begin{table*}[b]
\caption{Quantitative results on text-to-image generation with Stable Diffusion 3 (\textbf{Left}) and FLUX (\textbf{Right}). The baseline timestep schedule, denoted as ``Uniform'', includes the shifting term used by each corresponding model. SDM remains achieving the best model performance in FID while remaining competitive in CLIP for recent rectified flow-based text-to-image models.}
\label{tab:fid_sd3_flux}
\centering
\vspace{6pt}

\begin{minipage}[t]{0.48\textwidth}
\centering
\renewcommand{\arraystretch}{1.1}
\setlength{\tabcolsep}{6pt}
\begin{tabular}{l c c}
\toprule
\textbf{Schedule} & \textbf{FID} & \textbf{CLIP} \\
\midrule
Uniform & 23.41 & \textbf{0.3162} \\
SDM & \textbf{22.06} & 0.3151 \\
\bottomrule
\end{tabular}
\end{minipage}
\hfill
\begin{minipage}[t]{0.48\textwidth}
\centering
\renewcommand{\arraystretch}{1.1}
\setlength{\tabcolsep}{6pt}
\begin{tabular}{l c c}
\toprule
\textbf{Schedule} & \textbf{FID} & \textbf{CLIP} \\
\midrule
Uniform & 24.41 & \textbf{0.3112} \\
SDM & \textbf{23.87} & 0.3102 \\
\bottomrule
\end{tabular}
\end{minipage}
\end{table*}
\textbf{Unconditional Generation.}
Quantitative results for unconditional generation are reported in
\cref{table:results_unconditional}.
Across all datasets, solvers, and pre-trained model configurations, SDM consistently improves sample quality
over prior baselines, demonstrating robust performance under both VP and VE parameterizations. Under the Euler solver, SDM with adaptive scheduling yields substantial gains. For Heun solver, SDM remains competitive or superior to existing methods, e.g., achieving FID 2.41 on FFHQ and 2.13 on AFHQv2. Applying the adaptive solver alone surpasses EDM baseline with Heun solvers, achieving FID 1.93 on CIFAR-10, 2.58 on FFHQ, and especially reducing the NFE by approximately 15--20\% while achieving FID 1.98 on AFHQv2. Most notably, combining both adaptive solver with adaptive scheduling yields the best overall performance, reaching FID 1.97 for VP and 1.99 for VE configuration on AFHQv2.

\textbf{Extension to Few-Step Generation.}
To evaluate the practical efficiency of SDM, we evaluate the model performance in few-step generation settings. Tables \ref{table:results_10_step_generation} and \ref{table:results_20_step_generation} present the quantitative results for 10-step and 20-step generation, respectively. Our results demonstrate that SDM consistently outperforms EDM baselines across both FFHQ and AFHQv2 datasets. Most notably, the integration of both adaptive scheduling and adaptive solver yields the strongest performance in terms of FID, confirming that our principled SDM framework remains robust and effective even under highly constrained sampling budgets.

\textbf{Extension to Modern ODE solvers and High-Resolution Generation.} 
We extend our proposed method in modern ODE solvers and high-resolution generation settings. As shown in \cref{table:results_modern_ode_solver}, SDM schedules consistently outperform EDM baselines across standard benchmarks, even when integrated with state-of-the-art high-order solvers. We further validate the effectiveness of both adaptive solver and adaptive scheduling in \cref{tab:fid_edm2_imagenet512}, by decomposing the components and evaluating the model performance on ImageNet $512 \times 512$. SDM schedules provides the primary performance gain, followed by an additional improvement by adaptively increasing the solver order from 1M to 2M. 

\textbf{Extension to Text-to-Image Generation.} 
We provide quantitative results of our proposed scheduling method on SDXL in \cref{tab:fid_sdxl}, where SDM outperforms in FID against standard schedules, including time-uniform, EDM, AYS, and AutoDiffusion. We further evaluate SDM on Stable Diffusion 3 and FLUX in \cref{tab:fid_sd3_flux}, which are recent rectified flow-based text-to-image models, demonstrating superior performance in FID over baselines, hence demonstrating its generalizability to modern text-to-image generation settings. We note that quantitative metrics for text-to-image generation are often limited, as discussed in prior works \citep{podell2024sdxl}. We find that improvements in sample quality are more clearly demonstrated through qualitative observation, which we provide in \cref{fig:additional_results_t2i_sd3,fig:additional_results_t2i_sdxl}.

\begin{figure}[t]
\centering
    \begin{minipage}[t]{0.49\textwidth}
    \centering
    \includegraphics[width=\textwidth]{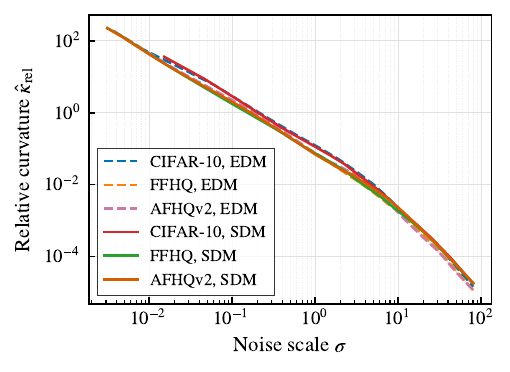}
    \end{minipage}
    \begin{minipage}[t]{0.49\textwidth}
    \centering
    \includegraphics[width=\textwidth]{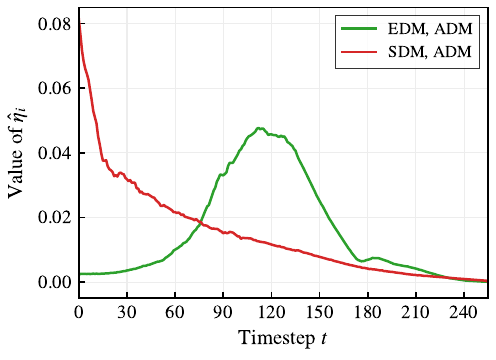}
    \end{minipage}
\caption{Empirical validation of trajectory geometry and Wasserstein error allocation. (\textbf{Left}) Relative curvature $\hat{\kappa}_{\mathrm{rel}}$ as a function of noise level $\sigma$ for standard benchmarks. The curvature exhibits an approximately linear correlation with noise levels $\sigma$ in log scale, consistent with the theoretical derivation of second order probability flow ODE. (\textbf{Right}) Distribution of local Wasserstein error bound $\hat{\eta}_i$ over diffusion timesteps for ImageNet $64\times64$. EDM schedules exhibit an initially increasing trend with a subsequent decay, reaching maximum during the intermediate sampling stages. In contrast, SDM schedules allocate a larger portion of the error budget to the early high-noise stages, resulting in improved sample quality.}
\label{fig:empirical_validation}
\end{figure}

\subsection{Analysis of $\hat{\eta}_i$}
We analyze the evolution of per-step error indicator $\hat{\eta}_i$ for both SDM schedules and EDM schedules. As shown in \cref{fig:empirical_validation}, for EDM schedules, $\hat{\eta}_i$ initially increases during the early sampling stages and decreases afterwards, indicating that there exists unnecessary error budget expended during the intermediate stages, leaving room for additional improvement. In contrast, SDM schedules exhibit a monotonically decreasing trend in $\hat{\eta}_i$ throughout the sampling trajectory, suggesting that SDM schedules efficiently allocates more of the allowable error budget to the early stages of sampling, where dynamics are smoother, while enforcing tighter error control in later stage when the sampling trajectory becomes more nonlinear. As a result, SDM schedules achieve improved sample quality by devoting higher precision to the more critical stages of generation.  

\section{Conclusion}
\label{sec:conclusion}
In this work, we introduce a sampling design space that improves both sample quality and sampling efficiency for pretrained diffusion-based generative models. By formulating the second-order behavior of the PF-ODE, we observe that the trajectory curvature remains negligible during the early sampling stage, which motivates the use of low-order ODE solvers in this regime, while reserving higher-order integration for later stages where the dynamics become more nonlinear. Based on this observation, we propose an adaptive solver that dynamically selects numerical solvers according to the local behavior of the sampling trajectory. In parallel, we introduce an adaptive scheduling method that optimizes diffusion timesteps based on a criterion derived from a Wasserstein error bound. Across extensive experiments against recent baselines, SDM demonstrates strong performance in both sampling quality and computational efficiency compared to baseline methods. Our sampling design currently focuses on optimizing visual quality, leaving the incorporation of text alignment objectives in text-to-image generation as an important direction for future work. Overall, SDM provides a strong and flexible framework for efficient diffusion sampling, and we believe its underlying principles can be extended to broader classes of diffusion-based generative models and applications.

\subsection*{AI use statement}

In this work, we used generative AI tools to aid and polish the writing of this paper, including improving clarity, grammar, and phrasing. We have reviewed all AI-assisted work, verifying the accuracy and consistency of the generated content with the underlying research. We take full responsibility for the final content of this work, including text, claims or artifacts produced with the aid of generative AI.

\subsection*{Reproducibility statement}

We provide anonymous source code as supplementary material to facilitate
the reproduction of our experiments. The paper describes the proposed
method, theoretical analysis, and experimental setup, while the
supplementary material provides the implementation and additional details
needed to reproduce the reported results.

\bibliography{iclr2027_conference}

@article{song2020score,
  title={Score-based generative modeling through stochastic differential equations},
  author={Song, Yang and Sohl-Dickstein, Jascha and Kingma, Diederik P and Kumar, Abhishek and Ermon, Stefano and Poole, Ben},
  journal={arXiv preprint arXiv:2011.13456},
  year={2020}
}

@article{ho2020denoising,
  title={Denoising diffusion probabilistic models},
  author={Ho, Jonathan and Jain, Ajay and Abbeel, Pieter},
  journal={Advances in neural information processing systems},
  volume={33},
  pages={6840--6851},
  year={2020}
}

@inproceedings{nichol2021improved,
  title={Improved denoising diffusion probabilistic models},
  author={Nichol, Alexander Quinn and Dhariwal, Prafulla},
  booktitle={International conference on machine learning},
  pages={8162--8171},
  year={2021},
  organization={PMLR}
}

@article{song2023consistency,
  title={Consistency Models},
  author={Song, Yang and Dhariwal, Prafulla and Chen, Mark and Sutskever, Ilya},
  journal={arXiv preprint arXiv:2303.01469},
  year={2023}
}

@article{luo2023latent,
  title={Latent consistency models: Synthesizing high-resolution images with few-step inference},
  author={Luo, Simian and Tan, Yiqin and Huang, Longbo and Li, Jian and Zhao, Hang},
  journal={arXiv preprint arXiv:2310.04378},
  year={2023}
}

@article{wang2024phased,
  title={Phased consistency models},
  author={Wang, Fu-Yun and Huang, Zhaoyang and Bergman, Alexander and Shen, Dazhong and Gao, Peng and Lingelbach, Michael and Sun, Keqiang and Bian, Weikang and Song, Guanglu and Liu, Yu and others},
  journal={Advances in neural information processing systems},
  volume={37},
  pages={83951--84009},
  year={2024}
}

@article{kim2023consistency,
  title={Consistency trajectory models: Learning probability flow ode trajectory of diffusion},
  author={Kim, Dongjun and Lai, Chieh-Hsin and Liao, Wei-Hsiang and Murata, Naoki and Takida, Yuhta and Uesaka, Toshimitsu and He, Yutong and Mitsufuji, Yuki and Ermon, Stefano},
  journal={arXiv preprint arXiv:2310.02279},
  year={2023}
}

@article{yang2024consistency,
  title={Consistency flow matching: Defining straight flows with velocity consistency},
  author={Yang, Ling and Zhang, Zixiang and Zhang, Zhilong and Liu, Xingchao and Xu, Minkai and Zhang, Wentao and Meng, Chenlin and Ermon, Stefano and Cui, Bin},
  journal={arXiv preprint arXiv:2407.02398},
  year={2024}
}

@inproceedings{meng2023distillation,
  title={On distillation of guided diffusion models},
  author={Meng, Chenlin and Rombach, Robin and Gao, Ruiqi and Kingma, Diederik and Ermon, Stefano and Ho, Jonathan and Salimans, Tim},
  booktitle={Proceedings of the IEEE/CVF conference on computer vision and pattern recognition},
  pages={14297--14306},
  year={2023}
}

@inproceedings{rombach2022high,
  title={High-resolution image synthesis with latent diffusion models},
  author={Rombach, Robin and Blattmann, Andreas and Lorenz, Dominik and Esser, Patrick and Ommer, Bj{\"o}rn},
  booktitle={Proceedings of the IEEE/CVF conference on computer vision and pattern recognition},
  pages={10684--10695},
  year={2022}
}

@article{ho2022video,
  title={Video diffusion models},
  author={Ho, Jonathan and Salimans, Tim and Gritsenko, Alexey and Chan, William and Norouzi, Mohammad and Fleet, David J},
  journal={Advances in neural information processing systems},
  volume={35},
  pages={8633--8646},
  year={2022}
}

@article{singer2022make,
  title={Make-a-video: Text-to-video generation without text-video data},
  author={Singer, Uriel and Polyak, Adam and Hayes, Thomas and Yin, Xi and An, Jie and Zhang, Songyang and Hu, Qiyuan and Yang, Harry and Ashual, Oron and Gafni, Oran and others},
  journal={arXiv preprint arXiv:2209.14792},
  year={2022}
}

@article{poole2022dreamfusion,
  title={Dreamfusion: Text-to-3d using 2d diffusion},
  author={Poole, Ben and Jain, Ajay and Barron, Jonathan T and Mildenhall, Ben},
  journal={arXiv preprint arXiv:2209.14988},
  year={2022}
}

@inproceedings{lin2023magic3d,
  title={Magic3d: High-resolution text-to-3d content creation},
  author={Lin, Chen-Hsuan and Gao, Jun and Tang, Luming and Takikawa, Towaki and Zeng, Xiaohui and Huang, Xun and Kreis, Karsten and Fidler, Sanja and Liu, Ming-Yu and Lin, Tsung-Yi},
  booktitle={Proceedings of the IEEE/CVF conference on computer vision and pattern recognition},
  pages={300--309},
  year={2023}
}

@article{lu2025dpm,
  title={Dpm-solver++: Fast solver for guided sampling of diffusion probabilistic models},
  author={Lu, Cheng and Zhou, Yuhao and Bao, Fan and Chen, Jianfei and Li, Chongxuan and Zhu, Jun},
  journal={Machine Intelligence Research},
  pages={1--22},
  year={2025},
  publisher={Springer}
}

@article{zhao2023unipc,
  title={Unipc: A unified predictor-corrector framework for fast sampling of diffusion models},
  author={Zhao, Wenliang and Bai, Lujia and Rao, Yongming and Zhou, Jie and Lu, Jiwen},
  journal={Advances in Neural Information Processing Systems},
  volume={36},
  pages={49842--49869},
  year={2023}
}

@article{zhang2022fast,
  title={Fast sampling of diffusion models with exponential integrator},
  author={Zhang, Qinsheng and Chen, Yongxin},
  journal={arXiv preprint arXiv:2204.13902},
  year={2022}
}

@article{salimans2022progressive,
  title={Progressive distillation for fast sampling of diffusion models},
  author={Salimans, Tim and Ho, Jonathan},
  journal={arXiv preprint arXiv:2202.00512},
  year={2022}
}

@article{li2025back,
  title={Back to basics: Let denoising generative models denoise},
  author={Li, Tianhong and He, Kaiming},
  journal={arXiv preprint arXiv:2511.13720},
  year={2025}
}

@article{albergo2022building,
  title={Building normalizing flows with stochastic interpolants},
  author={Albergo, Michael S and Vanden-Eijnden, Eric},
  journal={arXiv preprint arXiv:2209.15571},
  year={2022}
}

@article{lipman2022flow,
  title={Flow matching for generative modeling},
  author={Lipman, Yaron and Chen, Ricky TQ and Ben-Hamu, Heli and Nickel, Maximilian and Le, Matt},
  journal={arXiv preprint arXiv:2210.02747},
  year={2022}
}

@article{liu2022flow,
  title={Flow straight and fast: Learning to generate and transfer data with rectified flow},
  author={Liu, Xingchao and Gong, Chengyue and Liu, Qiang},
  journal={arXiv preprint arXiv:2209.03003},
  year={2022}
}

@inproceedings{esser2024scaling,
  title={Scaling rectified flow transformers for high-resolution image synthesis},
  author={Esser, Patrick and Kulal, Sumith and Blattmann, Andreas and Entezari, Rahim and M{\"u}ller, Jonas and Saini, Harry and Levi, Yam and Lorenz, Dominik and Sauer, Axel and Boesel, Frederic and others},
  booktitle={Forty-first international conference on machine learning},
  year={2024}
}

@inproceedings{gao2025diffusion,
  title={Diffusion Models and Gaussian Flow Matching: Two Sides of the Same Coin},
  author={Gao, Ruiqi and others},
  booktitle={The Fourth Blogpost Track at ICLR 2025},
  year={2025}
}

@article{karras2022elucidating,
  title={Elucidating the design space of diffusion-based generative models},
  author={Karras, Tero and Aittala, Miika and Aila, Timo and Laine, Samuli},
  journal={Advances in neural information processing systems},
  volume={35},
  pages={26565--26577},
  year={2022}
}

@article{williams2024score,
  title={Score-optimal diffusion schedules},
  author={Williams, Christopher and Campbell, Andrew and Doucet, Arnaud and Syed, Saifuddin},
  journal={Advances in Neural Information Processing Systems},
  volume={37},
  pages={107960--107983},
  year={2024}
}

@inproceedings{xue2024accelerating,
  title={Accelerating diffusion sampling with optimized time steps},
  author={Xue, Shuchen and Liu, Zhaoqiang and Chen, Fei and Zhang, Shifeng and Hu, Tianyang and Xie, Enze and Li, Zhenguo},
  booktitle={Proceedings of the IEEE/CVF Conference on Computer Vision and Pattern Recognition},
  pages={8292--8301},
  year={2024}
}

@inproceedings{li2023autodiffusion,
  title={Autodiffusion: Training-free optimization of time steps and architectures for automated diffusion model acceleration},
  author={Li, Lijiang and Li, Huixia and Zheng, Xiawu and Wu, Jie and Xiao, Xuefeng and Wang, Rui and Zheng, Min and Pan, Xin and Chao, Fei and Ji, Rongrong},
  booktitle={Proceedings of the IEEE/CVF International Conference on Computer Vision},
  pages={7105--7114},
  year={2023}
}

@article{hu2024adaflow,
  title={Adaflow: Imitation learning with variance-adaptive flow-based policies},
  author={Hu, Xixi and Liu, Qiang and Liu, Xingchao and Liu, Bo},
  journal={Advances in Neural Information Processing Systems},
  volume={37},
  pages={138836--138858},
  year={2024}
}

@article{sabour2024align,
  title={Align your steps: Optimizing sampling schedules in diffusion models},
  author={Sabour, Amirmojtaba and Fidler, Sanja and Kreis, Karsten},
  journal={arXiv preprint arXiv:2404.14507},
  year={2024}
}

@article{park2024textit,
  title={Jump Your Steps: Optimizing Sampling Schedule of Discrete Diffusion Models},
  author={Park, Yong-Hyun and Lai, Chieh-Hsin and Hayakawa, Satoshi and Takida, Yuhta and Mitsufuji, Yuki},
  journal={arXiv preprint arXiv:2410.07761},
  year={2024}
}

@article{10.5555/3722577.3722688,
author = {Kr\"{a}mer, Nicholas and Hennig, Philipp},
title = {Stable implementation of probabilistic ODE solvers},
year = {2024},
issue_date = {January 2024},
publisher = {JMLR.org},
volume = {25},
number = {1},
issn = {1532-4435},
journal = {J. Mach. Learn. Res.},
month = jan,
articleno = {111},
numpages = {29},
}

@article{armijo1966minimization,
  title   = {Minimization of functions having {L}ipschitz continuous first partial derivatives},
  author  = {Armijo, Larry},
  journal = {Pacific Journal of Mathematics},
  volume  = {16},
  number  = {1},
  pages   = {1--3},
  year    = {1966}
}

@article{wolfe1969convergence,
  title={Convergence conditions for ascent methods},
  author={Wolfe, Philip},
  journal={SIAM Review},
  volume={11},
  number={2},
  pages={226--235},
  year={1969},
  publisher={SIAM}
}

@article{wolfe1971convergence,
  title={Convergence conditions for ascent methods. II: Some corrections},
  author={Wolfe, Philip},
  journal={SIAM Review},
  volume={13},
  number={2},
  pages={185--188},
  year={1971},
  publisher={SIAM}
}

@inproceedings{karras2024analyzing,
  title={Analyzing and improving the training dynamics of diffusion models},
  author={Karras, Tero and Aittala, Miika and Lehtinen, Jaakko and Hellsten, Janne and Aila, Timo and Laine, Samuli},
  booktitle={2024 IEEE/CVF Conference on Computer Vision and Pattern Recognition (CVPR)},
  pages={24174--24184},
  year={2024},
  organization={IEEE}
}

@inproceedings{podell2024sdxl,
  title={Sdxl: Improving latent diffusion models for high-resolution image synthesis},
  author={Podell, Dustin and English, Zion and Lacey, Kyle and Blattmann, Andreas and Dockhorn, Tim and M{\"u}ller, Jonas and Penna, Joe and Rombach, Robin},
  booktitle={International Conference on Learning Representations},
  volume={2024},
  pages={1862--1874},
  year={2024}
}

@inproceedings{lin2014microsoft,
  title={Microsoft coco: Common objects in context},
  author={Lin, Tsung-Yi and Maire, Michael and Belongie, Serge and Hays, James and Perona, Pietro and Ramanan, Deva and Doll{\'a}r, Piotr and Zitnick, C Lawrence},
  booktitle={European conference on computer vision},
  pages={740--755},
  year={2014},
  organization={Springer}
}

@article{labs2025flux,
  title={Flux. 1 kontext: Flow matching for in-context image generation and editing in latent space},
  author={Labs, Black Forest and Batifol, Stephen and Blattmann, Andreas and Boesel, Frederic and Consul, Saksham and Diagne, Cyril and Dockhorn, Tim and English, Jack and English, Zion and Esser, Patrick and others},
  journal={arXiv preprint arXiv:2506.15742},
  year={2025}
}
\bibliographystyle{iclr2027_conference}

\newpage
\appendix

\section{Additional Preliminaries}
\subsection{Parameterizations}
Since the introduction of diffusion models, various parameterizations of the model output have been proposed. $x$-prediction and $\epsilon$-prediction formulation has first been introduced \citep{ho2020denoising}, where the later expressed as a reparameterization of the original training objective. Salimans and Ho \citep{salimans2022progressive} further proposed $v$-prediction, where the model predicts a linear combination of data and noise, improving stability and sample quality. Flow-based generative models \citep{albergo2022building, lipman2022flow, liu2022flow} can also be interpreted within the $v$-prediction model framework, and recent work has actively explored the connection between diffusion models and flow-based formulations \citep{esser2024scaling, lipman2022flow, gao2025diffusion}.

Despite the popularity of alternative parameterizations, recent work \citep{li2025back} has re-emphasized the importance of $x$-prediction, stating that directly predicting the clean image is more suitable for modeling low-dimensional manifolds embedded in high-dimensional space. This motivates us to reconsider the design space of EDM \citep{karras2022elucidating}, which establishes a strong baseline by incorporating $x$-prediction formulation for training and sampling with additional pre-conditioning scheme. 

\section{Derivation of ODE curvature}
\label{appendix:ode_derivation}
\textbf{General PF-ODE second-order derivative.}
We consider the probability flow ODE (PF-ODE) formulation of diffusion models.
Following \citep{song2020score}, the PF-ODE can be written as
\begin{equation}
dx
=
\left[
\frac{\dot{s}}{s}x
- s^2 \dot{\sigma}\sigma
\nabla_x \log p\!\left(\frac{x}{s};\sigma\right)
\right] dt .
\label{eq:pf_ode_general}
\end{equation}

Let $y=x/s$. Under the $x$-prediction parameterization,
\begin{equation}
\nabla_y \log p(y;\sigma)
=
\frac{D_\theta(y;\sigma)-y}{\sigma^2}.
\end{equation}

Hence, using the chain rule, we obtain
\begin{equation}
\nabla_x \log p(y;\sigma) =
\frac{1}{s}\nabla_y \log p(y;\sigma).
\end{equation}
which yields the equivalent ODE
\begin{equation}
\dot{x}
=
\frac{\dot{s}}{s}x
+
\frac{\dot{\sigma}}{\sigma}
\left(x - s D_\theta(x/s;\sigma)\right).
\label{eq:pf_ode_xpred}
\end{equation}

Defining
\begin{equation}
\label{eq:def_epsilon}
A = \frac{\dot{s}}{s}, \quad \epsilon_{\theta} := \frac{x - sD_{\theta}(x/s;\sigma)}{\sigma},
\end{equation}
we obtain the following
\begin{equation}
\label{eq:firstorder_eq}
\dot{x} = Ax + \dot{\sigma}\epsilon_{\theta}.
\end{equation}

\textbf{Derivation of $\dot{\epsilon}_{\theta}$.}
Recall from \cref{eq:def_epsilon} that
\begin{equation}
\epsilon_\theta = \frac{x-sD_\theta(x/s,\sigma)}{\sigma}, 
\qquad
N := x-sD_\theta(x/s,\sigma),
\end{equation}
so that $\epsilon_\theta = N/\sigma$. Differentiating yields
\begin{equation}
\dot{\epsilon}_{\theta}
= \frac{\dot{N}\sigma - N\dot{\sigma}}{\sigma^2}
= \frac{\dot N}{\sigma} - \frac{\dot\sigma}{\sigma}\epsilon_\theta .
\label{eq:eps_dot_start}
\end{equation}
Meanwhile, due to the definition of an auxiliary variable $y$,
\begin{equation}
\dot{y} = \frac{1}{s}\left(\dot{x}-\frac{\dot{s}}{s}x\right)
= \frac{\dot{\sigma}}{s}\epsilon_\theta,
\end{equation}
where the second inequality holds from $\dot{x} = \frac{\dot{s}}{s}x + \dot{\sigma}\epsilon_\theta$. 

Therefore, differentiating $N=x-sD_\theta$ gives
\begin{equation}
\begin{aligned}
\dot{N}
&= \dot{x} - \dot{s}D_\theta - s\dot{D}_\theta \\
&= \dot{x} - \dot{s}D_\theta - s\left(J_y\dot{y} + D_{\sigma}\dot{\sigma}\right) \\
&= \dot{x} - \dot{s}D_\theta - sJ_x\dot{\sigma}\epsilon_{\theta} -sD_{\sigma}\dot{\sigma},
\label{eq:N_dot}
\end{aligned}
\end{equation}
where $J_x := \nabla_x D_\theta(x/s;\sigma)$ and $D_\sigma := \partial_\sigma D_\theta(x/s;\sigma)$.

Substituting \cref{eq:N_dot} into \cref{eq:eps_dot_start} yields
\begin{equation}
\begin{aligned}
\dot{\epsilon}_{\theta}
&= \frac{\dot{x} - \dot{s}D_\theta - sJ_x\dot{\sigma}\epsilon_{\theta} -sD_{\sigma}\dot{\sigma}}{\sigma}
- \frac{\dot{\sigma}}{\sigma}\epsilon_{\theta} \\
&=
\frac{\left(\frac{\dot{s}}{s}x+\dot{\sigma}\epsilon_\theta\right)
-\dot{s}D_\theta
-sJ_x\dot{\sigma}\epsilon_\theta
-sD_\sigma\dot{\sigma}}{\sigma}
-
\frac{\dot{\sigma}}{\sigma}\epsilon_\theta \\
&=
\frac{\frac{\dot{s}}{s}(x-sD_\theta)}{\sigma}
+
\frac{\dot{\sigma}}{\sigma}\epsilon_\theta
-
\frac{s\dot{\sigma}}{\sigma}J_x\epsilon_\theta
-
\frac{s\dot{\sigma}}{\sigma}D_\sigma
-
\frac{\dot{\sigma}}{\sigma}\epsilon_\theta \\
&=
\frac{\dot{s}}{s}\epsilon_\theta
-
\frac{s\dot{\sigma}}{\sigma}J_x\epsilon_\theta
-
\frac{s\dot{\sigma}}{\sigma}D_\sigma.
\end{aligned}
\end{equation}

\textbf{Derivation of $\ddot{x}$.} 
Taking the time derivative yields the second-order ODE 
\begin{align}
\ddot{x} &= \dot{A}x + A\dot{x} + \ddot{\sigma}\epsilon_{\theta} + \dot{\sigma}\dot{\epsilon}_{\theta} \\ &= (\dot{A} + A^2)x + (A\dot{\sigma} + \ddot{\sigma})\epsilon_{\theta} + \dot{\sigma}\dot{\epsilon}_{\theta},
\end{align}
where the first-order derivative term is replaced using \cref{eq:firstorder_eq}. 

Using $A=\dot{s}/s$, we have
\begin{equation}
\dot{A} + A^2 = \frac{\ddot{s}}{s}, \quad A\dot{\sigma} + \ddot{\sigma} = \ddot{\sigma} + \dot{\sigma}\frac{\dot{s}}{s}.
\end{equation}

Substituting all terms, we obtain the general curvature expression
\begin{equation}
\label{eq:general_curvature}
\ddot{x} = \frac{\ddot{s}}{s}x + \left(\ddot{\sigma}+2\dot{\sigma}\frac{\dot{s}}{s}\right)\epsilon_\theta - \frac{s\dot{\sigma}^2}{\sigma}J_x\epsilon_\theta - \frac{s\dot{\sigma}^2}{\sigma}D_\sigma.
\end{equation}

\subsection{EDM formulation}
For EDM formulation \citep{karras2022elucidating}, we set 
\begin{equation}
\sigma(t) = t, \quad s(t) = 1.
\end{equation}
Note that although EDM discretizes sampling timesteps based on $\rho$-schedule in the noise domain, the underlying probability flow ODE is parameterized directly by the noise level. Hence, 
\begin{equation}
\dot{\sigma} = 1, \quad \ddot{\sigma} = 0, \quad \dot{s} = \ddot{s} = 0.
\end{equation}
Substituting into \cref{eq:general_curvature}, the second derivative simplifies to 
\begin{equation}
\ddot{x} = - \frac{1}{\sigma}(J_x\epsilon_{\theta} + D_{\sigma}) = - \frac{1}{\sigma^2}J_x(x-D_{\theta}) - \frac{1}{\sigma} D_{\sigma}.
\end{equation}

\subsection{VP formulation}
For the VP formulation, we use the parameterization
\begin{equation}
\label{eq:vp_formulation}
\sigma(t)=\sqrt{e^{u(t)}-1},\quad
s(t)=\frac{1}{\sqrt{1+\sigma(t)^2}},
\quad
u(t)=\frac{1}{2}\beta_d t^2+\beta_{\min}t .
\end{equation}
For simplicity, we define
\begin{equation}
\label{eq:first_order_u}
B(t):=\dot u(t)=\beta_{\min}+\beta_d t,
\quad
\dot B(t)=\beta_d .
\end{equation}
Since $1+\sigma(t)^2=e^{u(t)}$, substituting to \cref{eq:vp_formulation}, we have 
\begin{equation}
s(t)=\frac{1}{\sqrt{1+\sigma(t)^2}}=e^{-u(t)/2}.
\end{equation}

\textbf{Derivatives of $\sigma(t)$.}
By differentiating $\sigma(t)=\sqrt{e^{u(t)}-1}$, we obtain
\begin{equation}
\label{eq:first_order_sigma}
\dot\sigma
=\frac{e^{u}\dot u}{2\sqrt{e^{u}-1}}
=\frac{(1+\sigma^2)B}{2\sigma}
=\frac{1}{2} B\left(\sigma+\sigma^{-1}\right).
\end{equation}
Hence, we derive the second derivative of $\sigma(t)$ as 
\begin{align}
\ddot\sigma
&=\frac{1}{2}\dot B(\sigma+\sigma^{-1})
+\frac{1}{2} B\frac{d}{dt}(\sigma+\sigma^{-1}) \notag\\
&=\frac{1}{2}\beta_d(\sigma+\sigma^{-1})
+\frac{1}{2} B\dot\sigma(1-\sigma^{-2}).
\end{align}
Substituting $\dot\sigma$ from (46) gives
\begin{equation}
\label{eq:second_order_sigma}
\ddot\sigma
=\frac{1}{2}\beta_d(\sigma+\sigma^{-1})
+\frac{1}{4} B^2(\sigma-\sigma^{-3}).
\end{equation}

\textbf{Derivatives of $s(t)$.}
Using $s(t)=e^{-u(t)/2}$, we directly obtain
\begin{equation}
\label{eq:first_order_s}
\frac{\dot s}{s}=-\frac{1}{2}\dot u=-\frac{1}{2}B,
\end{equation}
hence
\begin{equation}
\dot s=-\frac{1}{2}Bs.
\end{equation}
Taking the second derivative gives
\begin{equation}
\ddot s = \frac{d}{dt}\left(-\frac{1}{2}Bs\right)
=-\frac{1}{2}\dot Bs-\frac{1}{2}B\dot s = -\frac{1}{2}\beta_d s+\frac{1}{4}B^2 s,
\end{equation}
In other words,
\begin{equation}
\frac{\ddot s}{s}=\frac{1}{4}B^2-\frac{1}{2}\beta_d .
\end{equation}

\textbf{Substitution into \cref{eq:general_curvature}.}
Substituting \cref{eq:first_order_s} to our general PF-ODE second order expression, we obtain the following terms as
\begin{equation}
\begin{aligned}
\ddot\sigma+2\dot\sigma\frac{\dot s}{s}
&=\ddot\sigma-B\dot\sigma = \frac{1}{2}\beta_d(\sigma+\sigma^{-1}) - \frac{1}{4}B^2(\sigma+2\sigma^{-1}+\sigma^{-3}), \\
\frac{s\dot{\sigma}^2}{\sigma} &= \frac{1}{4}B^2\frac{\sigma + 2\sigma^{-1} + \sigma^{-3}}{\sqrt{1+\sigma^2}}.
\end{aligned}
\end{equation}

As a result, under VP parameterization we get
\begin{equation}
\begin{aligned}
\ddot{x} &= \left(\frac{1}{4}B^2 - \frac{1}{2}\beta_d
\right)x + \left[\frac{1}{2}\beta_d(1+\sigma^{-2}) - \frac{1}{4}B^2 \left(1+2\sigma^{-2}+\sigma^{-4}\right)\right](x-sD_\theta) \\
&\quad - \frac{1}{4}B^2 \frac{\sigma+2\sigma^{-1}+\sigma^{-3}}{\sqrt{1+\sigma^2}} \left(J_x\epsilon_\theta + D_\sigma\right),
\end{aligned}
\end{equation}
where $B$, $\dot\sigma$, and $\ddot\sigma$ are defined in \cref{eq:first_order_u}, \cref{eq:first_order_sigma}, and \cref{eq:second_order_sigma}, respectively.

\subsection{VE formulation}
For VE formulation,
\begin{equation}
\sigma(t) = \sqrt{t}, \quad s(t) = 1,
\end{equation}
Hence,
\begin{equation}
\dot{\sigma} = \frac{1}{2\sqrt{t}} = \frac{1}{2\sigma}, \quad 
\ddot{\sigma} = \frac{d}{dt}(\frac{1}{2\sigma}) = \frac{d}{d\sigma}(\frac{1}{2\sigma})\dot{\sigma} = -\frac{1}{4\sigma^3}, \quad \dot{s} = \ddot{s} = 0.
\end{equation}
Substituting into \eqref{eq:general_curvature}, we obtain
\begin{equation}
\ddot{x} = - \frac{1}{4\sigma^3}(\epsilon_{\theta} + J_x\epsilon_{\theta} + D_{\sigma}) = - \frac{1}{4\sigma^4}(I+J_x)(x-D_{\theta}) - \frac{1}{4\sigma^3} D_{\sigma}.
\end{equation}

\section{Justification of the Proxy Curvature Estimator $\widehat{\kappa}_{\mathrm{rel}}(t_i)$}
\label{sec:proxy_estimator}
Given the following definition of relative local curvature $\kappa_{\mathrm{rel}}(t_i)$ and its corresponding estimator $\widehat{\kappa}_{\mathrm{rel}}(t_i)$:
\begin{equation}
\kappa_{\mathrm{rel}}(t_i) := \frac{\|v_{t_{i+1}}-v_{t_i}\|}{\Delta t_i \|v_{t_i}\|},
 \quad {\widehat{\kappa}}_{\mathrm{rel}}(t_i) := \frac{\|v_{t_i}-v_{t_{i-1}}\|}{\Delta \hat{t}_i \|v_{t_{i-1}}\|},
\end{equation}
we would like to prove the following equation
\begin{equation}
\kappa_{\mathrm{rel}}(t_{i-1}) = \widehat{\kappa}_{\mathrm{rel}}(t_i).
\end{equation}
In other words, $\widehat{\kappa}_{\mathrm{rel}}(t_i)$ is a one-step delayed evaluation of $\kappa_{\mathrm{rel}}(t_i)$, hence being a valid estimator.

Specifically if we consider the case $S_{\text{churn}}=0$ following the notation in \citep{karras2022elucidating}, no intermediate perturbation is introduced, hence the sampler evaluates the drift term exactly at the grid points, i.e.,
\begin{equation}
\hat{x}_{t_i} = x_{t_i}, \quad \hat{t}_i = t_i.
\end{equation}
Note that $(\hat{x}_{t_i},\hat{t}_i)$ denotes the actual evaluation point used by the EDM sampler
to compute the velocity $v_{t_i}$. Hence, the corresponding velocity field evaluations are numerically
identical:
\begin{equation}
v(\hat{x}_{t_i},\hat{t}_i) = v(x_{t_i},t_i) = v_{t_i}.
\end{equation}
Moreover, since 
\begin{equation}
\Delta \hat{t}_i = t_i - t_{i+1} =  \Delta t_i
\end{equation}
the effective step size also remains unchanged, hence the proof.

\section{Additional Proofs}
\subsection{Proof of \cref{thm:stepsize_bound}}
\label{appendix:stepsize_bound}
We provide a detailed proof of \cref{thm:stepsize_bound}, which establishes a sufficient condition on step size $\Delta t$ such that the Wasserstein distance between the true diffusion trajectory and its numerical approximation remains bounded. Let $\Phi_{t_i \to t_{i+1}}^{*}$ denote the exact PF-ODE transition map and $\Phi_{t_i \to t_{i+1}}^{\mathsf{E}}$ the Euler approximation. For a given state $x_{t_i}$, the exact and numerical updates are given by
\begin{equation}
x_{t_{i+1}}^{*} = \Phi_{t_i \to t_{i+1}}^{*}(x_{t_i}), 
\quad
x_{t_{i+1}}^{\mathsf{E}} = \Phi_{t_i \to t_{i+1}}^{\mathsf{E}}(x_{t_i}).
\end{equation}
Using the integral form of the ODE, we write
\begin{equation}
\Phi_{t_i \to t_{i+1}}^{*}(x_{t_i})
= x_{t_i} + \int_{t_i}^{t_{i+1}} v(x_{\tau}, \tau)\, d\tau,
\end{equation}
and the Euler update as
\begin{equation}
\label{eq:euler_prediction}
\Phi_{t_i \to t_{i+1}}^{\mathsf{E}}(x_{t_i})
= x_{t_i} - \Delta t_i \, v(x_{t_i}, t_i).
\end{equation}
The resulting discretization error is
\begin{equation}
\label{eq:discretization_error}
\Phi_{t_i \to t_{i+1}}^{*}(x_{t_i}) - \Phi_{t_i \to t_{i+1}}^{\mathsf{E}}(x_{t_i})
= \int_{t_i}^{t_{i+1}} \bigl( v(x_{\tau}, \tau) - v(x_{t_i}, t_i) \bigr)\, d\tau.
\end{equation}
Meanwhile, we define the auxiliary function
\begin{equation}
\label{eq:auxiliary_fn}
g(\tau) := v(x_{\tau},\tau).    
\end{equation}
Then, by the fundamental theorem of calculus,
\begin{equation}
g(\tau) - g(t_i) = \int_{t_i}^{\tau}g'(s)\, ds,
\end{equation}
which implies the bound
\begin{equation}
\|g(\tau)-g(t_i)\| \leq \left|\tau-t_i \right|\sup_{s\in[t_{i+1},t_i]}\|g'(s)\|.
\end{equation}
Substituting this bound into \cref{eq:discretization_error} yields 
\begin{equation}
\label{eq:euler_bound}
\begin{aligned}
\left\|\Phi_{t_i \to t_{i+1}}^{*}(x_{t_i}) - \Phi_{t_i \to t_{i+1}}^{\mathsf{E}}(x_{t_i})\right\| &\leq \left|\int_{t_i}^{t_{i+1}}(\tau-t_i)\,d\tau\right|\sup_{s\in[t_{i+1},t_i]}\left\|\frac{d}{ds}v(x_s,s)\right\| \\
&= \frac{\Delta t_i^2}{2} 
\sup_{s\in[t_{i+1},t_i]}\left\|\frac{d}{ds}v(x_s,s)\right\|.
\end{aligned}
\end{equation}

We now derive the discretization error in relation to the 2-Wasserstein distance. Recall the definition of Wasserstein distance between two probability measures $\mu$ and $\nu$:
\begin{equation}
W_2(\mu,\nu)^2 = \inf_{\pi \in \Pi(\mu,\nu)} \mathbb{E}_{(X,Y) \sim \pi}\left[\|X-Y\|^2\right].
\end{equation}
Let $x_t \sim p_t$ denote the random variable following the diffusion trajectory at time $t$. We then define the  coupled random variables 
\begin{equation}
x_{t_{i+1}}^{*} = \Phi_{t_i \to t_{i+1}}^{*}(x_{t_i}), 
\quad
x_{t_{i+1}}^{\mathsf{E}} = \Phi_{t_i \to t_{i+1}}^{\mathsf{E}}(x_{t_i}).
\end{equation}
where $x_{t_{i+1}}^*, x_{t_{i+1}}^\mathsf{E}$ corresponds to the exact and Euler-integrated flows, respectively. Due to the definition of Wasserstein distance, we obtain
\begin{equation}
\label{eq:wasserstein_bound}
 W_2\left(p_{t_{i+1}}^*,p_{t_{i+1}}^\mathsf{E}\right) \leq \left(\mathbb{E}\left[\left\|x_{t_{i+1}}^* - x_{t_{i+1}}^\mathsf{E}\right\|^2\right]\right)^{1/2} \\ \leq \frac{\Delta t_i^2}{2} \left(\mathbb{E}\left[\sup_{s\in[t_{i+1},t_i]}\left\|\frac{d}{ds}v(x_s^*,s)\right\|^2\right]\right)^{1/2}.
\end{equation}
Therefore, requiring the Wasserstein distance to be bounded by a specified tolerance $\eta$,
\begin{equation}
W_2\left(p_{t_{i+1}}^*,p_{t_{i+1}}^\mathsf{E}\right) \le \eta,
\end{equation}
leads to the following sufficient condition on the step size:
\begin{equation}
\Delta t_i \leq \sqrt{\frac{2\eta}{S_{i}}},
\end{equation}
where
\begin{equation}
S_{i} = \left(\mathbb{E}\left[\sup_{s\in[t_{i+1},t_i]}\left\|\frac{d}{ds}v(x_s^*,s)\right\|^2\right]\right)^{1/2}.
\end{equation}
Hence, the proof.

\subsection{Extension of \cref{thm:stepsize_bound} for second-order methods}
\label{appendix:stepsize_bound_heun}
We extend the analysis to the second-order Heun method. Let $\Phi_{t_i \to t_{i+1}}^{*}$ denote the exact PF-ODE transition, $\Phi_{t_i \to t_{i+1}}^{\mathsf{E}}$ the Euler predictor, and $\Phi_{t_i \to t_{i+1}}^{\mathsf{H}}$ the Heun update. For a given state $x_{t_i}$, the Heun update is defined as
\begin{equation}
\Phi_{t_i \to t_{i+1}}^{\mathsf{H}}(x_{t_i})
=
x_{t_i} - \frac{\Delta t_i}{2} \left(
v(x_{t_i}, t_i) + v\left(\Phi_{t_i \to t_{i+1}}^{\mathsf{E}}(x_{t_i}),\, t_{i+1}\right)
\right).
\end{equation}
The discretization error can be written as
\begin{equation}
\label{eq:heun_discretization_error}
\begin{aligned}
\Phi_{t_i \to t_{i+1}}^{*}(x_{t_i}) - \Phi_{t_i \to t_{i+1}}^{\mathsf{H}}(x_{t_i})
&= \left(\int_{t_i}^{t_{i+1}} v(x_{\tau}, \tau)\, d\tau +  \frac{\Delta t_i}{2}\left[ v(x_{t_i},t_i) + v(x_{t_{i+1}}, t_{i+1}) \right]\right) \\
&\quad - \frac{\Delta t_i}{2}\left(v(x_{t_{i+1}}, t_{i+1}) - v\left(\Phi_{t_i \to t_{i+1}}^{\mathsf{E}}(x_{t_i}),\, t_{i+1}\right)\right).
\end{aligned}
\end{equation}
For the first term, defining $g(\tau) := v(x_{\tau},\tau)$ as in \cref{eq:auxiliary_fn} and applying integration by parts twice, we obtain
\begin{equation}
\int_{t_i}^{t_{i+1}} g(\tau)\, d\tau + \frac{\Delta t_i}{2}\left[g(t_i) + g(t_{i+1})\right] = \frac{1}{2}\int_{t_i}^{t_{i+1}} (\tau - t_i)(t_{i+1} - \tau) g''(\tau)\, d\tau.
\end{equation}
Hence, 
\begin{equation}
\begin{aligned}
\left\|\int_{t_i}^{t_{i+1}}g(\tau)d\tau + \frac{\Delta t_i}{2}\left[g(t_i) + g(t_{i+1})\right] \right\| &\leq \frac{1}{2}\left|\int_{t_i}^{t_{i+1}} (\tau-t_i)(t_{i+1} - \tau)\, d\tau\right| \sup_{s \in [t_{i+1},t_i]} \|g''(s)\| \\ &= \frac{\Delta t_i^3}{12}\sup_{s \in [t_{i+1},t_i]} \|g''(s)\|.
\end{aligned}
\end{equation}
Meanwhile, assuming that $v(\cdot, t_{i+1})$ is $L$-Lipschitz,
\begin{equation}
\left\|v(x_{t_{i+1}}, t_{i+1}) - v\left(\Phi_{t_i \to t_{i+1}}^{\mathsf{E}}(x_{t_i}), t_{i+1}\right)\right\|
\le
L \left\|\Phi_{t_i \to t_{i+1}}^{*}(x_{t_i}) - \Phi_{t_i \to t_{i+1}}^{\mathsf{E}}(x_{t_i})\right\|.
\end{equation}
Using the Euler error bound from Theorem~3.2 and substituting into \cref{eq:heun_discretization_error} yields
\begin{equation}
\left\|\Phi_{t_i \to t_{i+1}}^{*}(x_{t_i}) - \Phi_{t_i \to t_{i+1}}^{\mathsf{H}}(x_{t_i})\right\|
\le 
\Delta t_i^3 \left(\frac{1}{12} \sup_{s \in [t_{i+1},t_i]}\|g''(s)\| + \frac{L}{4} \sup_{s \in [t_{i+1},t_i]} \|g'(s)\|\right).
\end{equation}
Finally, let $x_t \sim p_t$ and define
\begin{equation}
x_{t_{i+1}}^{*} = \Phi_{t_i \to t_{i+1}}^{*}(x_{t_i}),
\quad
x_{t_{i+1}}^{\mathsf{H}} = \Phi_{t_i \to t_{i+1}}^{\mathsf{H}}(x_{t_i}).
\end{equation}
Then, the Wasserstein distance bound can be derived as 
\begin{equation}
W_2\left(p_{t_{i+1}}^*,p_{t_{i+1}}^\mathsf{H}\right) \leq \left(\mathbb{E}\left[\left\|x_{t_{i+1}}^* - x_{t_{i+1}}^\mathsf{H}\right\|^2\right]\right)^{1/2} \\ \leq \Delta t_i^3 \left(\frac{Q_{i}}{12} + \frac{LS_{i}}{4}\right),
\end{equation}
where 
\begin{equation}
S_{i} = \left(\mathbb{E}\left[\sup_{s\in[t_{i+1},t_i]}\left\|\frac{d}{ds}v(x_{s}^*,s)\right\|^2\right]\right)^{1/2}, \quad Q_{i} = \left(\mathbb{E}\left[\sup_{s\in[t_{i+1},t_i]}\left\|\frac{d^2}{ds^2}v(x_{s}^*,s)\right\|^2\right]\right)^{1/2}.
\end{equation}

\subsection{Proof of Global Wasserstein Error Bound for Adaptive Solvers}
\label{appendix:mixed_solver_accounting}
\begin{proposition}[Total Wasserstein Error Bound for Adaptive Solvers]
\label{prop:mixed_solver_accounting} 
Let $\Phi_{t_i \to t_{i+1}}^{\mathsf{ODE}}$ be the selected numerical transition map, $(\cdot)_\#$ the pushforward operator, and $\mathcal P_2(\mathcal Z_i)$ the set of probability measures with finite second moment. Assume that each selected numerical transition is Wasserstein-stable, i.e.,
\begin{equation}
W_2 \left( (\Phi_{t_i \to t_{i+1}}^{\mathsf{ODE}})_\#\alpha,\, (\Phi_{t_i \to t_{i+1}}^{\mathsf{ODE}})_\#\beta \right) \le
\Lambda_i W_2(\alpha,\beta), \quad \forall \alpha,\beta\in\mathcal P_2(\mathcal Z_i),
\end{equation}
for some $\Lambda_i < \infty$. Then, the total Wasserstein error satisfies
\begin{equation}
W_2 \left( (\pi_N)_\#\mu_N,\, (\pi_N)_\#\nu_N \right) \le
e_0\prod_{j=0}^{N-1}\Lambda_j + \sum_{i=0}^{N-1} \eta_i \prod_{j=i+1}^{N-1}\Lambda_j,
\label{eq:mixed_solver_accounting_bound}
\end{equation}
where $e_0:=W_2(\mu_0,\nu_0)$ and $\mu_i$ and $\nu_i$ denote the distributions obtained by applying the exact and numerical transition maps up to timestep $t_i$, respectively.
\end{proposition}

We first introduce the necessary notation. For one-step solvers, the state is defined as $z_{t_i}=x_{t_i}$. For finite-memory solvers, we define $z_{t_i} = (x_{t_i}, r_{t_i}) \in \mathcal Z_i$, where $r_{t_i}$ denotes the solver history. We assume that each $\mathcal Z_i$ is a Polish space equipped with a metric, all measures belong to $\mathcal P_2(\mathcal Z_i)$, and all transition maps are measurable. The projection onto the sample coordinate is denoted by $\pi_i(z_{t_i}) = x_{t_i}$, which is assumed to be $1$-Lipschitz.

Let $\Phi_{t_i \to t_{i+1}}^{*}:\mathcal Z_i\to\mathcal Z_{i+1}$ denote the lifted exact PF-ODE transition and $\Phi_{t_i \to t_{i+1}}^{\mathsf{ODE}}:\mathcal Z_i\to\mathcal Z_{i+1}$ the selected numerical transition. We define
\begin{equation}
\mu_{i+1}=(\Phi_{t_i \to t_{i+1}}^{*})_\#\mu_i, \quad \nu_{i+1}=(\Phi_{t_i \to t_{i+1}}^{\mathsf{ODE}})_\#\nu_i,
\end{equation}
and the local transition error as
\begin{equation}
\eta_i := W_2 \left( (\Phi_{t_i \to t_{i+1}}^{*})_\#\mu_i,\, (\Phi_{t_i \to t_{i+1}}^{\mathsf{ODE}})_\#\mu_i \right).
\end{equation}
The lifted exact transition $\Phi_{t_i \to t_{i+1}}^{*}$ is chosen such that its sample marginal matches the exact PF-ODE flow. Namely, if $\phi_{t_i \to t_{i+1}}^{*}:\mathbb R^d\to\mathbb R^d$ denotes the exact PF-ODE transition from $t_i$ to $t_{i+1}$, then
\begin{equation}
\pi_{i+1}\circ \Phi_{t_i \to t_{i+1}}^{*} = \phi_{t_i \to t_{i+1}}^{*}\circ \pi_i .
\end{equation}
For finite-memory solvers, the lifted distribution $\mu_i$ stores the exact history variables, so that the local error term compares the numerical transition against the exact transition under exact input history.

\begin{lemma}[Wasserstein stability from Lipschitz continuity]
\label{lem:w2_lipschitz_pushforward}
Let $T:\mathcal Z \to \mathcal Z$ be $\Lambda$-Lipschitz. Then, for all $\alpha,\beta\in\mathcal P_2(\mathcal Z)$,
\begin{equation}
W_2(T_\#\alpha,T_\#\beta) \le \Lambda W_2(\alpha,\beta).
\end{equation}
\end{lemma}
\begin{proof}
Let $\gamma\in\Pi(\alpha,\beta)$ be any coupling. Then $(T,T)_\#\gamma$ is a coupling of $T_\#\alpha$ and $T_\#\beta$. Hence,
\begin{equation}
W_2^2(T_\#\alpha, T_\#\beta)
\le
\int \|T(z)-T(z')\|^2 \, d\gamma(z,z')
\le
\Lambda^2 \int \|z-z'\|^2 \, d\gamma(z,z').
\end{equation}
Taking the infimum over $\gamma\in\Pi(\alpha,\beta)$ proves the claim.
\end{proof}

\begin{proof}[Proof of Proposition~\ref{prop:mixed_solver_accounting}]
Define the local Wasserstein error at timestep $t_i$ as
\begin{equation}
e_i := W_2(\mu_i,\nu_i).
\end{equation}
Using
\begin{equation}
\mu_{i+1}=(\Phi_{t_i \to t_{i+1}}^{*})_\#\mu_i, \quad \nu_{i+1}=(\Phi_{t_i \to t_{i+1}}^{\mathsf{ODE}})_\#\nu_i,
\end{equation}
we have
\begin{equation}
\begin{aligned}
e_{i+1} &= W_2 \left( (\Phi_{t_i \to t_{i+1}}^{*})_\#\mu_i,\, (\Phi_{t_i \to t_{i+1}}^{\mathsf{ODE}})_\#\nu_i \right) \\
&\leq W_2 \left( (\Phi_{t_i \to t_{i+1}}^{*})_\#\mu_i,\, (\Phi_{t_i \to t_{i+1}}^{\mathsf{ODE}})_\#\mu_i \right) + W_2 \left( (\Phi_{t_i \to t_{i+1}}^{\mathsf{ODE}})_\#\mu_i,\, (\Phi_{t_i \to t_{i+1}}^{\mathsf{ODE}})_\#\nu_i \right),
\end{aligned}
\end{equation}
where the first term is $\eta_i$, and the second term is bounded by the assumed Wasserstein stability of $\Phi^{\mathsf{ODE}}_{t_i \to t_{i+1}}$, hence
\begin{equation}
e_{i+1}\le \eta_i+\Lambda_i e_i .
\end{equation}
Unrolling this recursion yields
\begin{equation}
e_N \le e_0\prod_{j=0}^{N-1}\Lambda_j + \sum_{i=0}^{N-1} \eta_i \prod_{j=i+1}^{N-1}\Lambda_j .
\end{equation}
Finally, since $\pi_N$ is $1$-Lipschitz, applying Lemma~\ref{lem:w2_lipschitz_pushforward} with $T=\pi_N$ gives
\begin{equation}
W_2 \left( (\pi_N)_\#\mu_N,\, (\pi_N)_\#\nu_N \right) \le W_2(\mu_N,\nu_N) = e_N .
\end{equation}
Combining the two inequalities proves
\cref{eq:mixed_solver_accounting_bound}.
\end{proof}

\paragraph{Extensions and Practical Considerations.}
We provide several remarks to clarify how local error modeling and solver selection interact with the proposed Wasserstein error accounting framework.

\begin{remark}[Euler-proxy calibration versus solver-specific defects]
\label{rem:euler_proxy_vs_solver_defect}
The default SDM implementation calibrates the timestep grid using the Euler-based proxy
\begin{equation}
\hat{\eta}_i^\mathsf{E}=\frac{\Delta t_i^2}{2}\widehat S_{i} .
\end{equation}
This proxy is used as a computationally efficient signal to estimate local trajectory difficulty during offline schedule construction. However, it does not represent the true local error of a given solver. In the theoretical bound in \cref{eq:mixed_solver_accounting_bound}, the relevant quantity is the actual local defect $\eta_i^m$ of the selected solver $m$, when such a bound is available.
\end{remark}

\begin{remark}[Heun local transition error]
\label{rem:heun_local_defect}
For Heun, under standard smoothness assumptions, Appendix~\ref{appendix:stepsize_bound_heun} gives
\begin{equation}
\eta_i^\mathsf{H} = \Delta t_i^3 \left( \frac{Q_i}{12} + \frac{LS_i}{4} \right),
\end{equation}
where $S_i$ and $Q_i$ quantify first- and second-order variations along the trajectory. This example shows how solver-specific transition errors can be incorporated into Proposition~\ref{prop:mixed_solver_accounting}.
\end{remark}

\begin{remark}[Finite-memory samplers]
\label{rem:finite_memory_samplers}
Multistep and predictor-corrector samplers can be written as one-step maps by augmenting the state. The history variable $r_{t_i}$ may contain previous samples, denoiser outputs, model evaluations, and the corresponding integration variables. Once the update is written as
\begin{equation}
z_{t_{i+1}}=\Phi_{t_i \to t_{i+1}}^{\mathsf{ODE}}(z_{t_i}),
\end{equation}
Proposition~\ref{prop:mixed_solver_accounting} applies whenever this map satisfies the Wasserstein-stability condition. Solver-specific local error bounds, once computed, can be substituted directly into \cref{eq:mixed_solver_accounting_bound}.
\end{remark}

\begin{remark}[Timestep-dependent solver allocation]
\label{rem:adaptive_solver_sequence}
In practice, SDM employs different numerical solvers across timesteps according to the estimated trajectory curvature. 
As a result, the overall sampler is implemented as a sequence of transition maps, e.g., Euler updates in early stages and Heun updates in later stages. Hence, Proposition~\ref{prop:mixed_solver_accounting} remains valid under adaptive solver sequences, provided the selected transition map at each timestep satisfies the Wasserstein-stability condition.
\end{remark}

\subsection{Energy-Optimized Geodesic Path for Weighted Incremental Cost}
\label{prop:weighted_sum}
\begin{proposition}
Assume the incremental cost admits the quadratic form 
\begin{equation}
\label{eq:quadratic_form}
L(t_i,t_{i+1}) \approx {\delta}(t_i)\Delta t_i^2, \quad \Delta t_i := t_i - t_{i+1},
\end{equation}
for some function ${\delta}(\cdot)$.
Then the weighted cost can be written as 
\begin{equation}
\tilde{L}(t_i, t_{i+1}) \approx \tilde{\delta}(t_i)\Delta t_i^2, \quad \tilde{\delta}(t) := w(t){\delta}(t).
\end{equation}
Consequently, the optimal schedule $\{t_i^*\}_{i=0}^{N}$ for the weighted objective traverses the diffusion path a constant speed with respect to the weighted geodesic length
\begin{equation}
\tilde{\Gamma} := \sum_{i=0}^{M-1} \sqrt{\tilde{L}(t_i,t_{i+1})} = \sum_{i=0}^{M-1} \sqrt{w(t_i)}\sqrt{L(t_i,t_{i+1})},
\end{equation}
i.e., 
\begin{equation}
\sqrt{\tilde{L}(t_i^*, t_{i+1}^*)} \approx \frac{\tilde{\Gamma}}{N}.
\end{equation}
\end{proposition}
\begin{proof}
By \cref{eq:weighted_cost} and \cref{eq:quadratic_form}, 
\begin{align}
\tilde{L}(t_i, t_{i+1}) &= w(t_i)L(t_i, t_{i+1}) \\ &\approx w(t_i){\delta}(t_i)\Delta t_i^2 = \tilde{{\delta}}(t_i)\Delta t_i^2.
\end{align}
hence, the proof. The remainder follows by applying the same proof with ${\delta}$ replaced by $\tilde{{\delta}}$ \citep{williams2024score}.
\end{proof}

\subsection{Proof of Wasserstein Error Bound after $N$-step Resampling}
\label{appendix:error_bound_n_step_resampling}

From the quadratic form in
\cref{eq:quadratic_form}, the weighted incremental cost satisfies
\begin{equation}
\sqrt{\tilde{L}(t_i,t_{i+1})}
\approx
\sqrt{\tilde{\delta}(t_i)}\Delta t_i.
\end{equation}
We therefore define the continuous density-like function
\begin{equation}
\label{eq:psi_equation}
\psi(t) := \sqrt{\tilde{\delta}(t)}.
\end{equation}
The corresponding weighted cumulative geodesic length can then be expressed as
\begin{equation}
\tilde{\Gamma}
=
\int_{0}^{t_0}\psi(u)\,du.
\end{equation}
Since $\psi$ is continuous and monotone, the uniform discretization of
the cumulative weighted geodesic length in \cref{prop:weighted_sum}
implies that each resampled interval satisfies
\begin{equation}
\int_{t_{i+1}^*}^{t_i^*}\psi(u)\,du
=
\frac{\tilde{\Gamma}}{N},
\end{equation}
By the mean value theorem, there exists some
$\xi_i\in[t_{i+1}^*,t_i^*]$ such that
\begin{equation}
\psi(\xi_i)\Delta t_i^*
=
\frac{\tilde{\Gamma}}{N},
\end{equation}
and hence
\begin{equation}
\label{eq:resampled_dt}
\Delta t_i^*
=
\frac{\tilde{\Gamma}}
{N\psi(\xi_i)}.
\end{equation}

Next, from \cref{eq:local_error_proxy} and the definition of
$\psi$, the weighted local error satisfies
\begin{equation}
\begin{aligned}
\psi(t_i^*)^2(\Delta t_i^*)^2 = \tilde{L}(t_i^*,t_{i+1}^*)
=
w(t_i^*)\hat{\eta}_i^* 
=
\frac{1}{2}w(t_i^*)\widehat{S}_{t_i^*}
(\Delta t_i^*)^2
\end{aligned}
\end{equation}
Therefore,
\begin{equation}
\label{eq:psi_error_relation}
\psi(t_i^*)^2
=
\frac{1}{2}w(t_i^*)\widehat{S}_{t_i^*}.
\end{equation}

Substituting \cref{eq:resampled_dt} into the local error proxy gives
\begin{equation}
\begin{aligned}
\hat{\eta}_i^*
&=
\frac{1}{2}\widehat{S}_{t_i^*}
\left(
\frac{\tilde{\Gamma}}
{N\psi(\xi_i)}
\right)^2 \\
&=
\frac{\tilde{\Gamma}^2}{N^2}
\frac{\widehat{S}_{t_i^*}}
{2\psi(\xi_i)^2} 
=
\frac{\tilde{\Gamma}^2}
{N^2w(t_i^*)}
\left(
\frac{\psi(t_i^*)}{\psi(\xi_i)}
\right)^2,
\end{aligned}
\end{equation}
using \cref{eq:psi_error_relation}.
Finally, assuming that $\log\psi$ is $K$-Lipschitz, since
$\xi_i\in[t_{i+1}^*,t_i^*]$,
\begin{equation}
\left|
\log\frac{\psi(t_i^*)}{\psi(\xi_i)}
\right|
\leq
K|t_i^*-\xi_i|
\leq
K\Delta t_i^*.
\end{equation}
Consequently,
\begin{equation}
\left(
\frac{\psi(t_i^*)}{\psi(\xi_i)}
\right)^2
\leq
e^{2K\Delta t_i^*},
\end{equation}
and thus
\begin{equation}
\hat{\eta}_i^*
\leq
e^{2K\Delta t_i^*}
\frac{\tilde{\Gamma}^2}
{N^2w(t_i^*)}.
\end{equation}

Thus, the local
transition error after $N$-step resampling admits an upper bound of the
same form as the original Wasserstein-based error bound, up to a factor controlled by the local variation of $\psi$, showing that
the practical resampling procedure preserves the Wasserstein-based
local error interpretation.

\clearpage
\section{Further Implementation Details}
\label{sec:further_implementations}
Experiments and scheduling algorithms for unconditional and conditional generation are implemented using a single NVIDIA RTX 5090 GPU, while a single NVIDIA A100 80GB GPU is used for high-resolution generation (EDM2-M) and text-to-image generation (SDXL, Stable Diffusion 3, FLUX).

\subsection{Algorithm for Adaptive Scheduling}
\label{appendix:pseudo_algorithm}
We provide pseudocode for the overall adaptive scheduling algorithm. The scheduling is performed offline prior to inference, where $\eta$-scheduling and $N$-step resampling are implemented as a two-stage procedure. 
In practice, the offline scheduling algorithm requires only a single inference pass with a single batch of size 128. The calibration takes approximately 45s for CIFAR-10, 2 minutes for FFHQ and AFHQv2, and 10--30 minutes for the remaining experiments, including high-resolution and text-to-image generation settings.
To assure robustness, we evaluate the model stability across independent pilot runs; the resulting FID for CIFAR-10, FFHQ, and AFHQv2 shows a standard deviation of only 0.0015, 0.0019, and 0.0014, respectively, confirming that the offline scheduling method is largely insensitive to initial sample choices.

\begin{figure}[htbp]
\centering

\begin{minipage}[t]{0.49\linewidth}
\vspace{0pt}
\hrule
\vspace{0.3em}
$\eta$-scheduling
\vspace{0.3em}
\hrule
\vspace{0.4em}

\begin{algorithmic}[1]
\STATE \textbf{Input:} Initial $(\mathbf{x}_{t_0},t_0)$, model $\mathbf{v}_\phi$, error tolerance map $\eta$
\STATE $i \leftarrow 0$
\STATE $\mathbf{v}_{t_0} \leftarrow \mathbf{v}_\phi(\mathbf{x}_{t_0},t_0)$
\WHILE{$t_i > 0$}
    \STATE $\tilde{t}_{i+1} \leftarrow \textsc{NextTimestep}(t_i)$
    \REPEAT
        \STATE $\Delta t_i \leftarrow t_i-\tilde{t}_{i+1}$
        \STATE $\tilde{\mathbf{x}}_{t_{i+1}}
        \leftarrow
        \mathbf{x}_{t_i}-\Delta t_i\mathbf{v}_{t_i}$
        \STATE $\tilde{\mathbf{v}}_{t_{i+1}}
        \leftarrow
        \mathbf{v}_\phi(\tilde{\mathbf{x}}_{t_{i+1}},\tilde{t}_{i+1})$
        \STATE $\widehat{S}_{i}
        \leftarrow
        \frac{
        \|\tilde{\mathbf{v}}_{t_{i+1}}-\mathbf{v}_{t_i}\|
        }{\Delta t_i}$
        \COMMENT{\cref{eq:approx_adaptive_stepsize}}
        \STATE $\tilde{t}_{i+1}
        \leftarrow
        \textsc{LineSearch}(\widehat{S}_{i},\tilde{t}_{i+1},t_i)$
    \UNTIL{no further change in $\tilde{t}_{i+1}$}
    \STATE $\Delta t_i
    \leftarrow
    \sqrt{2\eta_i/\hat{S}_{i}}$
    \COMMENT{\cref{thm:stepsize_bound}}
    \STATE $t_{i+1}\leftarrow t_i-\Delta t_i$
    \STATE $\mathbf{x}_{t_{i+1}}
    \leftarrow
    \mathbf{x}_{t_i}-\Delta t_i\mathbf{v}_{t_i}$
    \STATE $\mathbf{v}_{t_{i+1}}
    \leftarrow
    \mathbf{v}_\phi(\mathbf{x}_{t_{i+1}},t_{i+1})$
    \STATE $i\leftarrow i+1$
\ENDWHILE
\STATE \textbf{Output:} Initial schedule $\{t_i\}$
\end{algorithmic}

\vspace{0.3em}
\hrule
\end{minipage}
\hfill
\begin{minipage}[t]{0.49\linewidth}
\vspace{0pt}

\hrule
\vspace{0.3em}
$N$-step resampling
\vspace{0.3em}
\hrule
\vspace{0.4em}

\begin{algorithmic}[1]
\STATE \textbf{Input:} Initial schedule $\{t_i\}$, weights $\{w_i\}_{i=0}^{M-1}$, cost function $\{L(t_i,t_{i+1})\}_{i=0}^{M-1}$
\FOR{$i \leftarrow 0$ to $M-1$}
    \STATE $\tilde{L}(t_i,t_{i+1}) \leftarrow w(t_i)L(t_i,t_{i+1})$
    \STATE $\tilde{\Gamma}(t_i) \leftarrow \sum_{j=0}^{i-1}\sqrt{\tilde{L}(t_j,t_{j+1})}$
\ENDFOR
\STATE $\Gamma^{-1}(\cdot) \leftarrow \textsc{Interpolate}(\{\tilde{\Gamma}(t_i)\}_{i=0}^{M-1})$
\FOR{$i \leftarrow 0$ to $N-1$}
    \STATE $t_i^* \leftarrow \Gamma^{-1}\!\left(\tilde{\Gamma}(t_{M-1})\frac{i}{N}\right)$
\ENDFOR
\STATE \textbf{Output:} $N$-step schedule $\{t_i^*\}$
\end{algorithmic}

\vspace{0.4em}
\hrule
\vspace{0.83em} 

\hrule
\vspace{0.3em}
Inference
\vspace{0.3em}
\hrule
\vspace{0.4em}

\begin{algorithmic}[1]
\STATE \textbf{Input:} Initial $(\mathbf{x}_{t_0^*},t_0^*)$, model $\mathbf{v}_{\phi}$, $N$-step schedule $\{t_i^*\}$
\FOR{$i \leftarrow 0$ to $N-1$}
    \STATE $\mathbf{x}_{t_{i+1}^*} \leftarrow \textsc{SolverStep}(\mathbf{x}_{t_i^*}, t^*_i; \mathbf{v}_{\phi})$
\ENDFOR
\STATE \textbf{Output:} $\mathbf{x}_{t_N^*}$
\end{algorithmic}

\vspace{0.3em}
\hrule
\end{minipage}
\end{figure}

\subsection{Hyperparameter Settings}

\cref{tab:tauk_grid} represents the search grids for hyperparameters used throughout our experiments. For the step schedule based adaptive solver, we vary the curvature threshold $\tau_{\kappa}$
across datasets and timestep schedules. We observe that VP and VE models exhibit similar performance under the same threshold configuration on standard benchmarks, with the exception of AFHQv2, and therefore share the same $\tau_{\kappa}$ values. For AFHQv2, we select $\tau_{\kappa} = 2 \times 10^{-4}$ for VP parameterization and $\tau_{\kappa} = 1 \times 10^{-3}$ for VE parameterization, which yields the best FID performance.

\begin{table}[htbp]
\centering
\caption{Parameter search grid for threshold $\tau_{\kappa}$ for step schedule based adaptive solver selection.}
\label{tab:tauk_grid}
\renewcommand{\arraystretch}{1.2}
\setlength{\tabcolsep}{5pt}

\resizebox{\linewidth}{!}{
    \begin{tabular}{|>{\centering\arraybackslash}p{0.12\textwidth}|c|c|c|c|>{\centering\arraybackslash}p{0.24\textwidth}|}
    \hline
    Timestep Schedule
    & CIFAR-10 & FFHQ & AFHQv2 & ImageNet & Grid search \\
    \hline
    \multirow{2}{*}{EDM}
    & \multirow{2}{*}{$2\times10^{-4}$} & \multirow{2}{*}{$1\times10^{-4}$} & \multirow{2}{*}{$1\times10^{-3}$} & \multirow{2}{*}{$1\times10^{-4}$}
    & \multirow{2}{*}{\makecell[c]{$\{2,5,10,20,$\\$50,100\}\times10^{-5}$}} \\
    & & & & & \\   
    \hline
    \multirow{2}{*}{SDM}
    & \multirow{2}{*}{$2\times10^{-4}$}
    & \multirow{2}{*}{$1\times10^{-4}$}
    & $2\times10^{-4} (\text{VP})$
    & \multirow{2}{*}{$1\times10^{-4}$}
    & \multirow{2}{*}{\makecell[c]{$\{2,5,10,20,$\\$50,100\}\times10^{-5}$}} \\
    &  &  & $1\times10^{-3} (\text{VE})$ &  &  \\
    \hline
    \end{tabular}
}
\end{table}

For adaptive timestep scheduling, as shown in \cref{tab:wasserstein_param_grid}, we tune four hyperparameters, $\eta_{\min}$, $\eta_{\max}$, $p$, and $q$, which jointly control the overall Wasserstein error tolerance and its allocation across diffusion timesteps. We perform explicit tuning across different training configurations for CIFAR-10, but the model performance is largely insensitive to the hyperparameter grid. It is also worth noting that the choice of initial grid is not entirely heuristic. Specifically, $(\eta_{\mathrm{min}}, \eta_{\mathrm{\max}})$ is determined by the minimum and maximum Wasserstein error values induced by the baseline schedule, thereby preserving the overall error range while allowing $q$ to reallocate the error budget.

\begin{table}[htbp]
\centering
\caption{Parameter search grid for Wasserstein error tolerance and $N$-step resampling parameters for CIFAR-10.}
\label{tab:wasserstein_param_grid}
\renewcommand{\arraystretch}{1.2}
\setlength{\tabcolsep}{10pt}

\resizebox{\linewidth}{!}{
    \begin{tabular}{|c|c|c|c|c|c|}
    \hline
    \multirow{2}{*}{Parameter} 
    & \multicolumn{4}{c|}{CIFAR-10} 
    & \multirow{2}{*}{Grid search} \\
    \cline{2-5}
    & Uncond. VP & Uncond. VE & Cond. VP & Cond. VE & \\
    \hline
    $\eta_{\mathrm{min}}$ & 0.01 & 0.01 & 0.01 & 0.02 & 0.01, 0.02, 0.03, 0.04, 0.05 \\
    $\eta_{\mathrm{max}}$ & 0.40 & 0.40 & 0.40 & 0.10 & 0.10, 0.20, 0.30, 0.40, 0.50 \\
    $p$ & 1.0 & 1.0 & 1.0 & 1.0 & 0.8, 1.0, 1.2 \\
    $q$ & 0.1 & 0.25 & 0.1 & 0.25 & 0.1, 0.25\\
    \hline
    \end{tabular}
}
\end{table}

\section{Additional Results}
\label{appendix:additional_results}
In this section, we provide additional results for our proposed SDM sampler.

\begin{table*}[htbp]
\caption{Quantitative results on conditional generation for CIFAR-10 $32\times32$ and ImageNet $64\times64$, measured by FID and NFE.
``SDM (Adaptive Scheduling)'' denotes the Wasserstein-bounded timestep schedule for the solver indicated in the leftmost column. Rows under ``SDM (Adaptive Solver)'' use the curvature-based solver selection with the optimized threshold parameter $\tau_{\kappa}$. The best FID within each solver block for a given dataset and parameterization are highlighted in \textbf{bold}, while the overall best FID in each column are highlighted in \underline{\textbf{bold}}.}

\label{table:results_conditional}
\centering
\renewcommand{\arraystretch}{1.15}
\setlength{\tabcolsep}{4pt}

\resizebox{\textwidth}{!}{
\begin{tabular}{
    p{0.16\textwidth} @{\hspace{12pt}} l p{0.06\textwidth}
    *{3}{>{\centering\arraybackslash}p{0.07\textwidth}}
    }
    \toprule
        \multicolumn{3}{c}{} &
        \multicolumn{2}{c}{Conditional} &
        \multicolumn{1}{c}{Conditional} \\
        \multicolumn{3}{c}{} &
        \multicolumn{2}{c}{CIFAR-10 $32 \times 32$} &
        \multicolumn{1}{c}{ImageNet $64 \times 64$} \\
        \cmidrule(lr){4-5}\cmidrule(lr){6-6}
        \multicolumn{3}{c}{} & \multicolumn{1}{c}{VP} & \multicolumn{1}{c}{VE} & \multicolumn{1}{c}{ADM} \\
        \midrule
        \textbf{Solver} & \textbf{Schedule} & & & & \\
        \midrule
    
        \multirow{4}{*}{Euler}
            & EDM $(\rho=7)$ & & \multicolumn{1}{c}{7.09} & \multicolumn{1}{c}{6.75} & \multicolumn{1}{c}{3.48} \\
            & COS \citep{williams2024score} & & \multicolumn{1}{c}{7.05} & \multicolumn{1}{c}{6.66} & \multicolumn{1}{c}{3.74} \\
            & SDM (Adaptive Scheduling) & & \multicolumn{1}{c}{\textbf{6.10}} & \multicolumn{1}{c}{\textbf{5.12}} & \multicolumn{1}{c}{\textbf{2.32}} \\
        \cmidrule{2-6}
            & & NFE & \multicolumn{1}{c}{17} & \multicolumn{1}{c}{17} & \multicolumn{1}{c}{255} \\
        \midrule
        \addlinespace[2pt]
    
        \multirow{4}{*}{Heun}
            & EDM $(\rho=7)$ & & \multicolumn{1}{c}{\textbf{1.82}} & \multicolumn{1}{c}{\textbf{1.78}} & \multicolumn{1}{c}{1.45} \\
            & COS \citep{williams2024score} & & \multicolumn{1}{c}{1.92} & \multicolumn{1}{c}{1.87} & \multicolumn{1}{c}{\underline{\textbf{1.40}}} \\
            & SDM (Adaptive Scheduling) & & \multicolumn{1}{c}{1.92} & \multicolumn{1}{c}{1.92} & \multicolumn{1}{c}{2.25} \\
        \cmidrule{2-6}
            & & NFE & \multicolumn{1}{c}{35} & \multicolumn{1}{c}{35} & \multicolumn{1}{c}{511} \\
        \midrule
        \addlinespace[2pt]
    
        \multirow{5}{*}{\shortstack[l]{SDM\\(Adaptive Solver)}}
            & EDM $(\rho=7)$ & &
            \multicolumn{1}{c}{\underline{\textbf{1.81}}} &
            \multicolumn{1}{c}{\underline{\textbf{1.77}}} &
            \multicolumn{1}{c}{\textbf{1.45}} \\
            \cmidrule{2-6}
            & & NFE & \multicolumn{1}{c}{30} & \multicolumn{1}{c}{32} & \multicolumn{1}{c}{476} \\
            \cmidrule{2-6}
            & SDM (Adaptive Scheduling) & &
            \multicolumn{1}{c}{\underline{\textbf{1.81}}} &
            \multicolumn{1}{c}{\underline{\textbf{1.77}}} &
            \multicolumn{1}{c}{2.25} \\
            \cmidrule{2-6}
            & & NFE & \multicolumn{1}{c}{32} & \multicolumn{1}{c}{33} & \multicolumn{1}{c}{503} \\
    \bottomrule
\end{tabular}
}
\end{table*}

\subsection{Quantitative Results for Conditional Generation}
We report the quantitative results for conditional generation in \cref{table:results_conditional}. Under the Euler solver, SDM with adaptive scheduling achieves strong gains, obtaining state-of-the-art FID 6.10 and 5.12 on CIFAR-10 with VP and VE based models, respectively, and 2.32 on ImageNet. For adaptive solver settings, SDM further improves both sample quality and sampling efficiency, achieving FID 1.81 and 1.77 on CIFAR-10 with VP and VE based models, respectively, and 1.45 on ImageNet. Conditional generation is more setting-dependent than unconditional generation: SDM scheduling substantially improves Euler baselines, whereas gains under Heun or adaptive-solver settings are less consistent.

\subsection{Quantitative Results for modern ODE solvers}
\begin{table*}[htbp]
\caption{Quantitative results on modern ODE samplers, including DPM-Solver++ and UniPC, for CIFAR-10 $32\times32$, FFHQ $64\times64$, and AFHQv2 $64\times64$. Overall, SDM schedules achieves superior performance in FID compared to EDM baselines.}
\label{table:results_modern_ode_solver}
\centering
\renewcommand{\arraystretch}{1.15}
\setlength{\tabcolsep}{4pt}

\resizebox{\textwidth}{!}{
\begin{tabular}{
    p{0.21\textwidth} @{\hspace{10pt}} l
    *{6}{>{\centering\arraybackslash}p{0.07\textwidth}}
    }
    \toprule
        \multicolumn{2}{c}{} & 
        \multicolumn{2}{c}{Unconditional} &
        \multicolumn{2}{c}{Unconditional} &
        \multicolumn{2}{c}{Unconditional} \\      
        \multicolumn{2}{c}{} & 
        \multicolumn{2}{c}{CIFAR-10 $32 \times 32$} &
        \multicolumn{2}{c}{FFHQ $64 \times 64$} &
        \multicolumn{2}{c}{AFHQv2 $64 \times 64$} \\
        \cmidrule(lr){3-4}\cmidrule(lr){5-6}\cmidrule(lr){7-8}
        \multicolumn{2}{c}{} & 
        \multicolumn{1}{c}{VP} & \multicolumn{1}{c}{VE} &
        \multicolumn{1}{c}{VP} & \multicolumn{1}{c}{VE} &
        \multicolumn{1}{c}{VP} & \multicolumn{1}{c}{VE} \\
        \midrule
        \textbf{Solver} & \textbf{Schedule} & & & & & & \\
        \midrule
        
        \multirow{2}{*}{DPM-Solver++(2M)} & EDM $(\rho=7)$ & \multicolumn{1}{c}{2.57} & \multicolumn{1}{c}{2.52} & \multicolumn{1}{c}{2.60} & \multicolumn{1}{c}{\textbf{2.71}} & \multicolumn{1}{c}{2.12} & \multicolumn{1}{c}{2.26} \\
        & SDM (Adaptive Scheduling) & \textbf{2.33} & \multicolumn{1}{c}{\textbf{2.33}} & \multicolumn{1}{c}{\textbf{2.59}} & \multicolumn{1}{c}{2.73} & \multicolumn{1}{c}{\textbf{2.07}} & \multicolumn{1}{c}{\textbf{2.19}} \\
        \midrule
        \addlinespace[2pt]
        
        \multirow{2}{*}{UniPC} & EDM $(\rho=7)$ & \multicolumn{1}{c}{2.32} & \multicolumn{1}{c}{2.28} & \multicolumn{1}{c}{2.51} & \multicolumn{1}{c}{2.62} & \multicolumn{1}{c}{2.07} & \multicolumn{1}{c}{2.20} \\
        & SDM (Adaptive Scheduling) & \multicolumn{1}{c}{\textbf{2.26}} & \multicolumn{1}{c}{\textbf{2.13}} & \multicolumn{1}{c}{\textbf{2.49}} & \multicolumn{1}{c}{\textbf{2.60}} & \multicolumn{1}{c}{\textbf{2.04}} & \multicolumn{1}{c}{\textbf{2.14}} \\
    \bottomrule
    \end{tabular}
}
\end{table*}

\subsection{Quantitative Results for Few-Step Generation}
\begin{table*}[htbp]
\caption{Quantitative results on 10-step generation for FFHQ $64 \times 64$ and AFHQv2 $64 \times 64$. Results for the adaptive solver are reported with the optimal threshold parameter $\tau_{\kappa}$. Adaptive scheduling is applied to Euler, Heun, and SDM-based adaptive solvers. The best-performing configurations compared to
the corresponding baselines are highlighted in \textbf{bold}, while the overall best results across all methods are highlighted in \underline{\textbf{bold}}. Overall, the proposed SDM sampler consistently shows superior performance in both sample quality and computational efficiency for 10-step generation.}
\label{table:results_10_step_generation}
\centering
\renewcommand{\arraystretch}{1.15}
\setlength{\tabcolsep}{4pt}

\resizebox{\textwidth}{!}{
\begin{tabular}{
    p{0.21\textwidth} @{\hspace{10pt}} l
    *{6}{>{\centering\arraybackslash}p{0.07\textwidth}}
    }
    \toprule
        \multicolumn{2}{c}{} & 
        \multicolumn{2}{c}{Unconditional} &
        \multicolumn{2}{c}{Unconditional} \\      
        \multicolumn{2}{c}{} & 
        \multicolumn{2}{c}{FFHQ $64 \times 64$} &
        \multicolumn{2}{c}{AFHQv2 $64 \times 64$} \\
        \cmidrule(lr){3-4}\cmidrule(lr){5-6}
        \multicolumn{2}{c}{} & 
        \multicolumn{1}{c}{VP} & \multicolumn{1}{c}{VE} &
        \multicolumn{1}{c}{VP} & \multicolumn{1}{c}{VE} \\
        \midrule
        \textbf{Solver} & \textbf{Schedule} & & & & \\
        \midrule
        
        \multirow{2}{*}{Euler} & EDM $(\rho=7)$ & \multicolumn{1}{c}{20.92} & \multicolumn{1}{c}{21.06} & \multicolumn{1}{c}{10.14} & \multicolumn{1}{c}{10.73} \\
        & SDM (Adaptive Scheduling) & \multicolumn{1}{c}{\textbf{17.94}} & \multicolumn{1}{c}{\textbf{19.03}} & \multicolumn{1}{c}{\textbf{9.73}} & \multicolumn{1}{c}{\textbf{10.07}} \\
        \midrule
        \addlinespace[2pt]
        
        \multirow{2}{*}{Heun} & EDM $(\rho=7)$ & \multicolumn{1}{c}{5.27} & \multicolumn{1}{c}{5.44} & \multicolumn{1}{c}{2.99} & \multicolumn{1}{c}{3.16} \\
        & SDM (Adaptive Scheduling) & \multicolumn{1}{c}{\underline{\textbf{3.78}}} & \multicolumn{1}{c}{\underline{\textbf{3.21}}} & \multicolumn{1}{c}{5.06} & \multicolumn{1}{c}{3.42} \\
        \midrule
        \addlinespace[2pt]

        \multirow{2}{*}{SDM} & EDM $(\rho=7)$ & \multicolumn{1}{c}{5.34} & \multicolumn{1}{c}{5.52} & \multicolumn{1}{c}{\underline{\textbf{2.92}}} & \multicolumn{1}{c}{\textbf{3.02}} \\
        & SDM (Adaptive Scheduling) & \multicolumn{1}{c}{\textbf{3.96}} & \multicolumn{1}{c}{\textbf{3.42}} & \multicolumn{1}{c}{3.28} & \multicolumn{1}{c}{\underline{\textbf{2.50}}} \\
    \bottomrule
    \end{tabular}
}
\end{table*}

\begin{table*}[htbp]
\caption{Quantitative results on 20-step generation for FFHQ $64 \times 64$ and AFHQv2 $64 \times 64$. Results for the adaptive solver are reported with the optimal threshold parameter $\tau_{\kappa}$. Adaptive scheduling is applied to Euler, Heun, and SDM-based adaptive solvers. The best-performing configurations compared to
the corresponding baselines are highlighted in \textbf{bold}, while the overall best results across all methods are highlighted in \underline{\textbf{bold}}. Overall, the proposed SDM sampler consistently shows superior performance in both sample quality and computational efficiency for 20-step generation.}
\label{table:results_20_step_generation}
\centering
\renewcommand{\arraystretch}{1.15}
\setlength{\tabcolsep}{4pt}

\resizebox{\textwidth}{!}{
\begin{tabular}{
    p{0.21\textwidth} @{\hspace{10pt}} l
    *{6}{>{\centering\arraybackslash}p{0.07\textwidth}}
    }
    \toprule
        \multicolumn{2}{c}{} & 
        \multicolumn{2}{c}{Unconditional} &
        \multicolumn{2}{c}{Unconditional} \\      
        \multicolumn{2}{c}{} & 
        \multicolumn{2}{c}{FFHQ $64 \times 64$} &
        \multicolumn{2}{c}{AFHQv2 $64 \times 64$} \\
        \cmidrule(lr){3-4}\cmidrule(lr){5-6}
        \multicolumn{2}{c}{} & 
        \multicolumn{1}{c}{VP} & \multicolumn{1}{c}{VE} &
        \multicolumn{1}{c}{VP} & \multicolumn{1}{c}{VE} \\
        \midrule
        \textbf{Solver} & \textbf{Schedule} & & & & \\
        \midrule
        
        \multirow{2}{*}{Euler} & EDM $(\rho=7)$ & \multicolumn{1}{c}{8.64} & \multicolumn{1}{c}{8.75} & \multicolumn{1}{c}{4.61} & \multicolumn{1}{c}{5.04} \\
        & SDM (Adaptive Scheduling) & \multicolumn{1}{c}{\textbf{7.37}} & \multicolumn{1}{c}{\textbf{8.13}} & \multicolumn{1}{c}{\textbf{4.21}} & \multicolumn{1}{c}{\textbf{4.73}} \\
        \midrule
        \addlinespace[2pt]
        
        \multirow{2}{*}{Heun} & EDM $(\rho=7)$ & \multicolumn{1}{c}{2.64} & \multicolumn{1}{c}{2.76} & \multicolumn{1}{c}{2.08} & \multicolumn{1}{c}{2.22} \\
        & SDM (Adaptive Scheduling) & \multicolumn{1}{c}{\underline{\textbf{2.38}}} & \multicolumn{1}{c}{\textbf{2.54}} & \multicolumn{1}{c}{2.31} & \multicolumn{1}{c}{2.23} \\
        \midrule
        \addlinespace[2pt]

        \multirow{2}{*}{SDM} & EDM $(\rho=7)$ & \multicolumn{1}{c}{2.65} & \multicolumn{1}{c}{\textbf{2.75}} & \multicolumn{1}{c}{\underline{\textbf{1.99}}} & \multicolumn{1}{c}{\textbf{2.11}} \\
        & SDM (Adaptive Scheduling) & \multicolumn{1}{c}{\textbf{2.43}} & \multicolumn{1}{c}{\underline{\textbf{2.53}}} & \multicolumn{1}{c}{\textbf{2.01}} & \multicolumn{1}{c}{\underline{\textbf{2.09}}} \\
    \bottomrule
    \end{tabular}
}
\end{table*}

\newpage
\subsection{Quantitative Results for High Resolution Generation}
\begin{table}[htbp]
\caption{Quantitative results on high-resolution generation for ImageNet $512 \times 512$ using the EDM2-M pretrained model. (\textbf{Left}) FID comparison under EDM and SDM schedules with fixed solvers, DPM-Solver++(2M) and UniPC, respectively. (\textbf{Right}) FID comparison with changes in adaptive solver (DPM-Solver++(1M) $\rightarrow$ DPM-Solver++(2M)) and adaptive scheduling.}
\label{tab:fid_edm2_imagenet512}
\centering
\vspace{6pt}

\begin{minipage}[t]{0.48\linewidth}
\centering
\renewcommand{\arraystretch}{1.1}
\setlength{\tabcolsep}{6pt}

\begin{tabular}{l c c}
\toprule
\textbf{Solver} & \textbf{Schedule} & \textbf{FID} \\
\midrule
DPM-Solver++(2M) & EDM & 2.48 \\
DPM-Solver++(2M) & SDM & \textbf{2.36} \\
\midrule
UniPC & EDM & 2.45 \\
UniPC & SDM & \textbf{2.37} \\
\bottomrule
\end{tabular}
\end{minipage}%
\begin{minipage}[t]{0.48\linewidth}
\centering
\renewcommand{\arraystretch}{1.1}
\setlength{\tabcolsep}{6pt}

\begin{tabular}{l c c}
\toprule
\textbf{Solver} & \textbf{Schedule} & \textbf{FID} \\
\midrule
DPM-Solver++(2M) & EDM & 2.48 \\
SDM & EDM & \textbf{2.47} \\
\midrule
DPM-Solver++(2M) & SDM & 2.36 \\
SDM & SDM & \textbf{2.34} \\
\bottomrule
\end{tabular}
\end{minipage}
\end{table}

\subsection{Quantitative Results for Text-to-Image Generation}
\begin{table}[htbp]
\caption{Quantitative results on text-to-image generation with SDXL. Comparison of SDM against standard schedules, including time-uniform, EDM, AYS, OTS \citep{xue2024accelerating}, and AutoDiffusion. SDM demonstrates competitive performance against recent schedule optimization baselines. \textbf{Bold} and \underline{underlined} values indicate the best and second-best results, respectively.}
\label{tab:fid_sdxl}
\centering
\vspace{6pt}

\renewcommand{\arraystretch}{1.1}
\setlength{\tabcolsep}{6pt}

\begin{tabular}{l c c}
\toprule
\textbf{Schedule} & \textbf{FID} & \textbf{CLIP} \\
\midrule
Uniform & 18.29 & 0.3209 \\
EDM & 18.10 & \textbf{0.3238} \\
AYS & 17.90 & 0.3205 \\
OTS \citep{xue2024accelerating} & \textbf{17.49} & \underline{0.3219} \\
AutoDiffusion & 17.76 & 0.3182 \\
SDM & \underline{17.72} & 0.3201 \\
\bottomrule
\end{tabular}
\end{table}

\clearpage
\subsection{Ablation Studies}
We further present ablation results on the following key components of the proposed SDM sampler for both unconditional and conditional generation settings in \cref{sec:experiments}.

\begin{figure}[htbp]
\centering

\begin{minipage}{0.4\textwidth}
\textbf{Threshold $\tau_{\kappa}$.} We investigate the effect of the threshold parameter $\tau_{\kappa}$, which controls the activation of higher order solvers. We evaluate FID while varying $\tau_{\kappa}$ for each dataset, as shown in \cref{fig:ablation_tauk_afhqv2_cifar10}. Based on the results, we set $\tau_{\kappa} = 2 \times 10^{-4}$ for CIFAR-10, $\tau_{\kappa} = 1 \times 10^{-4}$ for FFHQ and ImageNet, and $\tau_{\kappa} = 1 \times 10^{-3}$ for AFHQv2 under EDM schedules. For SDM schedules, we further adjust the threshold on AFHQv2, using $\tau_{\kappa} = 2 \times 10^{-4}$ for VP parameterization. Note that as long as $\tau_{\kappa} > 0$, the number of function evaluation is less than 2 per diffusion timestep, $i.e., \text{NFE} < 2$.

\vspace{3pt}

\textbf{Error tolerance and $N$-step resampling parameters.}
For adaptive scheduling, we perform a grid search over the error tolerance and resampling parameters on CIFAR-10, as detailed in \cref{tab:wasserstein_param_grid}. For the remaining datasets, we observe that performance is largely insensitive to the parameters. Consequently, we set parameters $\eta_{\mathrm{min}} = 0.02, \eta_{\mathrm{max}} = 0.20$, $p = 1.0$, and $q = 0.25$ for FFHQ and AFHQv2, and $\eta_{\mathrm{min}} = 0.001, \eta_{\mathrm{max}} = 0.01$, $p = 1.0$, and $q = 0.25$ for ImageNet.
\end{minipage}
\hfill
\begin{minipage}{0.55\textwidth}
\centering
  \includegraphics[width=\linewidth]{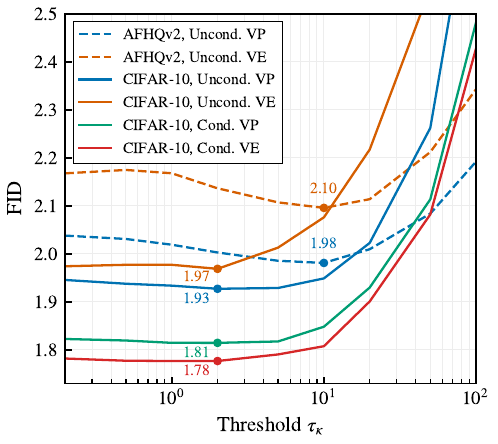}
  \caption{Ablation on threshold $\tau_{\kappa}$. FID as a function of the curvature threshold $\tau_{\kappa}$ for CIFAR-10 ($32\times32$, solid)
and AFHQv2 ($64\times64$, dashed), evaluated under unconditional and conditional
settings using the curvature-thresholded adaptive solver. Markers denote the selected $\tau_{\kappa}$ values that yield the best model performance for each dataset and training configuration.}
  \label{fig:ablation_tauk_afhqv2_cifar10}
\end{minipage}
\end{figure}

\subsection{Qualitative Results}
In this section, we present qualitative comparisons between the proposed SDM sampler and the baseline EDM sampler \citep{karras2022elucidating} across different datasets and training parameterizations.

\begin{figure*}[!b]
  \centering

  \begin{subfigure}[t]{0.45\textwidth}
    \centering
    {EDM (Heun), VP}\\
    \includegraphics[width=\linewidth]
    {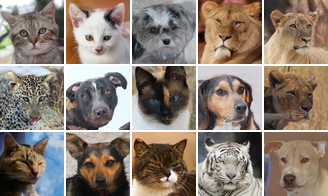}\\[-2pt]
    {FID=2.04, \; NFE=79}
  \end{subfigure}\hspace{12pt}
  \begin{subfigure}[t]{0.45\textwidth}
    \centering
    {EDM (Heun), VE}\\
    \includegraphics[width=\linewidth]{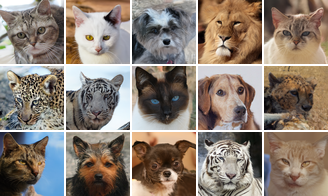}\\[-2pt]
    {FID=2.17, \; NFE=79}
  \end{subfigure}

  \vspace{6pt}

  \begin{subfigure}[t]{0.45\textwidth}
    \centering
    {SDM (Adaptive Solver + Scheduling), VP}\\
    \includegraphics[width=\linewidth]{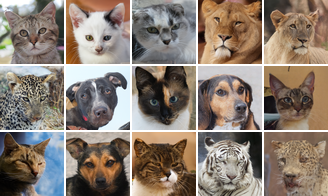}\\[-2pt]
    {FID=1.97, \; NFE=75}
  \end{subfigure}\hspace{12pt}
  \begin{subfigure}[t]{0.45\textwidth}
    \centering
    {SDM (Adaptive Solver + Scheduling), VE}\\
    \includegraphics[width=\linewidth]{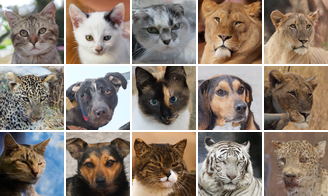}\\[-2pt]
    {FID=1.99, \; NFE=72}
  \end{subfigure}

  \caption{
    Qualitative comparison on AFHQv2 ($64\times64$) across VP (left) and VE (right) parameterizations. 
    The top row shows EDM (Heun), while the bottom row shows SDM with adaptive solver and scheduling.
    Each panel shows a $3 \times 5$ grid of generated samples; corresponding FID and NFE are reported below.
  }
  \label{fig:additional_results_afhqv2}
\end{figure*}

\begin{figure}[ht]
\centering
\setlength{\tabcolsep}{3pt}

\begin{tabular}{
>{\centering\arraybackslash}m{0.25cm}
>{\centering\arraybackslash}m{0.2\textwidth}
>{\centering\arraybackslash}m{0.2\textwidth}
>{\centering\arraybackslash}m{0.2\textwidth}
>{\centering\arraybackslash}m{0.2\textwidth}
}

\rotatebox{90}{\textbf{EDM}} &
\includegraphics[width=\linewidth]{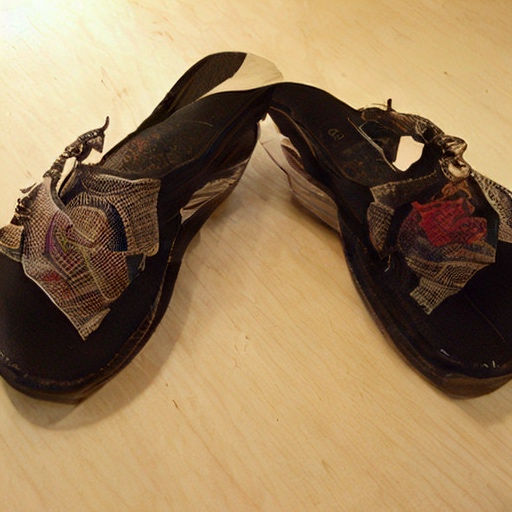} &
\includegraphics[width=\linewidth]{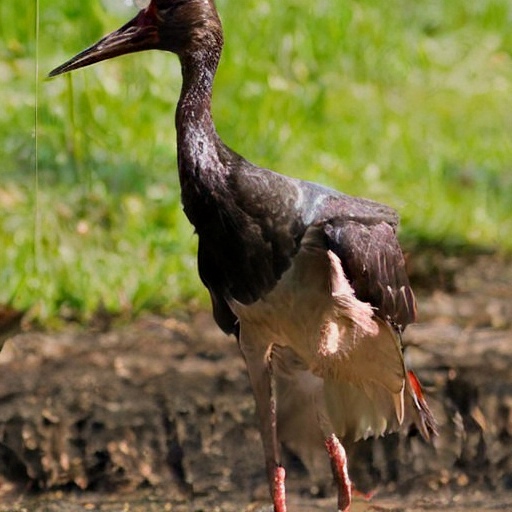} &
\includegraphics[width=\linewidth]{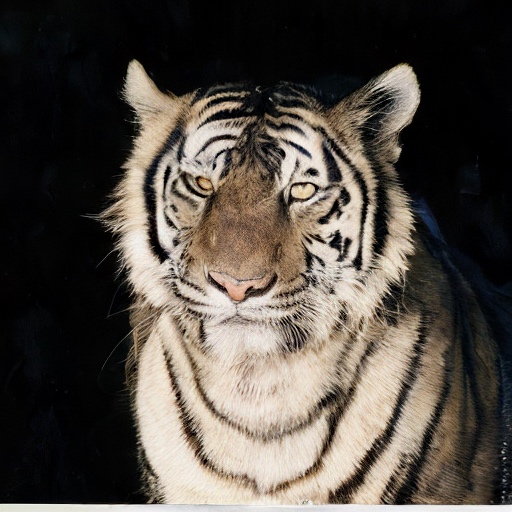} &
\includegraphics[width=\linewidth]{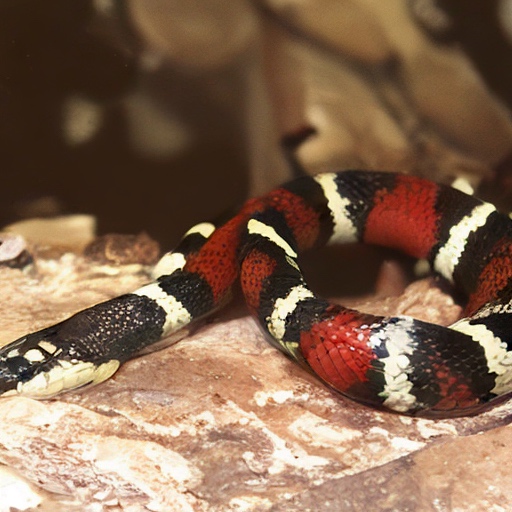} \\
\rotatebox{90}{\textbf{SDM}} &
\includegraphics[width=\linewidth]{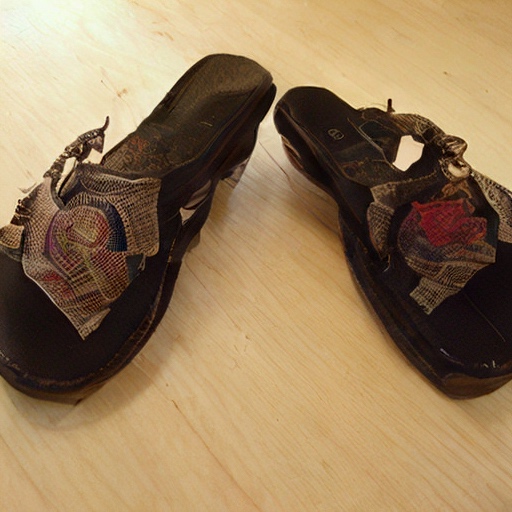} &
\includegraphics[width=\linewidth]{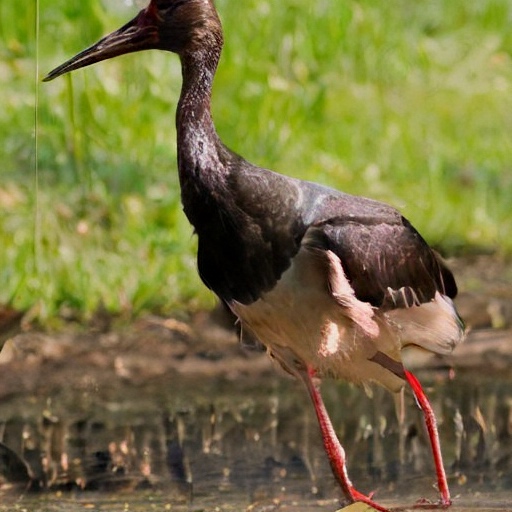} &
\includegraphics[width=\linewidth]{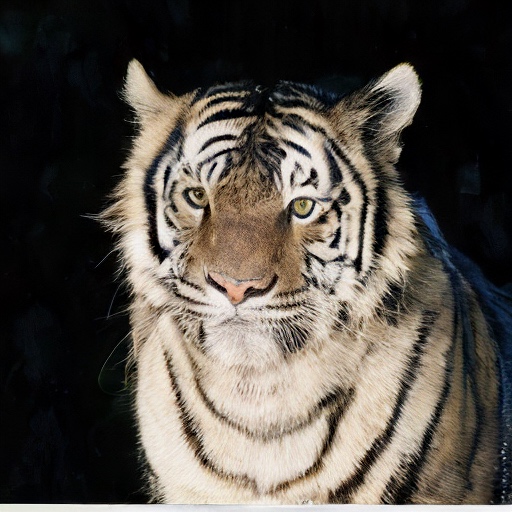} &
\includegraphics[width=\linewidth]{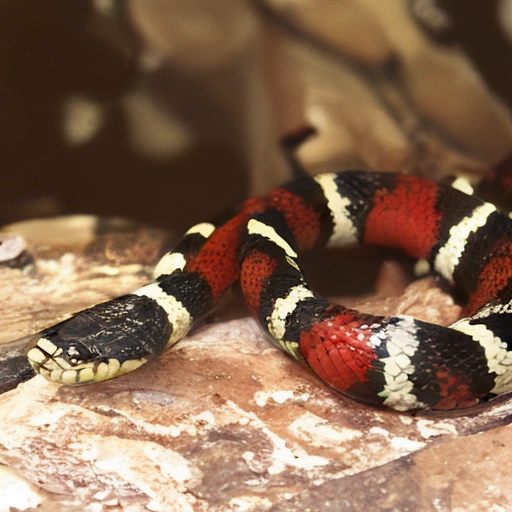} \\

\end{tabular}
\caption{
    Qualitative comparison on ImageNet $512 \times 512$ using the pretrained EDM2-M model. For each pair of rows, the top and bottom row represents results from EDM and SDM schedules, respectively. Across the examples, the EDM schedule exhibits more noticeable structural inconsistencies in generated objects, whereas SDM produces more coherent object structures.
    All results are generated using DPM-Solver++(2M) as the numerical solver. 
}
\label{fig:additional_results_imagenet512}
\end{figure}

\begin{figure}[ht]
\centering
\setlength{\tabcolsep}{4pt}

\begin{tabular}{cccc}

\textbf{Time uniform} & \textbf{EDM} & \textbf{AYS} & \textbf{SDM} \\

\includegraphics[width=0.20\textwidth]{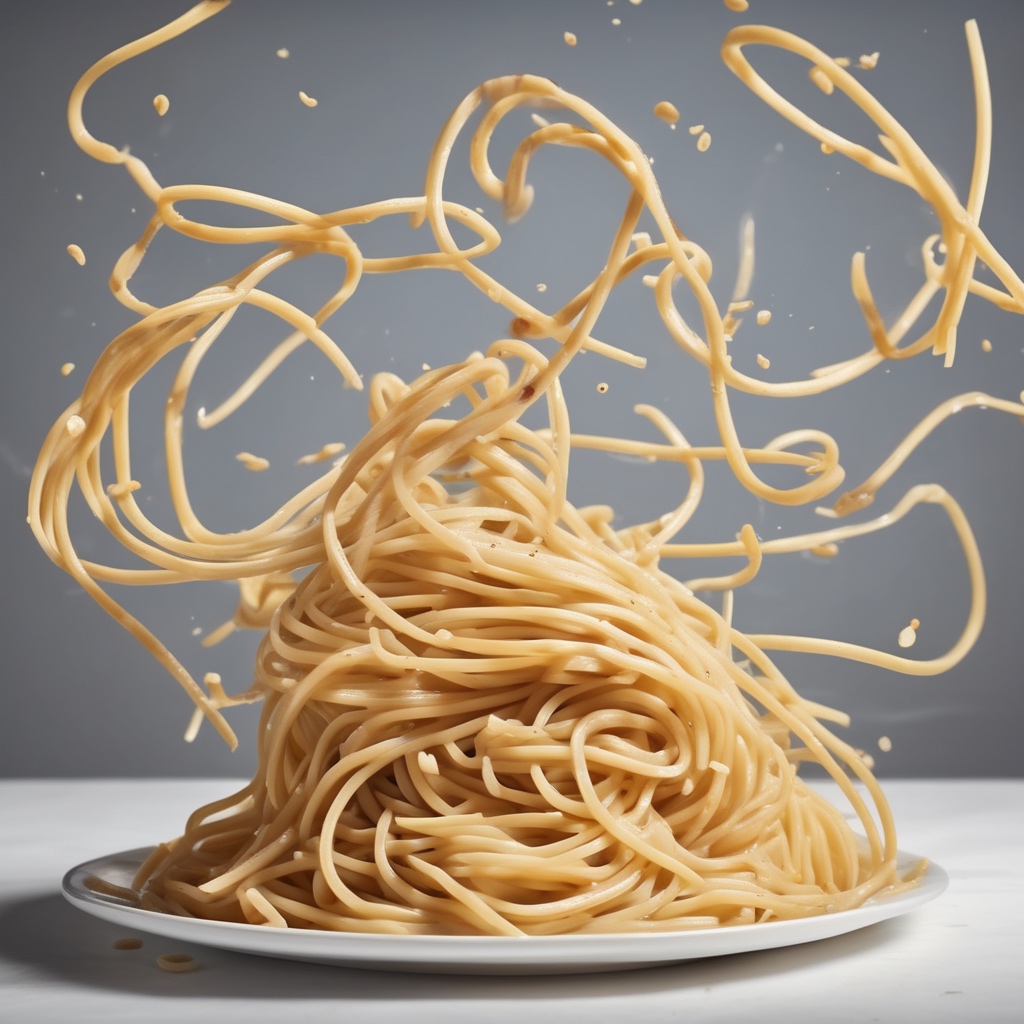} &
\includegraphics[width=0.20\textwidth]{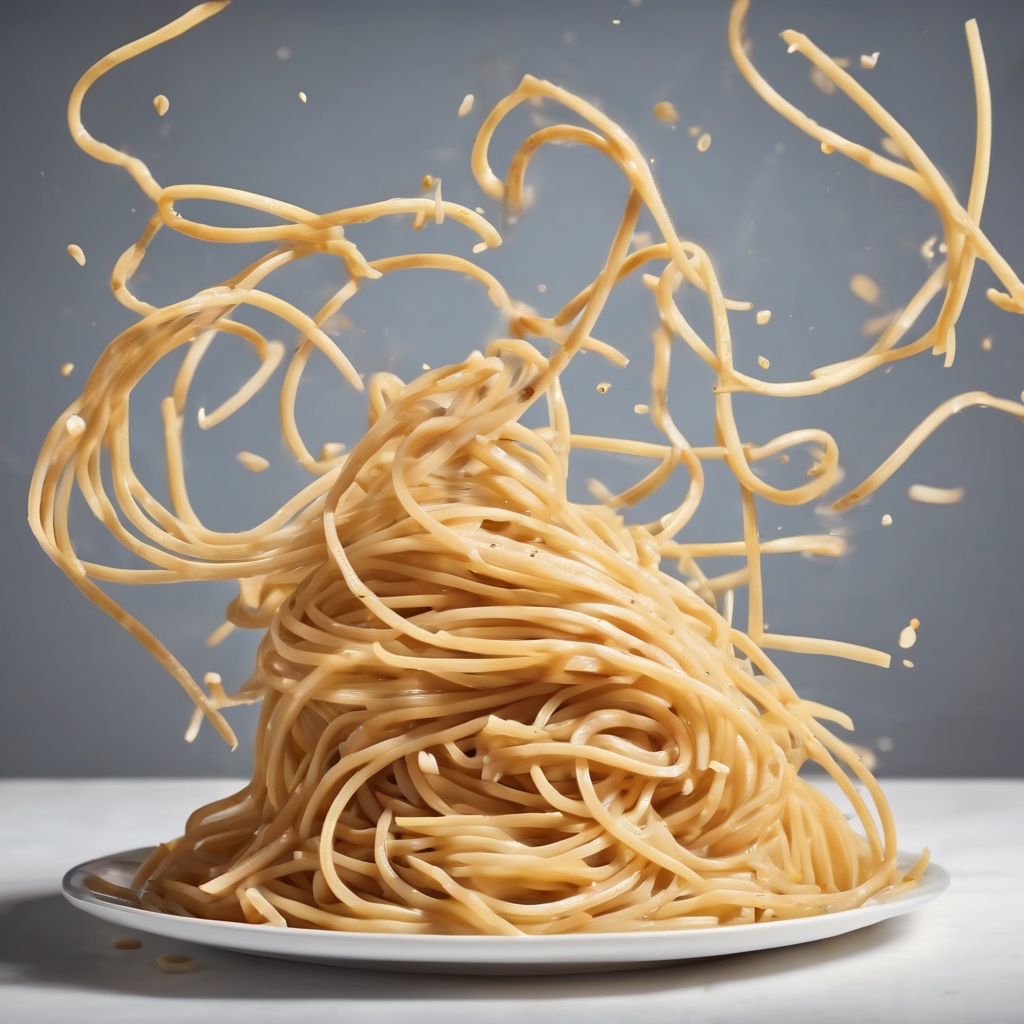} &
\includegraphics[width=0.20\textwidth]{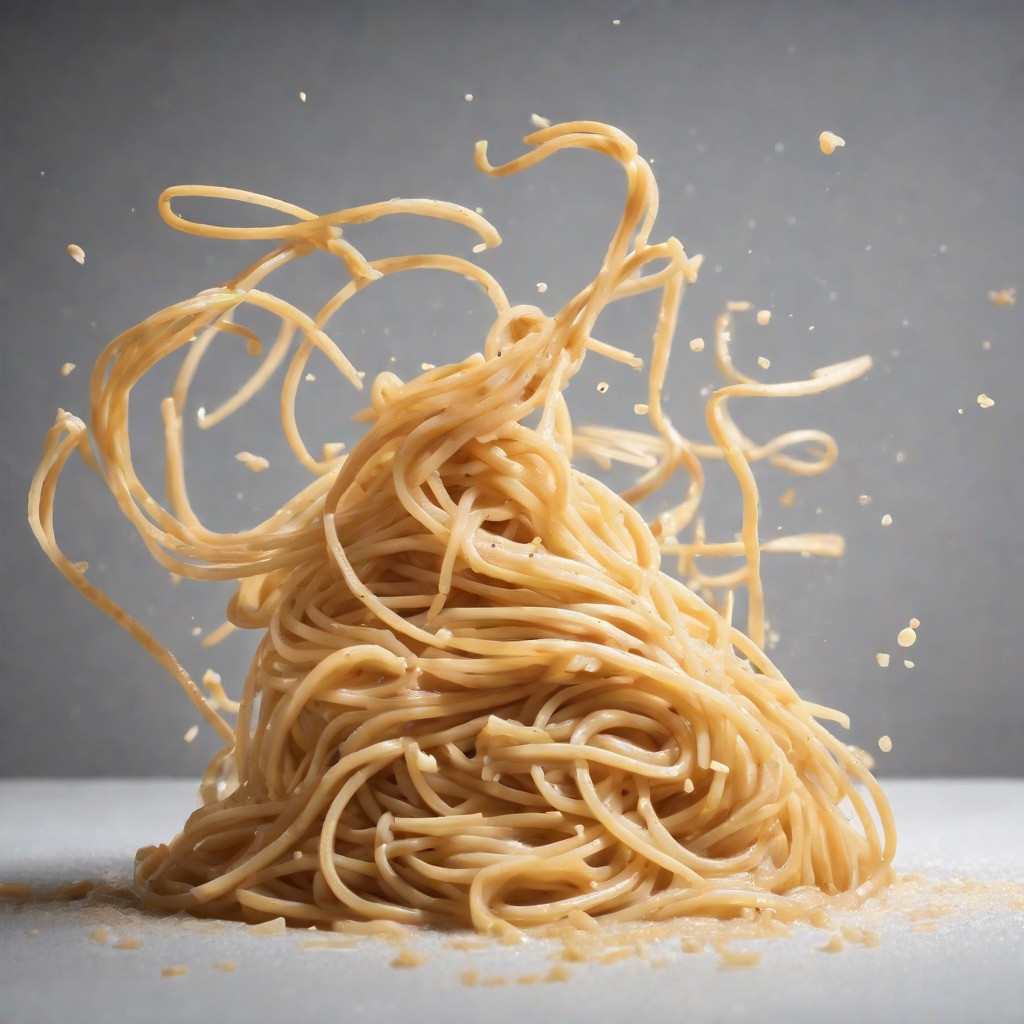} & 
\includegraphics[width=0.20\textwidth]{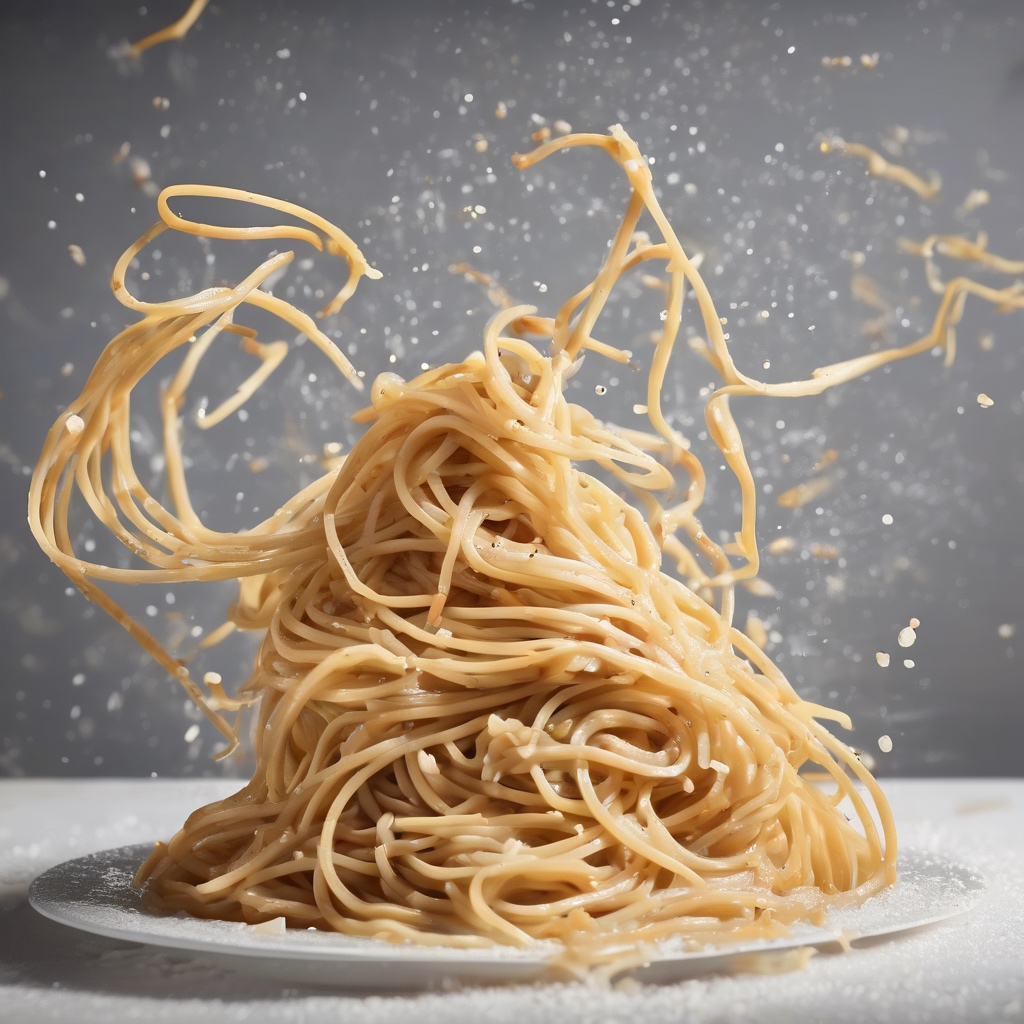} \\
\multicolumn{4}{c}{
\textit{Text prompt: "A plate of spaghetti falling mid-air, noodles suspended, dynamic freeze"}} \\[0.5em]

\includegraphics[width=0.20\textwidth]{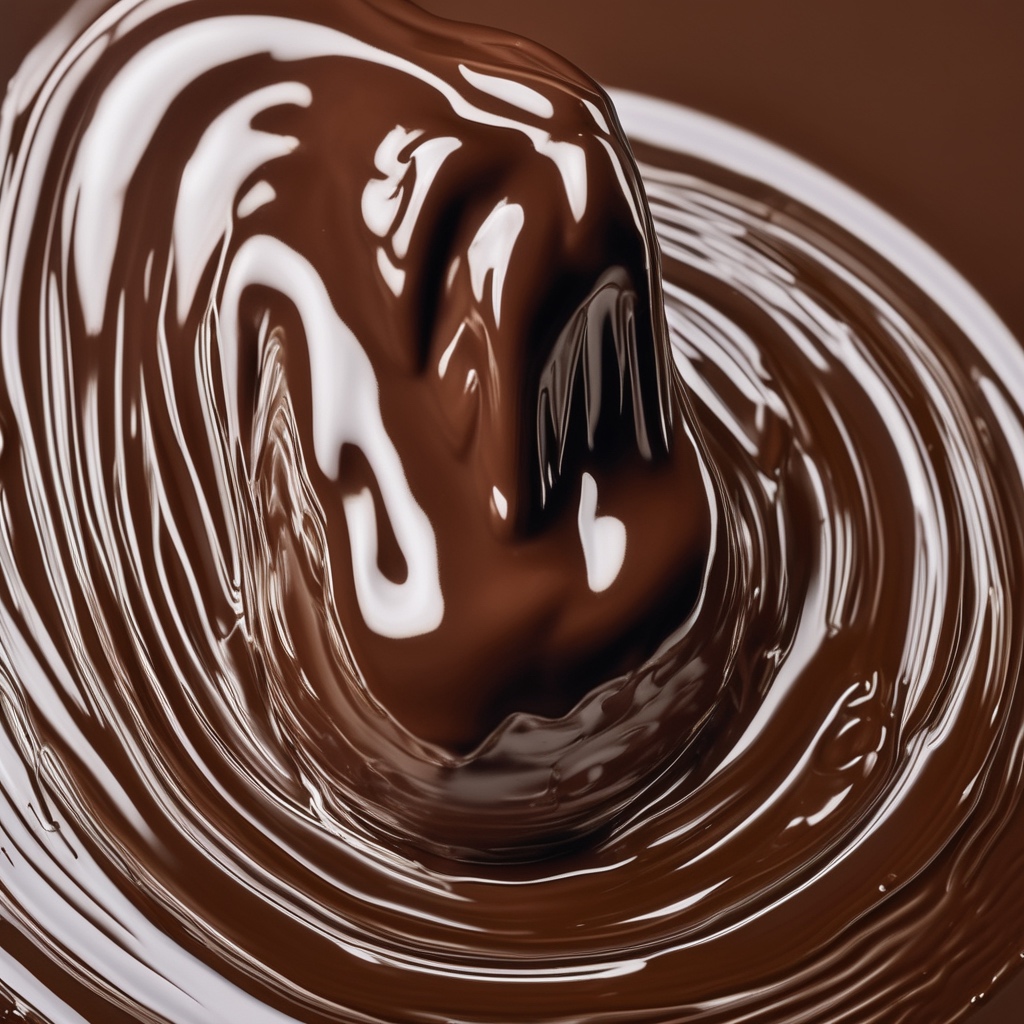} &
\includegraphics[width=0.20\textwidth]{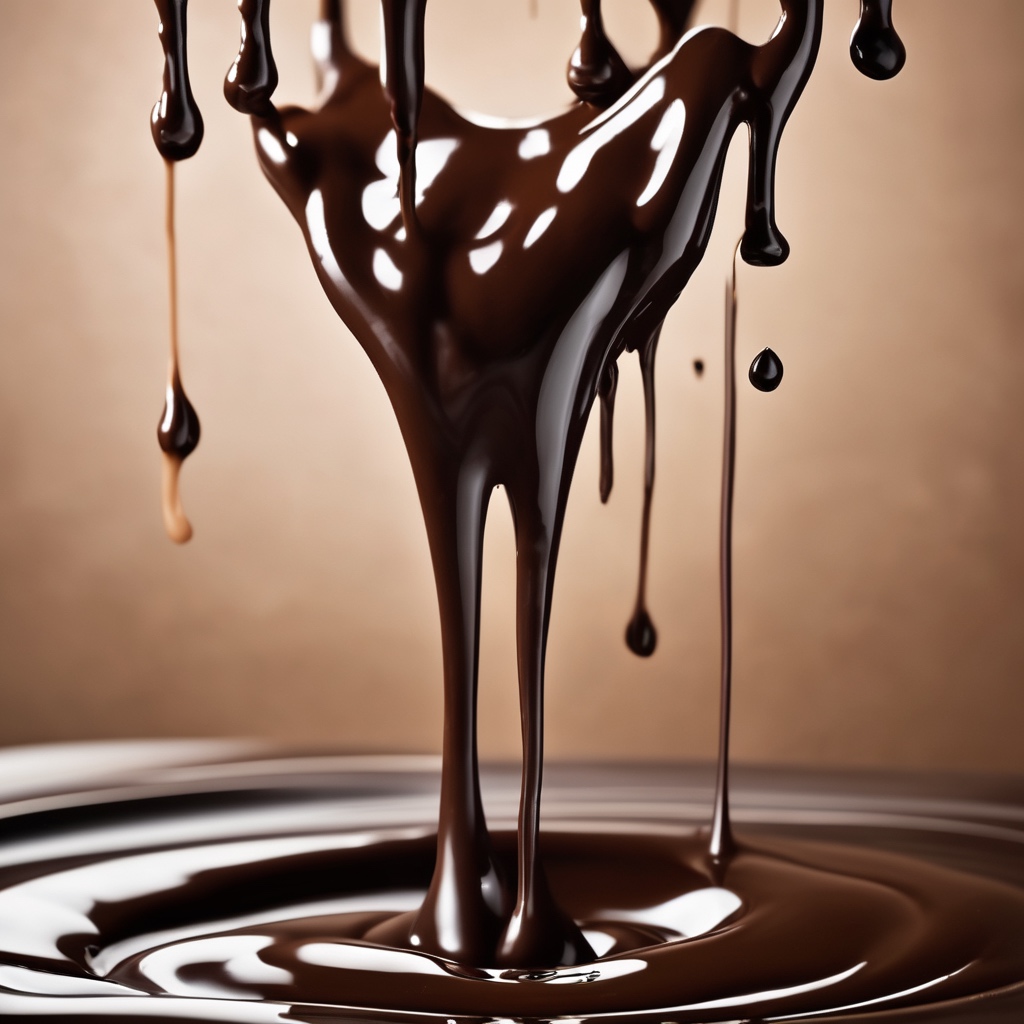} &
\includegraphics[width=0.20\textwidth]{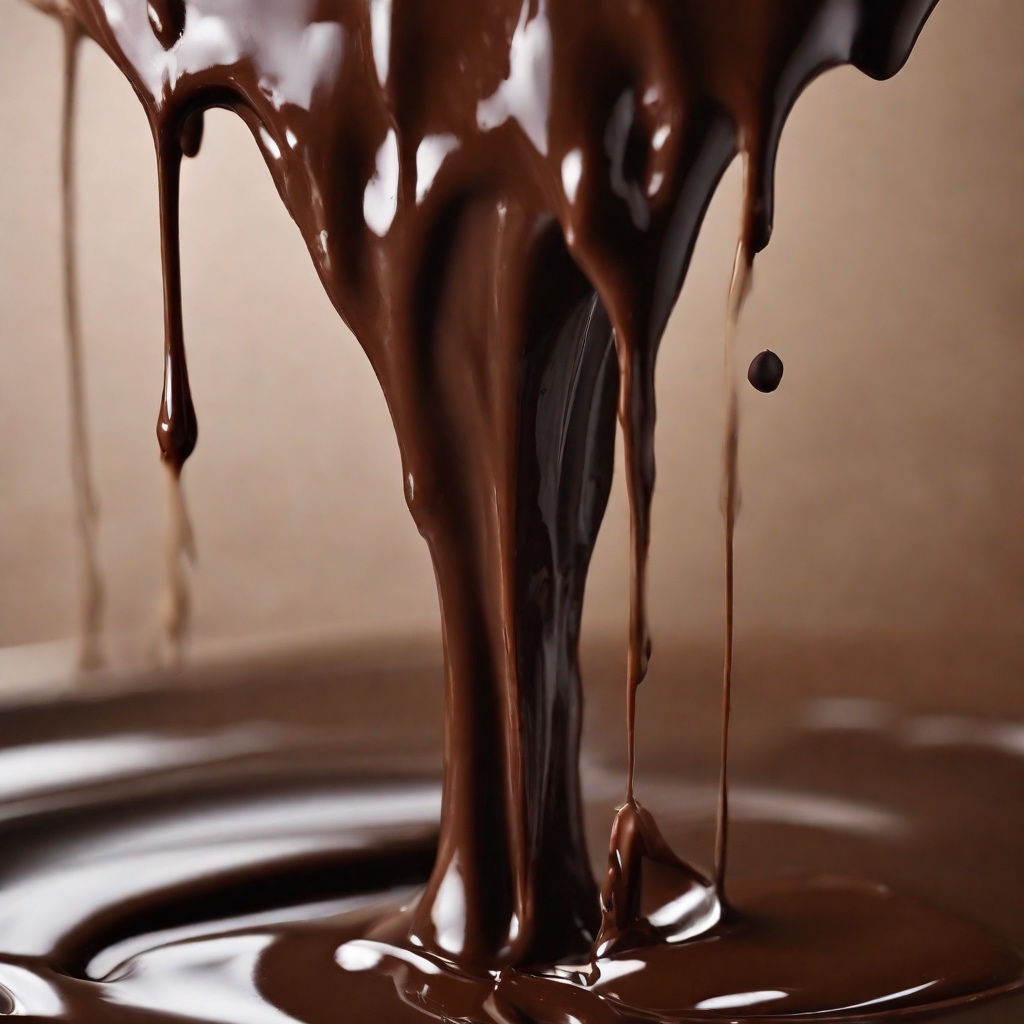} & 
\includegraphics[width=0.20\textwidth]{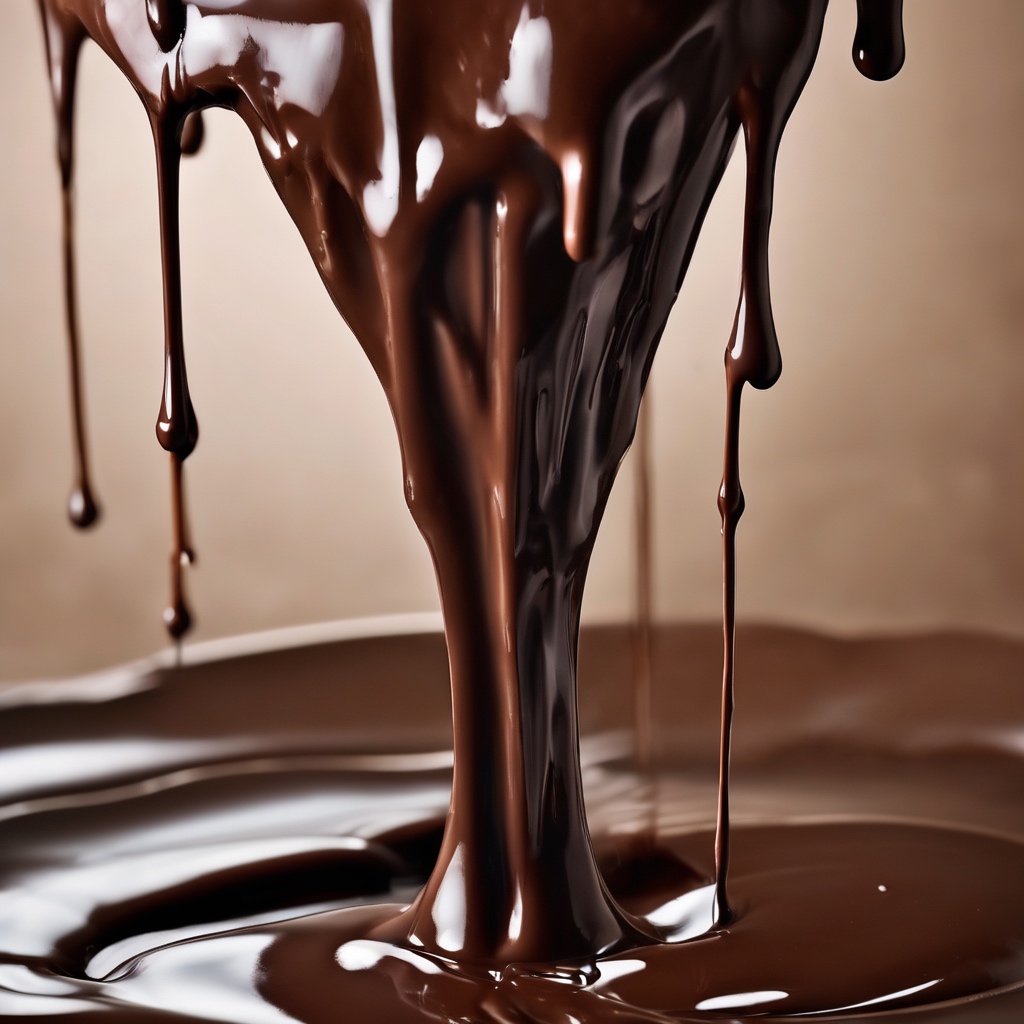} \\
\multicolumn{4}{c}{
\textit{Text prompt: "A close-up of melting chocolate dripping, food photography"}} \\[0.5em]

\includegraphics[width=0.21\textwidth]{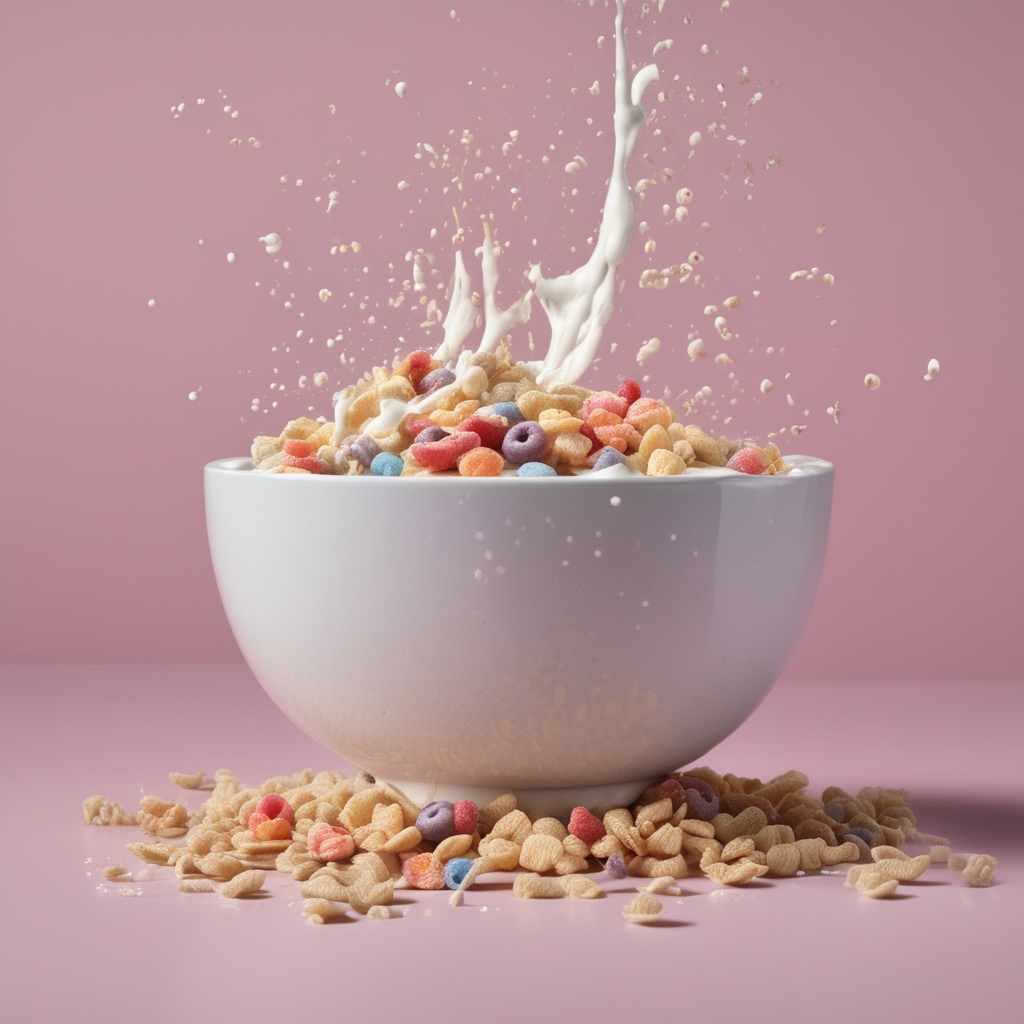} &
\includegraphics[width=0.21\textwidth]{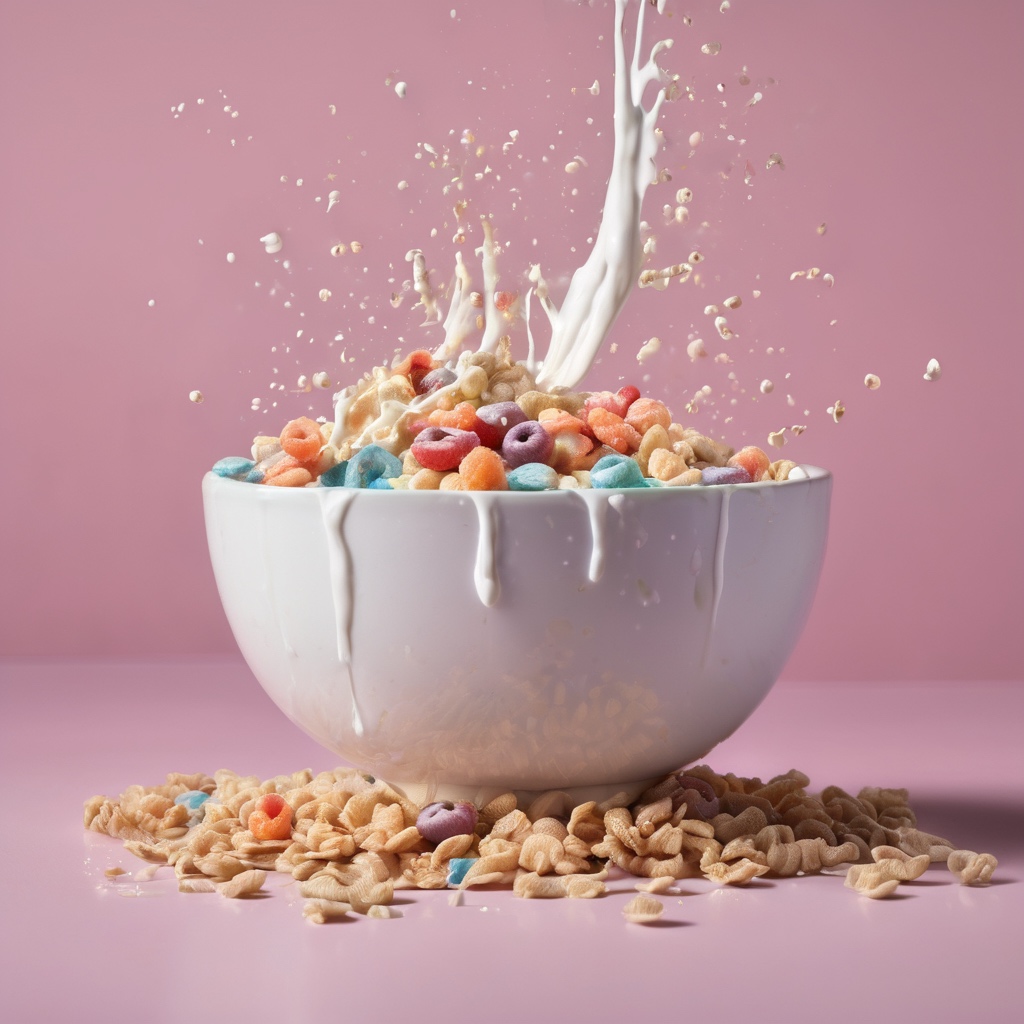} &
\includegraphics[width=0.21\textwidth]{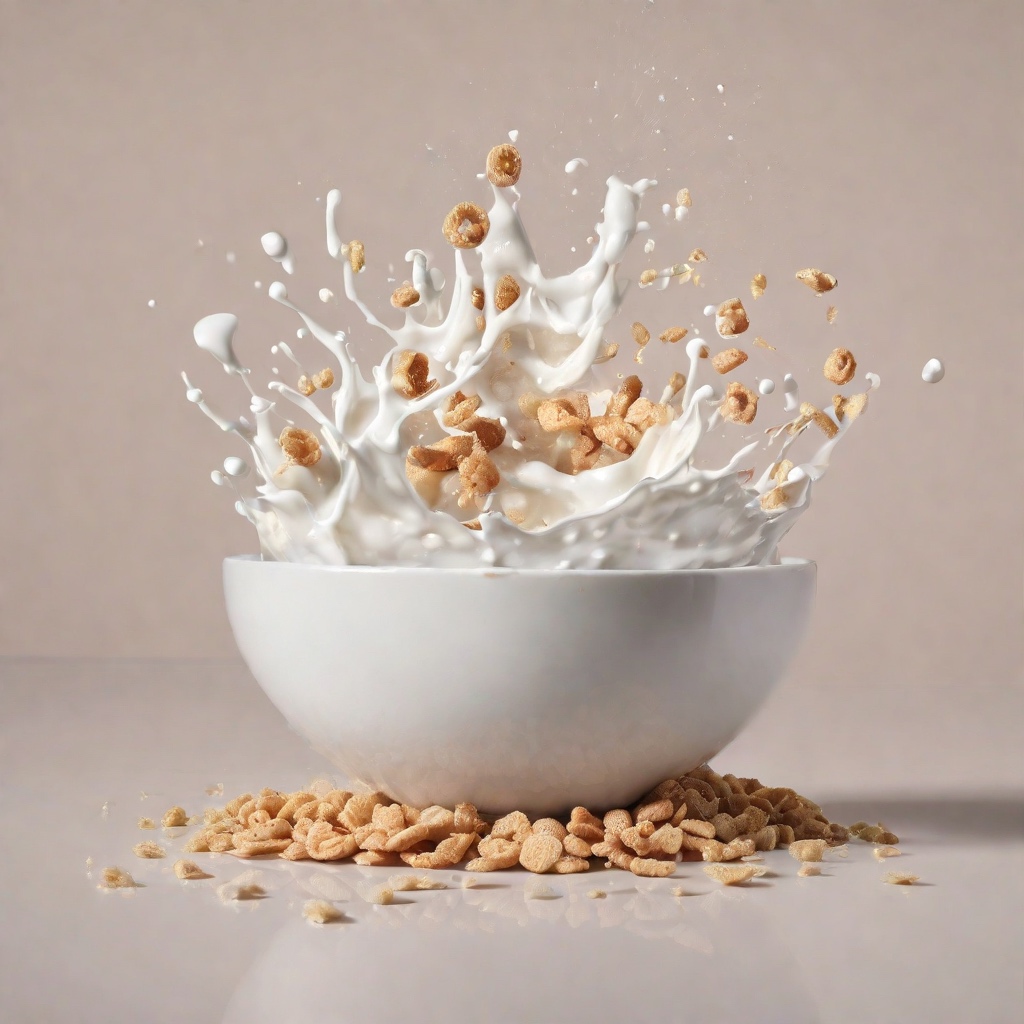} & 
\includegraphics[width=0.21\textwidth]{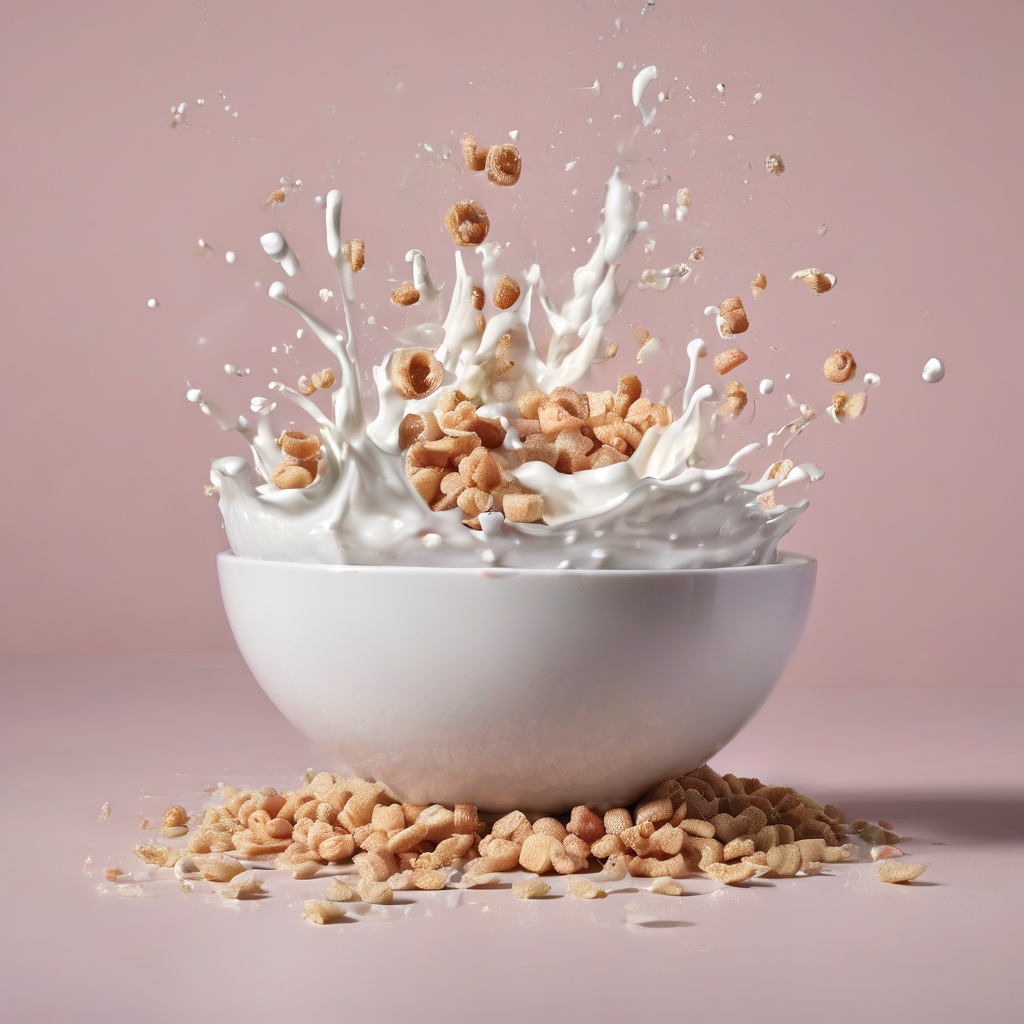} \\
\multicolumn{4}{c}{
\textit{Text prompt: "A bowl of cereal with milk splashing dramatically"}} \\[0.5em]

\includegraphics[width=0.20\textwidth]{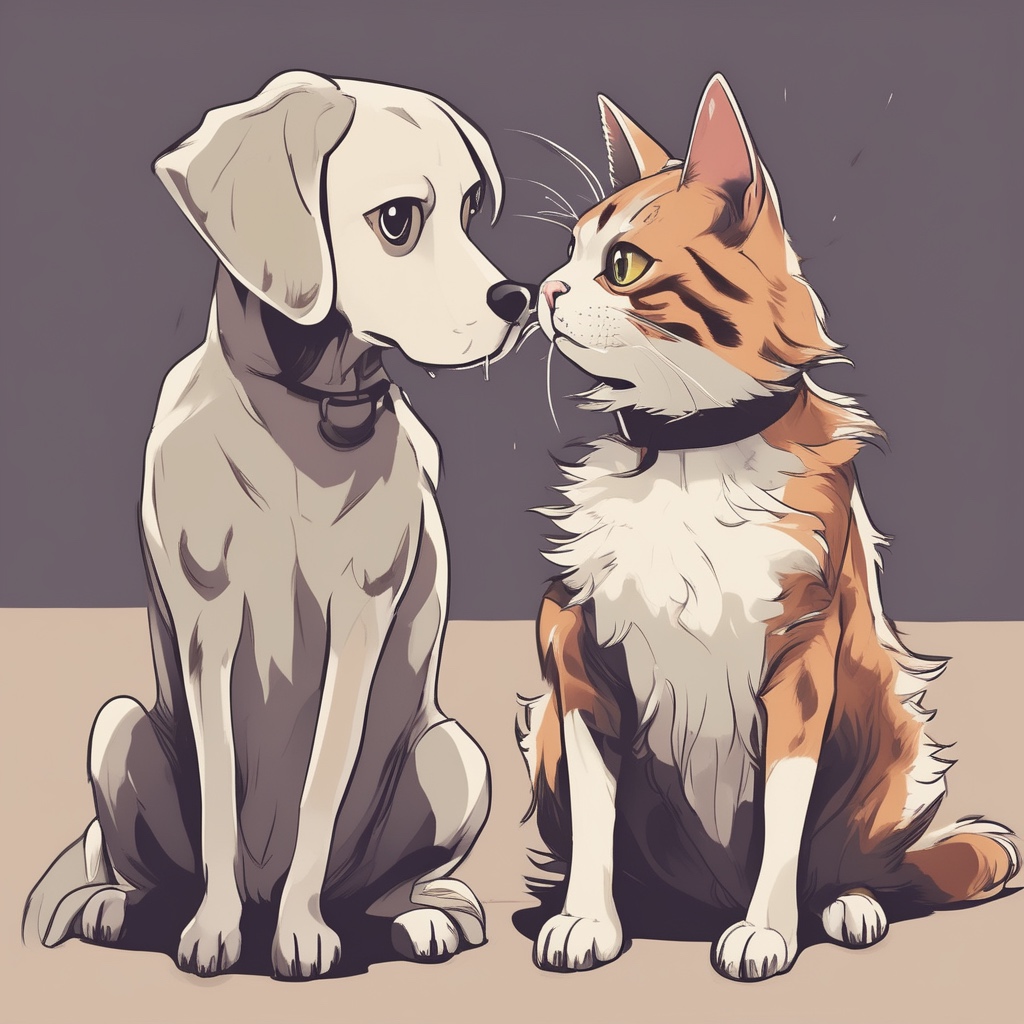} &
\includegraphics[width=0.20\textwidth]{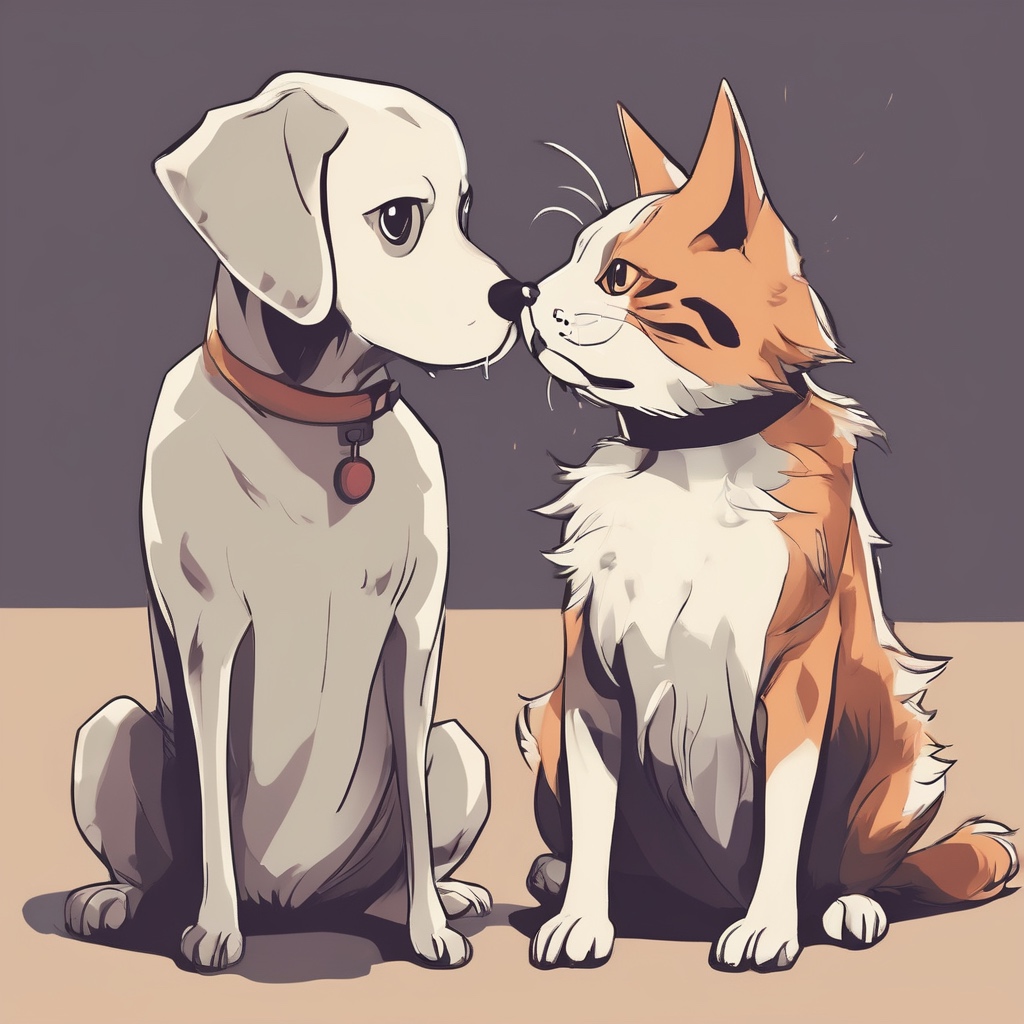} &
\includegraphics[width=0.20\textwidth]{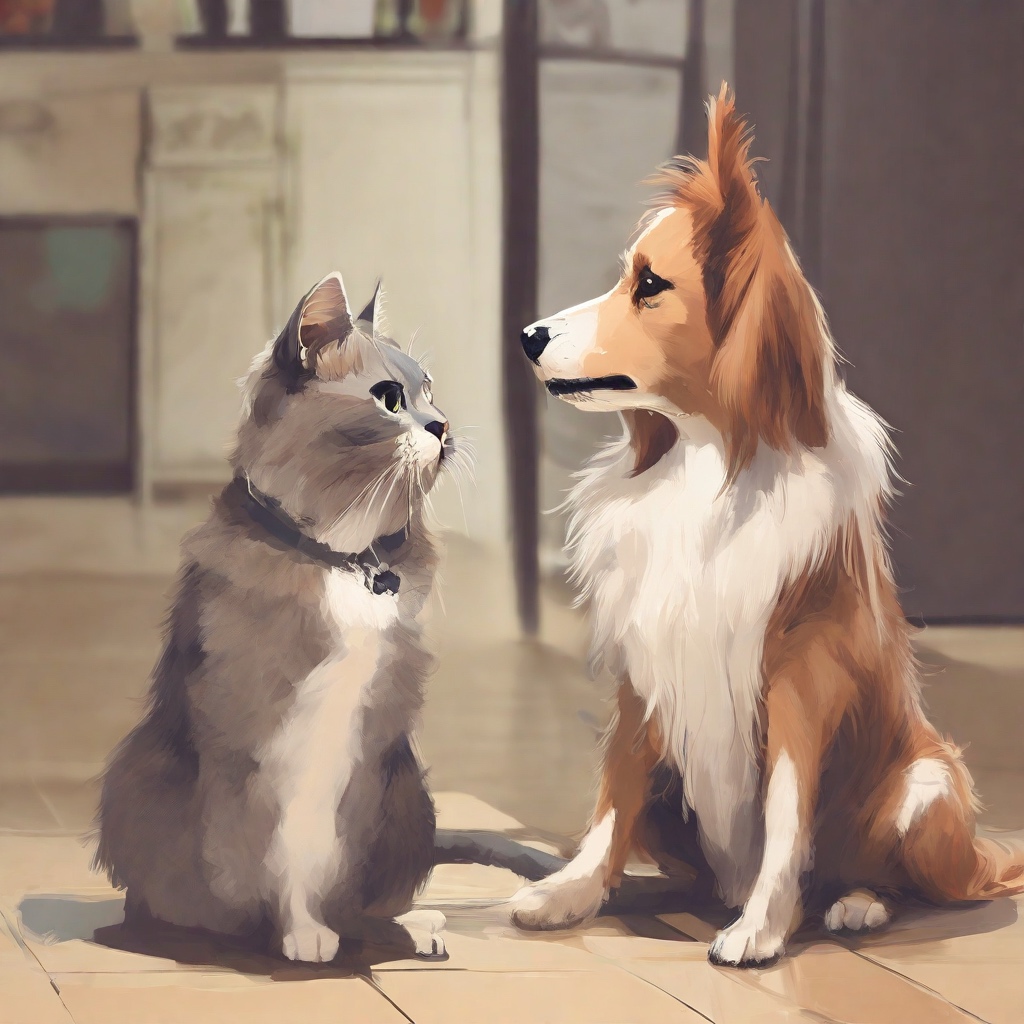} & 
\includegraphics[width=0.20\textwidth]{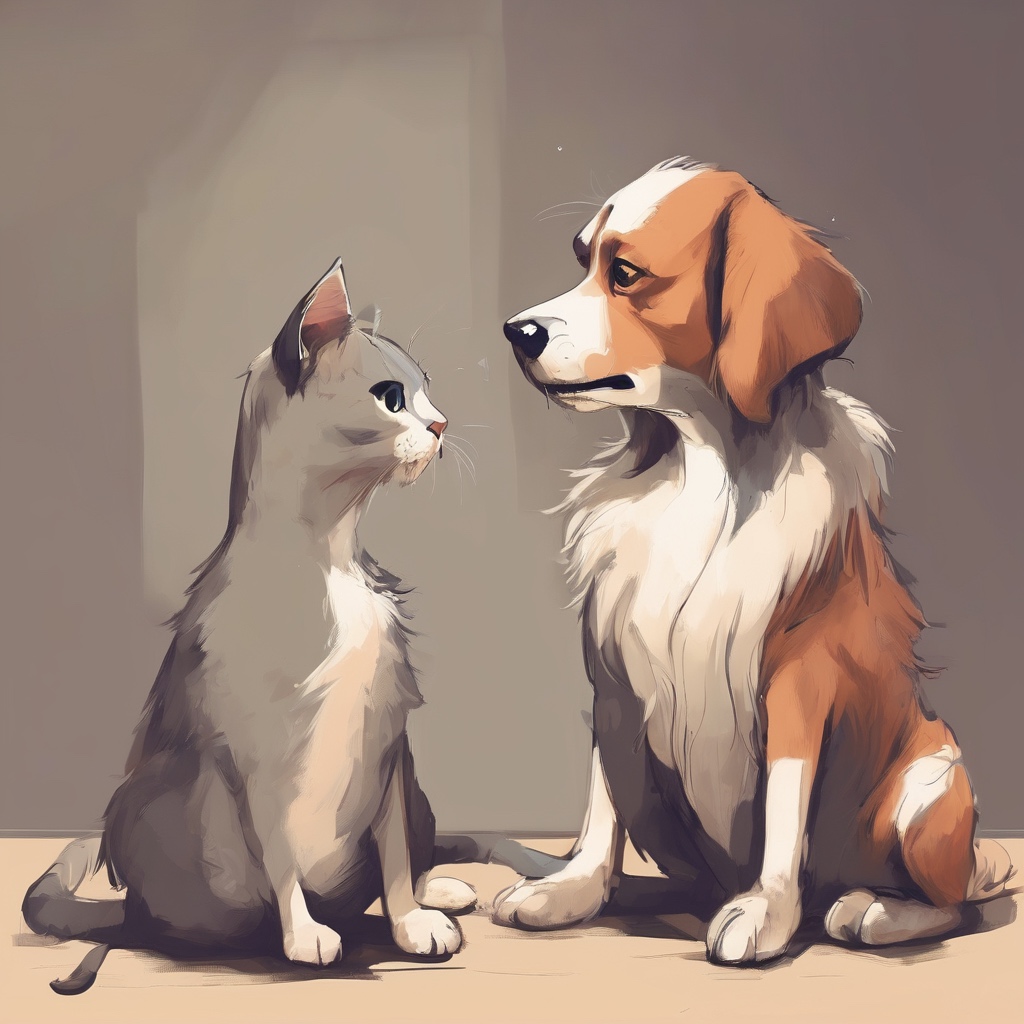} \\
\multicolumn{4}{c}{
\textit{Text prompt: "A dog and cat staring at each other intensely, humorous tension"}} \\[0.5em]

\includegraphics[width=0.20\textwidth]{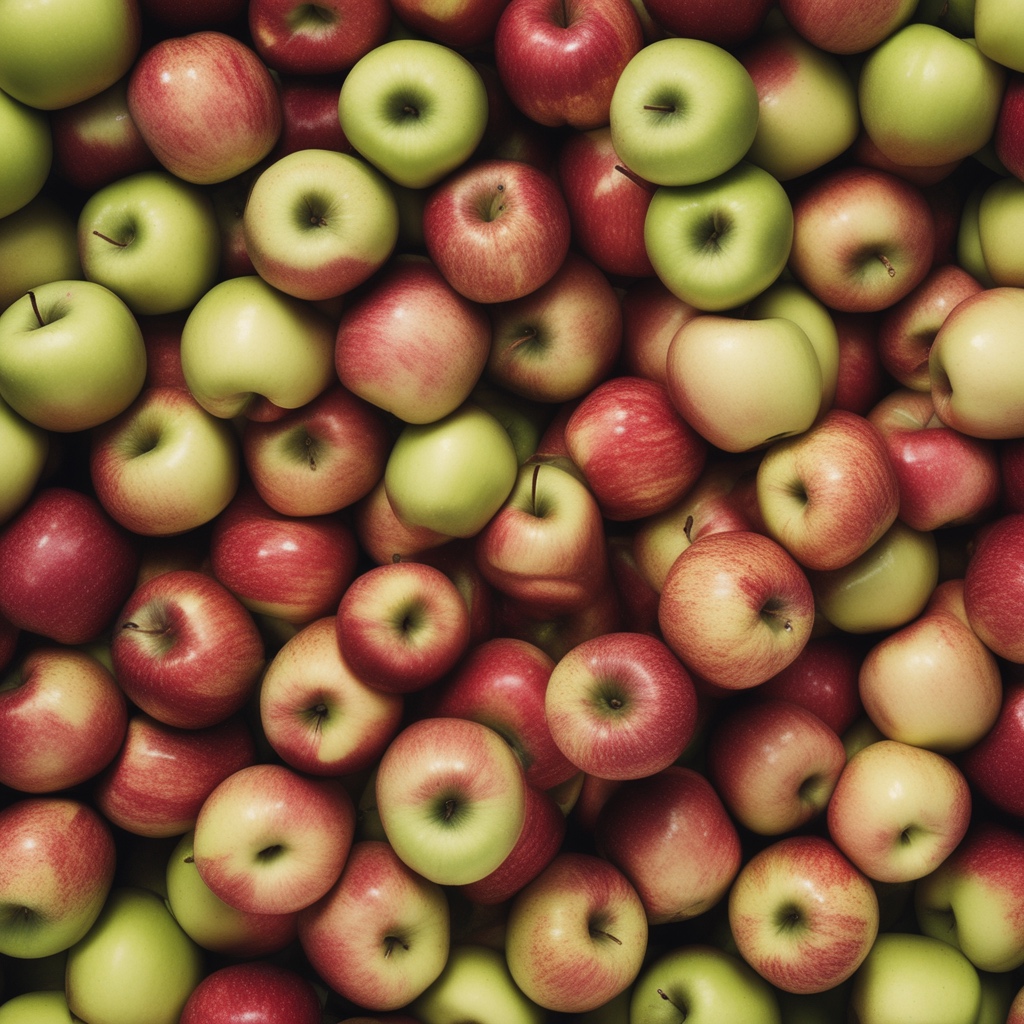} &
\includegraphics[width=0.20\textwidth]{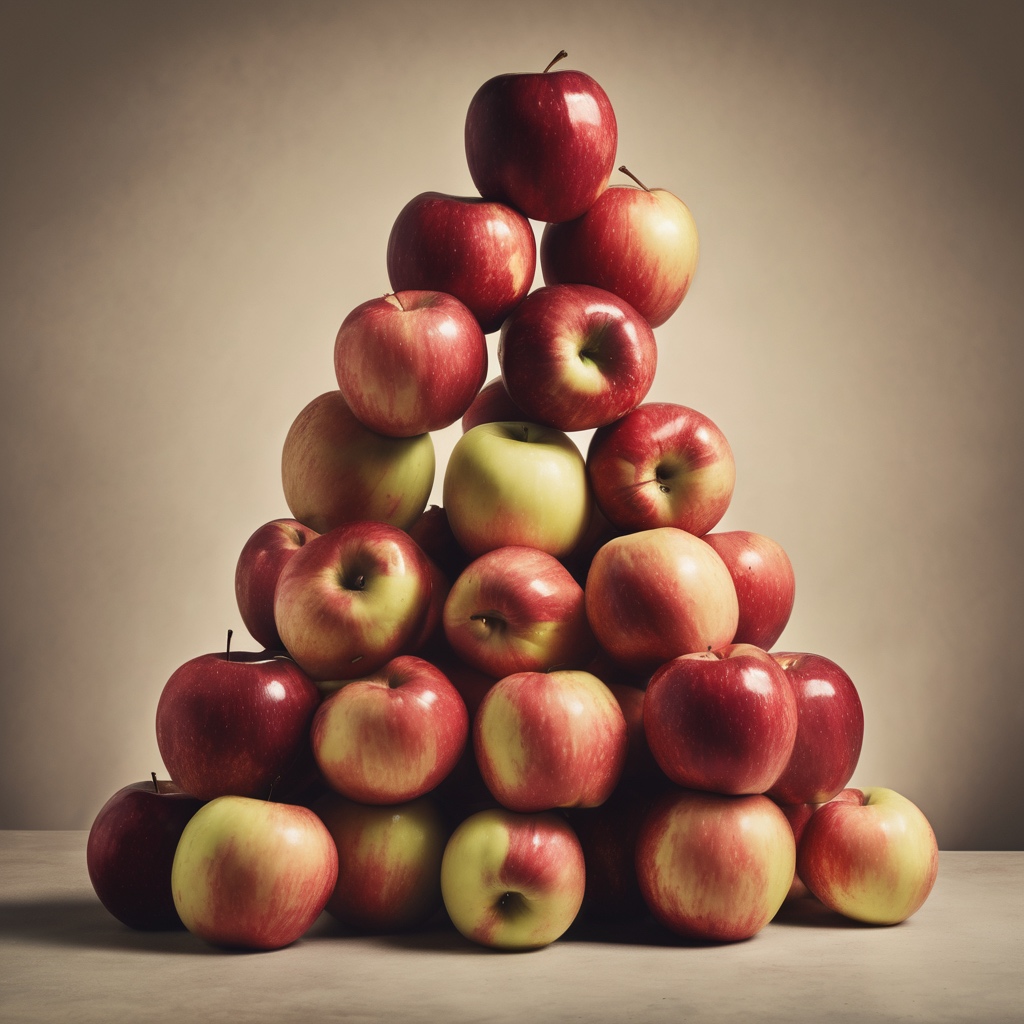} &
\includegraphics[width=0.20\textwidth]{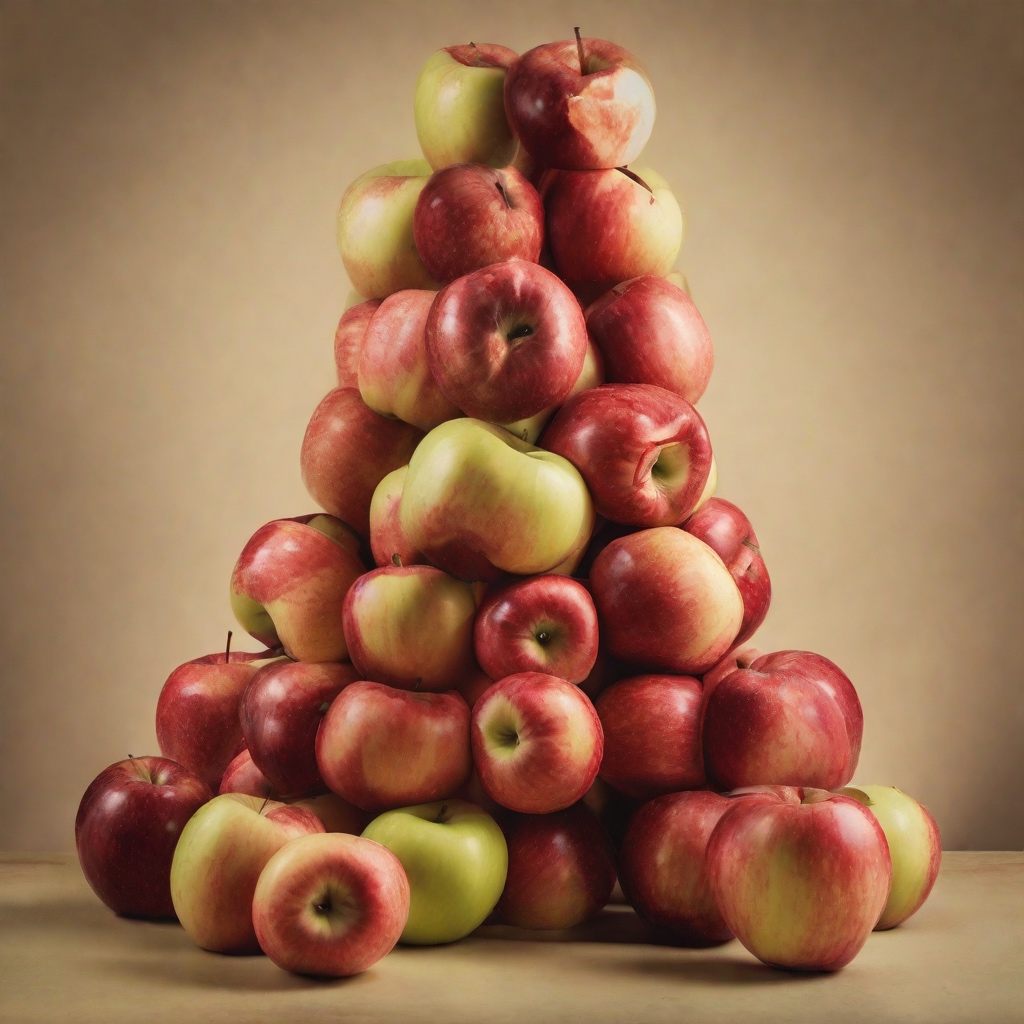} & 
\includegraphics[width=0.20\textwidth]{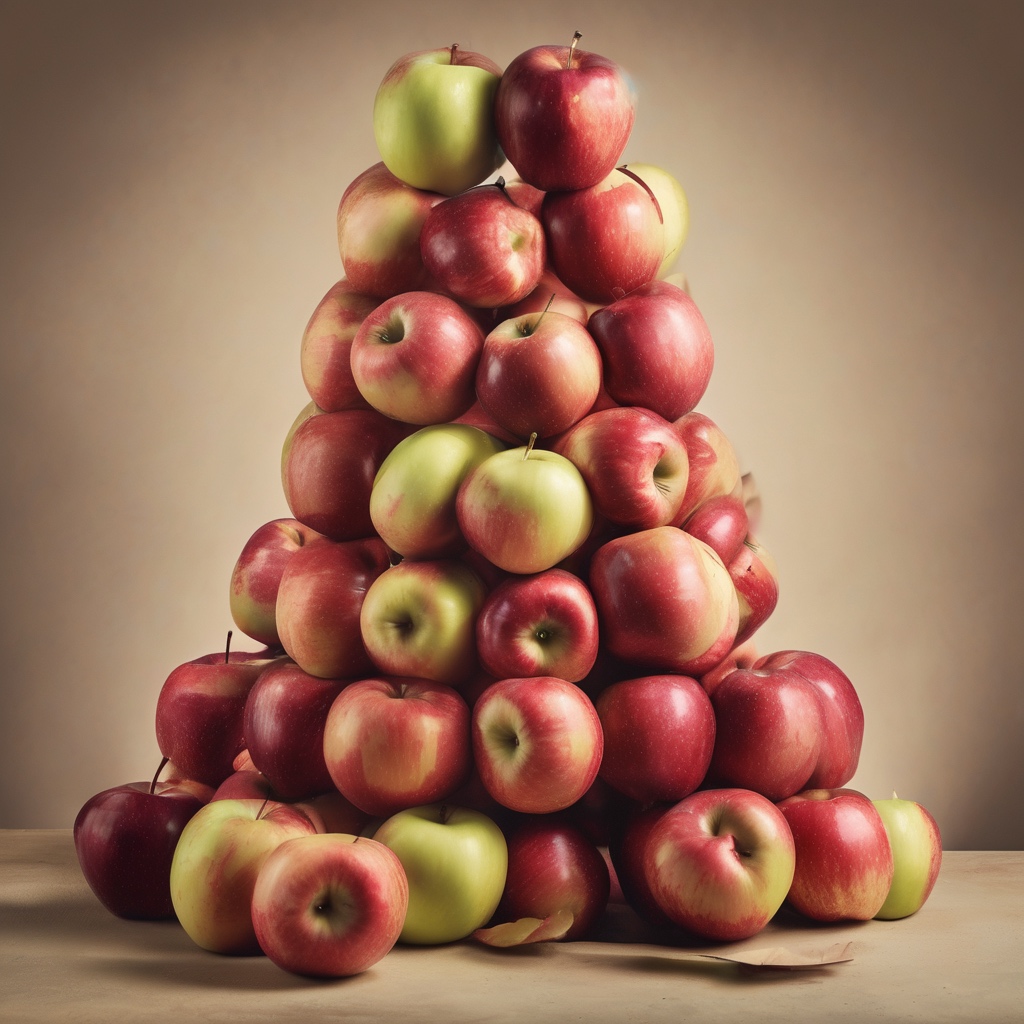} \\
\multicolumn{4}{c}{
\textit{Text prompt: "Several apples stacked in a pyramid structure, perfectly balanced"}} \\

\end{tabular}
\caption{
    Qualitative comparison of different sampling schedules using SDXL for text-to-image generation. From left to right: Time uniform, EDM, AYS, and SDM.
    All results are generated using DPM-Solver++(2M) as the numerical solver.
}
\label{fig:additional_results_t2i_sdxl}
\end{figure}

\begin{figure}[ht]
\centering
\setlength{\tabcolsep}{4pt}

\begin{tabular}{cccc}
\textbf{Time uniform} &
\textbf{SDM} &
\textbf{Time uniform} &
\textbf{SDM} \\

\includegraphics[width=0.22\textwidth]{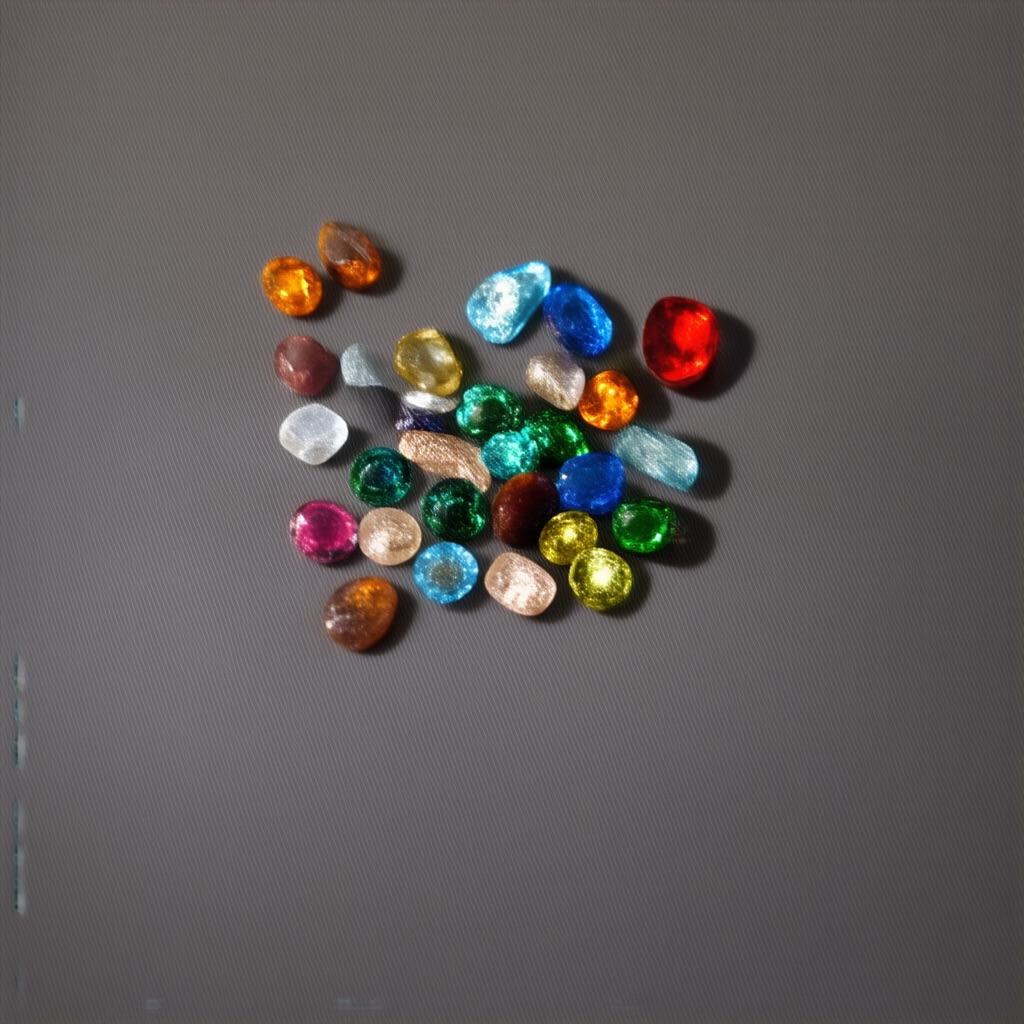} &
\includegraphics[width=0.22\textwidth]{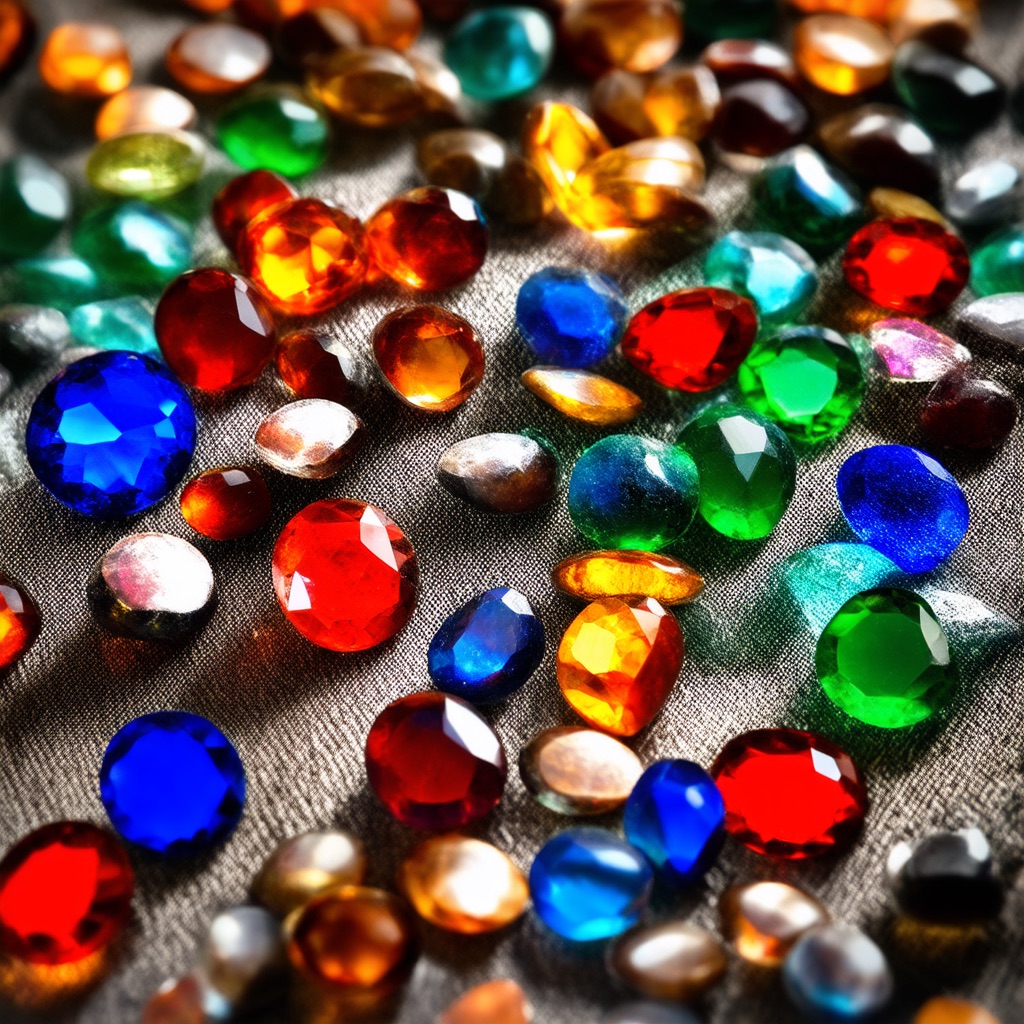} &
\includegraphics[width=0.22\textwidth]{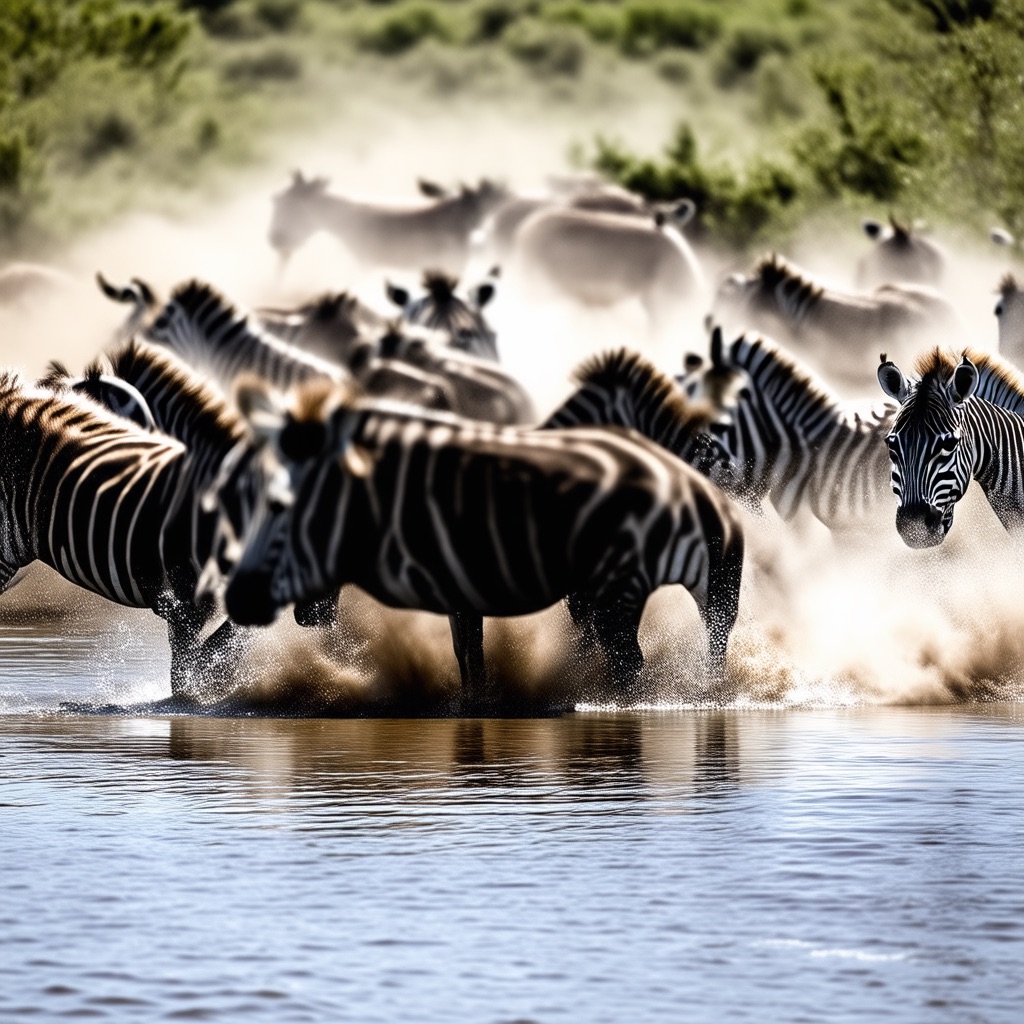} &
\includegraphics[width=0.22\textwidth]{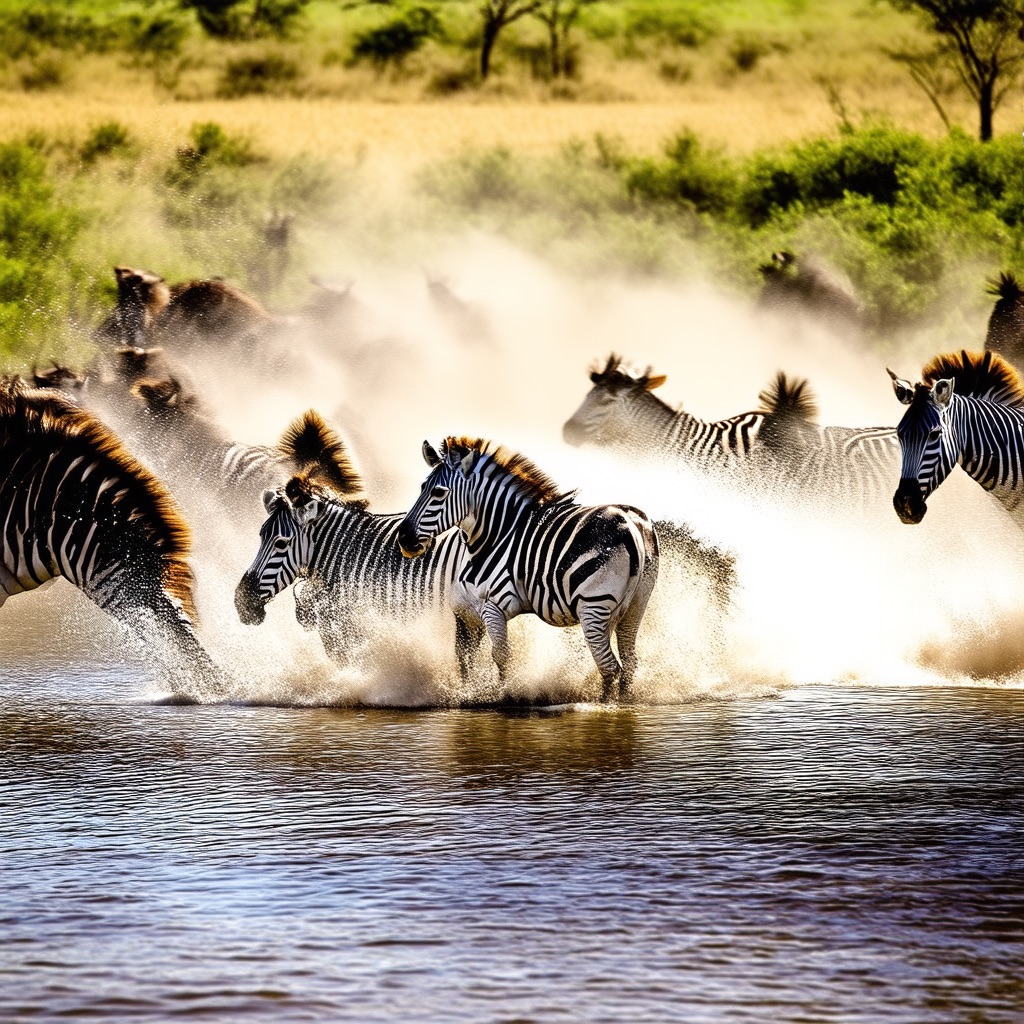} \\

\multicolumn{2}{c}{
\parbox{0.45\textwidth}{\centering
\textit{Text prompt: ``A pile of polished gemstones in various colors, sharp reflections''}
}
} &
\multicolumn{2}{c}{
\parbox{0.45\textwidth}{\centering
\textit{Text prompt: ``A herd of zebras crossing a river in the Serengeti, splashing water, dust, wildlife photography''}
}
} \\[1.5em]

\includegraphics[width=0.22\textwidth]{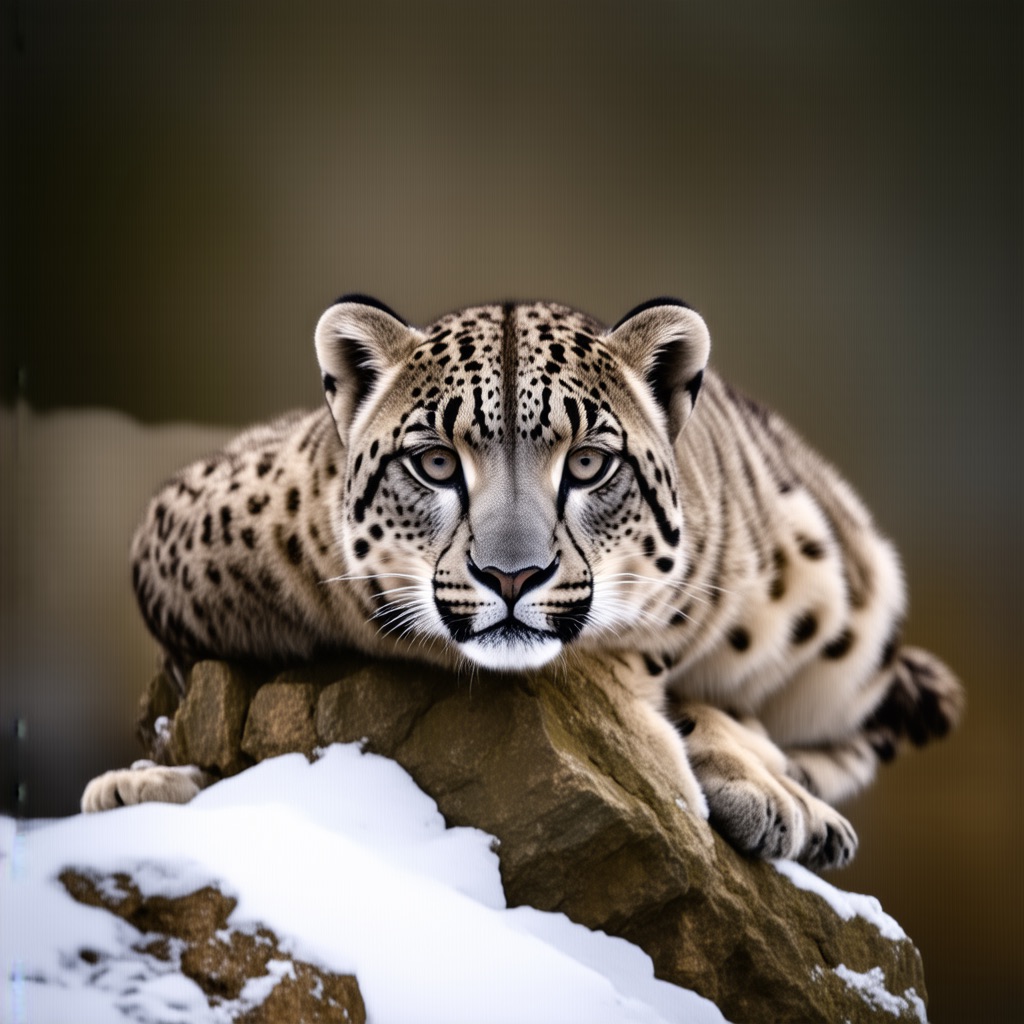} &
\includegraphics[width=0.22\textwidth]{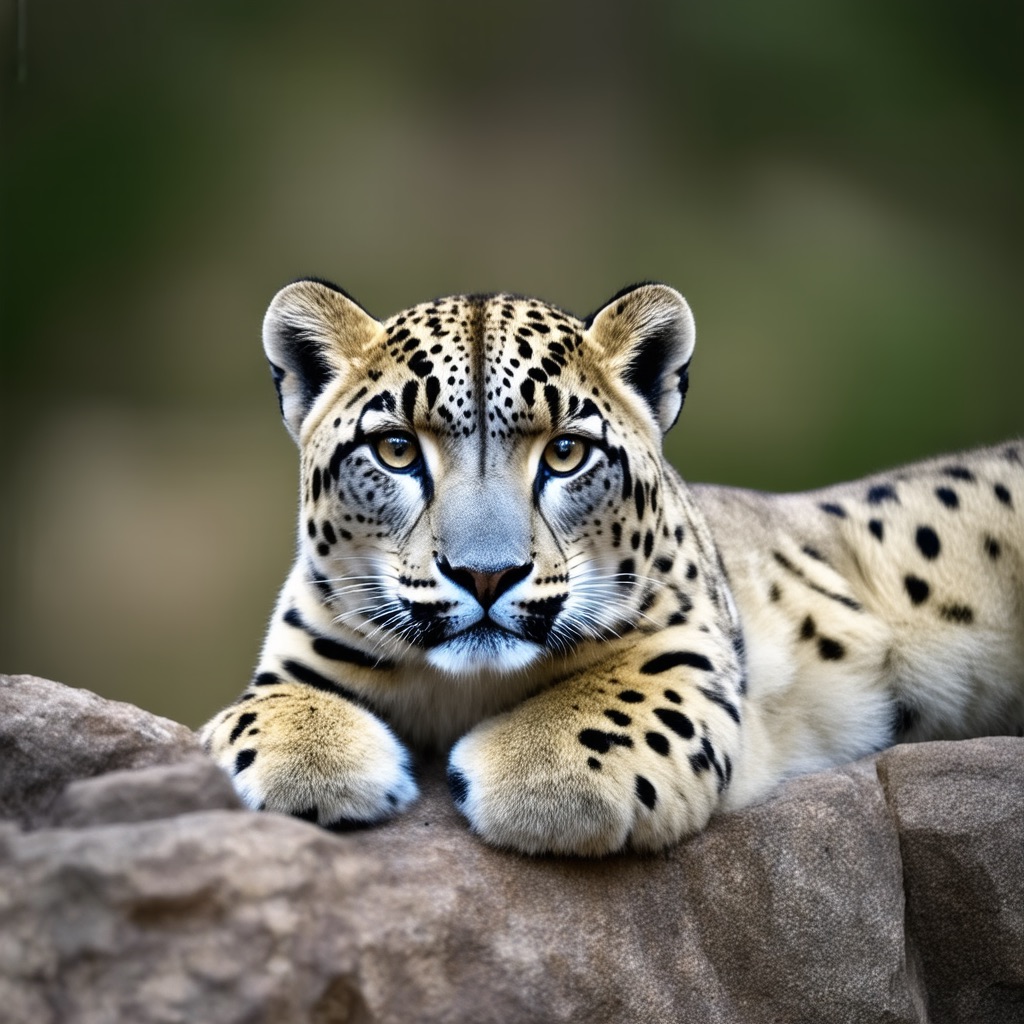} &
\includegraphics[width=0.22\textwidth]{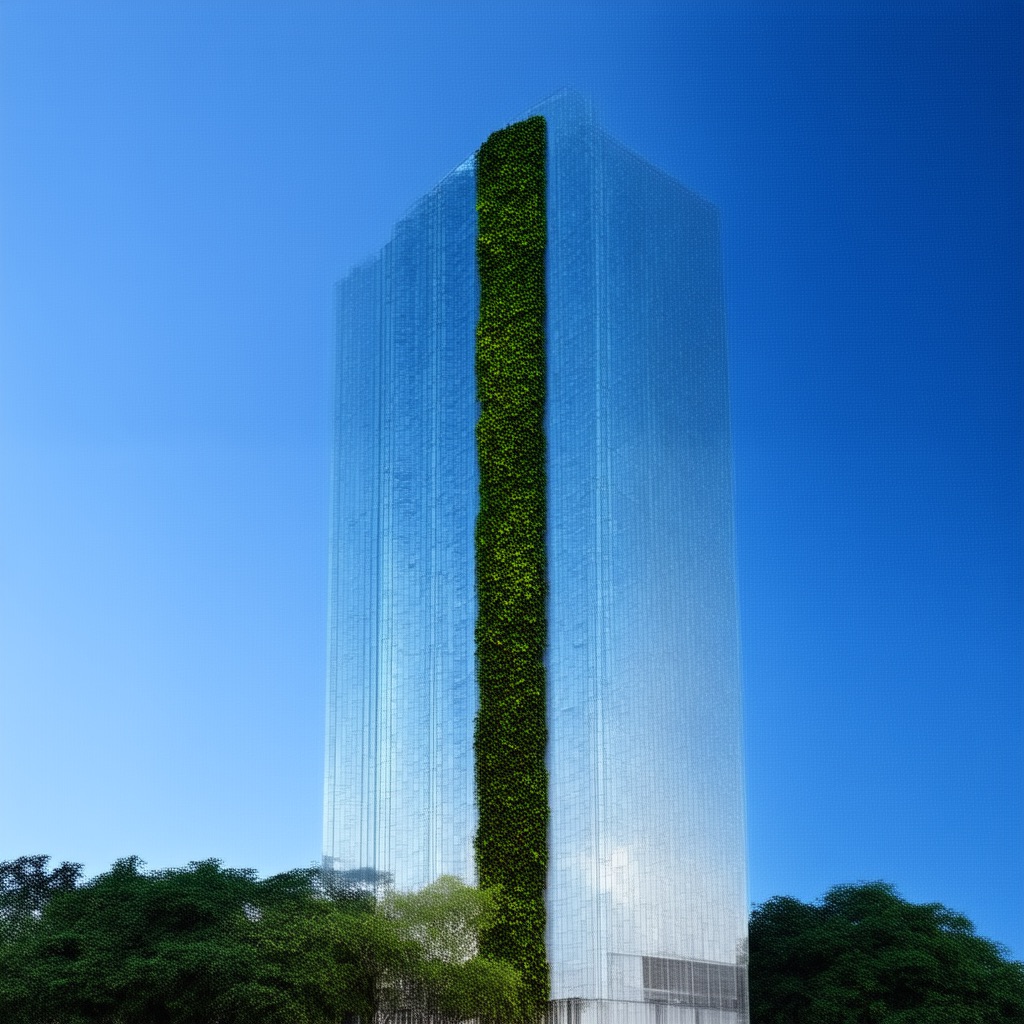} &
\includegraphics[width=0.22\textwidth]{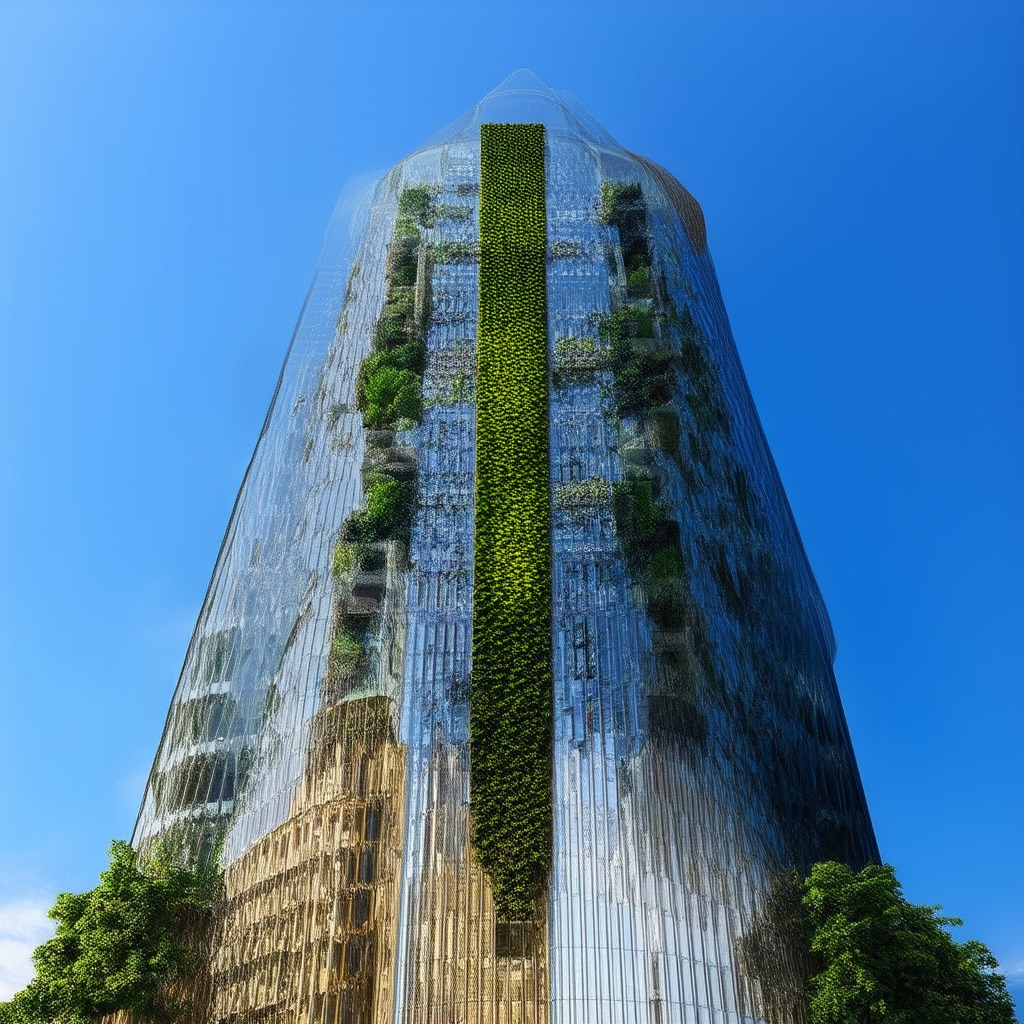} \\

\multicolumn{2}{c}{
\parbox{0.45\textwidth}{\centering
\textit{Text prompt: ``A snow leopard resting on a rocky cliff, piercing eyes, detailed fur, telephoto shot''}
}
} &
\multicolumn{2}{c}{
\parbox{0.45\textwidth}{\centering
\textit{Text prompt: ``A futuristic glass skyscraper with vertical gardens, blue sky, architectural photography''}
}
} \\[1.5em]

\includegraphics[width=0.22\textwidth]{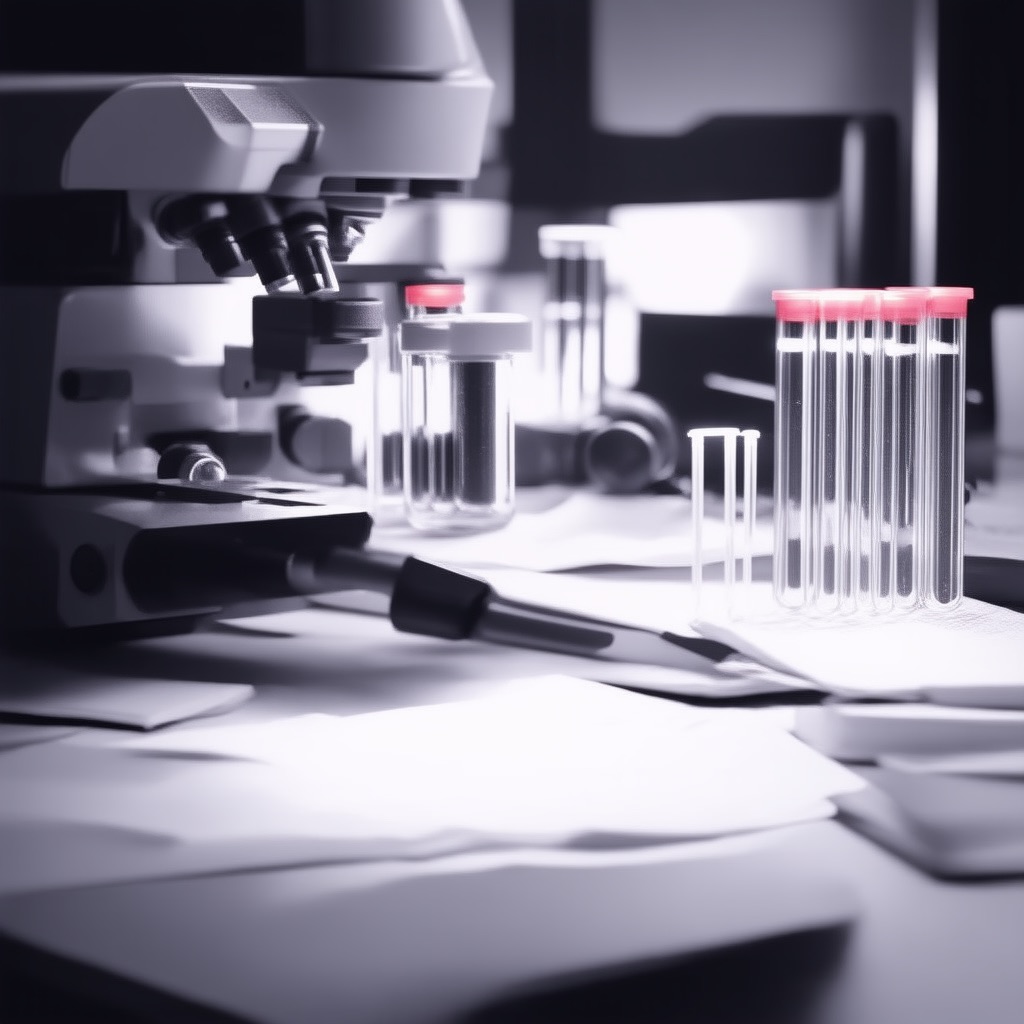} &
\includegraphics[width=0.22\textwidth]{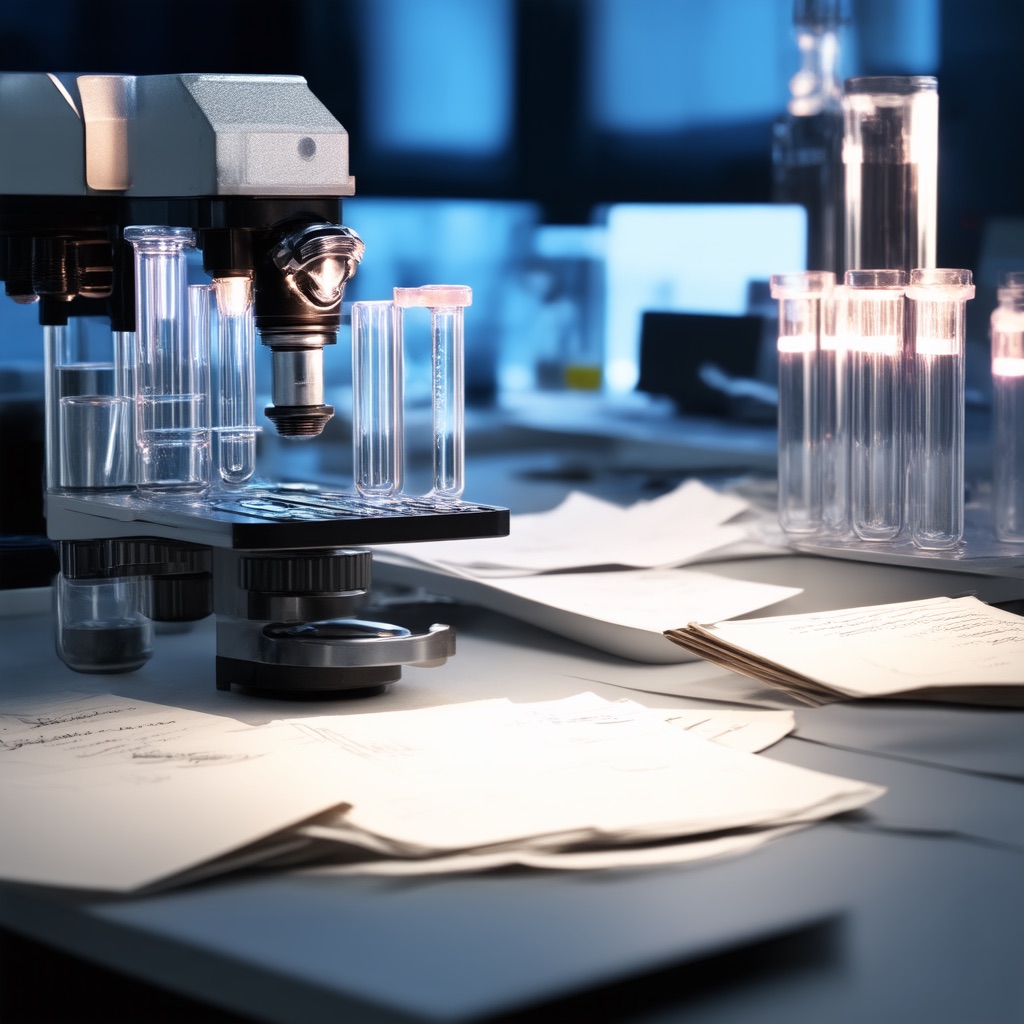} &
\includegraphics[width=0.22\textwidth]{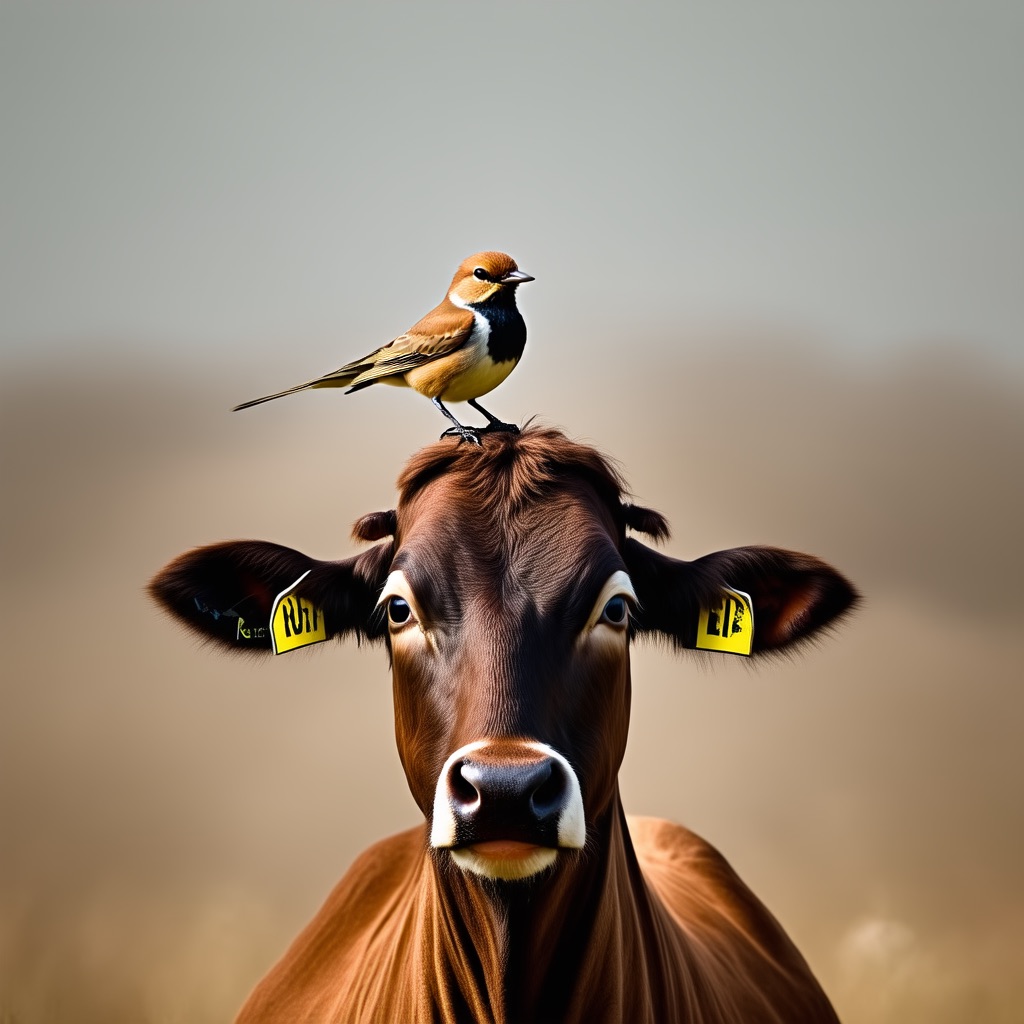} &
\includegraphics[width=0.22\textwidth]{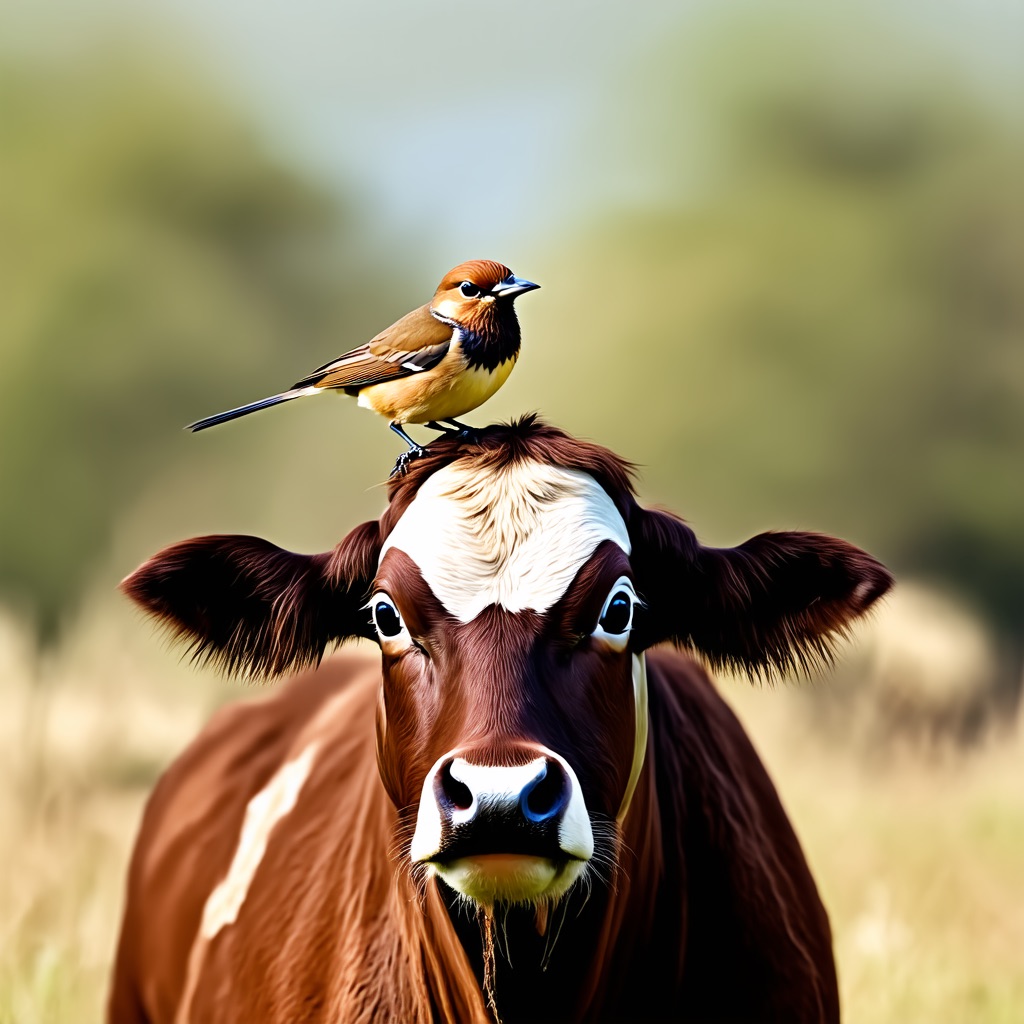} \\

\multicolumn{2}{c}{
\parbox{0.45\textwidth}{\centering
\textit{Text prompt: ``A science lab with glowing test tubes, microscopes and scattered notes, cinematic lighting''}
}
} &
\multicolumn{2}{c}{
\parbox{0.45\textwidth}{\centering
\textit{Text prompt: ``A small bird standing on the head of a large cow''}
}
} \\

\end{tabular}
\caption{
    Qualitative comparison of different sampling schedules using Stable Diffusion 3 for text-to-image generation. From left to right: Time uniform and SDM.
    All results are generated using Euler as the numerical solver.
}
\label{fig:additional_results_t2i_sd3}
\end{figure}

\end{document}